\documentclass{article}

\usepackage[final]{neurips_2019}
\usepackage[utf8x]{inputenc} 
\usepackage[T1]{fontenc}    
\usepackage[colorlinks=True,linkcolor=blue,citecolor=blue]{hyperref}       
\usepackage{caption,subcaption}

\usepackage{wrapfig}

\usepackage{url}            
\usepackage{booktabs}       
\usepackage{amsfonts}       
\usepackage{nicefrac}       
\usepackage{microtype}      
\usepackage{amsmath}
\usepackage{amssymb}
\usepackage{amsthm}
\usepackage{bm}
\usepackage{graphicx}
\usepackage{caption}
\usepackage[skip=.1ex]{subcaption}
\usepackage{soul}

\usepackage{tikz}
\usetikzlibrary{bayesnet,calc,angles,quotes}
\usepackage{tikz-3dplot}
\usetikzlibrary{decorations.pathreplacing}

\newtheorem{theorem}{Theorem}
\newtheorem{corollary}{Corollary}
\newtheorem{proposition}{Proposition}

\newtheorem{lemma}{Lemma}

\AtEndDocument{\refstepcounter{theorem}\label{finaltheorem}}
\AtEndDocument{\refstepcounter{corollary}\label{finalcorollary}}
\AtEndDocument{\refstepcounter{proposition}\label{finalproposition}}
\AtEndDocument{\refstepcounter{definition}\label{finaldefinition}}
\AtEndDocument{\refstepcounter{lemma}\label{finallemma}}
\AtEndDocument{\refstepcounter{assumption}\label{finalassumption}}
\AtEndDocument{\refstepcounter{problem}\label{finalproblem}}
\AtEndDocument{\refstepcounter{example}\label{finalexample}}

\DeclareMathOperator{\Tr}{Tr}

\DeclareMathOperator{\Det}{Det}

\DeclareMathOperator{\Span}{\mathrm{Span}}

\DeclareMathOperator{\DPP}{\mathrm{DPP}}

\DeclareMathOperator{\Tran}{\intercal}

\DeclareMathOperator{\EX}{\mathbb{E}}

\DeclareMathOperator{\Prb}{\mathbb{P}}

\DeclareMathOperator*{\KDPP}{\mathfrak{K}}

\def\twofig{.49\textwidth}

\title{Kernel quadrature with DPPs}

\author{%
  Ayoub Belhadji, \: Rémi Bardenet, \: Pierre Chainais \\
  Univ. Lille, CNRS, Centrale Lille, UMR 9189 - CRIStAL, Villeneuve d’Ascq, France \\
  \texttt{\{ayoub.belhadji, remi.bardenet, pierre.chainais\}@univ-lille.fr} \\
}

\begin{document}

\maketitle

\begin{abstract}
We study quadrature rules for functions from an RKHS, using nodes sampled from a determinantal point process (DPP). DPPs are parametrized by a kernel, and we use a truncated and saturated version of the RKHS kernel.
This link between the two kernels, along with DPP machinery, leads to relatively tight bounds on the quadrature error, that depends on the spectrum of the RKHS kernel.
Finally, we experimentally compare DPPs to existing kernel-based quadratures such as herding, Bayesian quadrature, or leverage score sampling. Numerical results confirm the interest of DPPs, and even suggest faster rates than our bounds in particular cases.

\end{abstract}

\section{Introduction}\label{s:intro}

Numerical integration \cite{DaRa07} is an important tool for Bayesian methods \cite{Rob07} and model-based machine learning \cite{Mur12}.
Formally, numerical integration consists in approximating
\begin{equation}
\int_{\mathcal{X}}f(x)g(x)\mathrm{d}\omega(x) \approx \sum\limits_{j \in [N]} w_{j}f(x_{j}) ,
\label{e:quadrature}
\end{equation}
where $\mathcal{X}$ is a topological space, $\mathrm{d}\omega$ is a Borel probability measure on $\mathcal{X}$, $g$ is a square integrable function, and $f$ is a function belonging to a space to be precised. In the quadrature formula \eqref{e:quadrature}, the $N$ points $x_{1}, \dots, x_{N} \in \mathcal{X}$ are called the quadrature \emph{nodes}, and $w_{1}, \dots, w_{N}$ the corresponding \emph{weights}.

The accuracy of a quadrature rule is assessed by the quadrature error, i.e., the absolute difference between the left-hand side and the right-hand side of \eqref{e:quadrature}.
Classical Monte Carlo algorithms, like importance sampling or Markov chain Monte Carlo \cite{RoCa04}, pick up the nodes as either independent samples or a sample from a Markov chain on $\mathcal{X}$, and all achieve a root mean square quadrature error in $\mathcal{O}(1/\sqrt{N})$. Quasi-Monte Carlo quadrature \cite{DiPi10} is based on deterministic, low-discrepancy sequences of nodes, and typical error rates for $\mathcal{X}=\mathbb{R}^d$ are $\mathcal{O}(\log^d N /N)$. Recently, kernels have been used to derive quadrature rules such as herding \cite{BaLaOb12,ChWeSm10}, Bayesian quadrature \cite{HuDu12,Hag91}, sophisticated control variates \cite{LiLe17,OaGiCh17}, and leverage-score quadrature \cite{Bac17} under the assumption that $f$ lies in a RKHS. The main theoretical advantage is that the resulting error rates are  faster than classical Monte Carlo and adapt to the smoothness of $f$.

In this paper, we propose a new quadrature rule for functions in a given RKHS. Our nearest scientific neighbour is \cite{Bac17}, but instead of sampling nodes independently, we leverage dependence and use a repulsive distribution called a projection determinantal point process (DPP), while the weights are obtained through a simple quadratic optimization problem. DPPs were originally introduced by \cite{Mac75} as probabilistic models for beams of fermions in quantum optics. Since then, DPPs have been thoroughly studied in random matrix theory \citep{Joh05}, and have more recently been adopted in machine learning \citep*{KuTa12} and Monte Carlo methods \citep{BaHa16}.

In practice, a projection DPP is defined through a reference measure $\mathrm{d}\omega$ and a repulsion kernel $\KDPP$.
In our approach, the repulsion kernel is a modification of the underlying RKHS kernel. This ensures that sampling is tractable, and, as we shall see, that the expected value of the quadrature error is controlled by the decay of the eigenvalues of the integration operator associated to the measure $\mathrm{d}\omega$.
Note that quadratures based on projection DPPs have already been studied in the literature: implicitly in \cite[Corollary 2.3]{Joh97} in the simple case where $\mathcal{X} = [0,1]$ and $\mathrm{d}\omega$ is the uniform measure, and in \cite{BaHa16} for $[0,1]^{d}$ and more general measures. In the latter case, the quadrature error is asymptotically of order $N^{-1/2-1/2d}$ \cite{BaHa16}, with $f$ essentially $\mathcal{C}^1$. In the current paper, we leverage the smoothness of the integrand to improve the convergence rate of the quadrature in general spaces $\mathcal{X}$.
%

This article is organized as follows. Section~\ref{s:relatedWork} reviews kernel-based quadrature. In Section~\ref{s:DPPs}, we recall some basic properties of projection DPPs. Section~\ref{s:DPPKQ} is devoted to the exposition of our main result, along with a sketch of proof. We give precise pointers to the supplementary material for missing details. Finally, in Section~\ref{s:numsims} we illustrate our result and compare to related work using numerical simulations, for the uniform measure in $d=1$ and $2$, and the Gaussian measure on $\mathbb{R}$.

\paragraph{Notation.}
\label{s:notation}
Let $\mathcal{X}$ be a topological space equipped with a Borel measure $\mathrm{d}\omega$ and assume that the support of $\mathrm{d}\omega$ is $\mathcal{X}$. Let $\mathbb{L}_{2}(\mathrm{d}\omega)$ be the Hilbert space of square integrable, real-valued functions defined on $\mathcal{X}$, with the usual inner product denoted by $\langle \cdot, \cdot \rangle_{\mathrm{d}\omega}$, and the associated norm by $\|.\|_{\mathrm{d}\omega}$.
%
Let $k : \mathcal{X} \times \mathcal{X} \rightarrow \mathbb{R}_{+}$ be a symmetric and continuous function such that, for any finite set of points in $\mathcal{X}$, the matrix of pairwise kernel evaluations is positive semi-definite. Denote by $\mathcal{F}$ the associated reproducing kernel Hilbert space (RKHS) of real-valued functions \cite{BeTh11}.
We assume that $x \mapsto k(x,x)$ is integrable with respect to the measure $\mathrm{d}\omega$ so that $\mathcal{F} \subset \mathbb{L}_{2}(\mathrm{d}\omega)$. Define the integral operator
\begin{equation}
\bm{\Sigma} f (\cdot) = \int_{\mathcal{X}} k(\cdot,y)f(y) \mathrm{d}\omega(y), \quad f \in \mathbb{L}_{2}(\mathrm{d}\omega).
\end{equation}
By construction, $\bm{\Sigma}$ is self-adjoint, positive semi-definite, and trace-class \cite{Sim05}.
%
For $m \in \mathbb{N}$, denote by $e_{m}$ the $m$-th eigenfunction of $\bm{\Sigma}$, normalized so that $\|e_{m}\|_{\mathrm{d}\omega} = 1$ and $\sigma_{m}$ the corresponding eigenvalue. The integrability of the diagonal $x \mapsto k(x,x)$ implies that $\mathcal{F}$ is compactly embedded in $\mathbb{L}_{2}(\mathrm{d}\omega)$, that is, the identity map $I_{\mathcal{F}}: \mathcal{F} \longrightarrow \mathbb{L}_{2}(\mathrm{d}\omega)$ is compact; moreover, since $\mathrm{d}\omega$ is of full support in $\mathcal{X}$, $I_{\mathcal{F}}$ is injective \cite{StCh08}. This implies a Mercer-type decomposition of $k$,
\begin{equation}
k(x,y)= \sum\limits_{m \in \mathbb{N}^{*}}\sigma_{m}e_{m}(x)e_{m}(y),
\end{equation}
where $\mathbb{N}^{*} = \mathbb{N} \smallsetminus \{0\}$ and the convergence is point-wise \cite{StSc12}. Moreover, for $m \in \mathbb{N}^{*}$, we write $e_{m}^{\mathcal{F}} = \sqrt{\sigma_{m}}e_{m}$. Since $I_{\mathcal{F}}$ is injective \cite{StSc12}, $(e_{m}^{\mathcal{F}})_{m \in \mathbb{N}^{*}}$ is an orthonormal basis of $\mathcal{F}$. Unless explicitly stated, we assume that $\mathcal{F}$ is dense in $\mathbb{L}_{2}(\mathrm{d}\omega)$, so that $(e_{m})_{m \in \mathbb{N}^{*}}$ is an orthonormal basis of $\mathbb{L}_{2}(\mathrm{d}\omega)$. For more intuition, under these assumptions, $f \in \mathcal{F}$ if and only if $\sum_m \sigma_{m}^{-1} \langle f,e_{m} \rangle_{\mathbb{L}_{2}(\mathrm{d} \omega)}^{2}$ converges.

\section{Related work on kernel-based quadrature}
\label{s:relatedWork}
%
When the integrand $f$ belongs to the RKHS $\mathcal{F}$ of kernel $k$ \citep{CrSh04}, the quadrature error reads \citep{SmGrSoSc07}
\begin{align}
\label{eq:integral_bound_mean_element}
  \int_{\mathcal{X}} f(x)g(x)\mathrm{d}\omega(x) - \sum\limits_{j \in [N]} w_{j}f(x_{j})
  & = \langle f, \mu_{g} - \sum\limits_{j \in [N]} w_{j} k(x_{j},.) \rangle_{\mathcal{F}} \nonumber\\
  & \leq \|f\|_{\mathcal{F}} \, \Big\|\mu_{g} - \sum\limits_{j \in [N]} w_{j} k(x_{j},.)\Big\|_{\mathcal{F}}\,,
\end{align}
where $\mu_{g} = \int_{\mathcal{X}}g(x) k(x,.) \mathrm{d}\omega(x)$ is the so-called \emph{mean element} \citep{DiPi14,MuFuSrSc17}.
A tight approximation of the mean element by a linear combination of functions $k(x_{j},.)$ thus guarantees low quadrature error. The approaches described in this section differ by their choice of nodes and weights.

\subsection{Bayesian quadrature and the design of nodes}
\emph{Bayesian Quadrature} initially \citep{Lar72} considered a fixed set of nodes and put a Gaussian process prior on the integrand $f$. Then, the weights were chosen to minimize the posterior variance of the integral of $f$. If the kernel of the Gaussian process is chosen to be $k$, this amounts to minimizing the RHS of \eqref{eq:integral_bound_mean_element}. The case of the Gaussian reference measure was later investigated in detail \citep{Hag91}, while parametric integrands were considered in \citep{Min00}. Rates of convergence were provided in \citep{BOGOS2019} for specific kernels on compact spaces, under a \emph{fill-in} condition \citep{Wend04} that encapsulates that the nodes must progressively fill up the (compact) space.


Finding the weights that optimize the RHS of \eqref{eq:integral_bound_mean_element} for a fixed set of nodes is a relatively simple task, see later Section~\ref{subsec:unreg_opt_problem}, the cost of which can even be reduced using symmetries of the set of nodes \citep{JaHi18,KaSaOa18}.
Jointly optimizing on the nodes and weights, however, is only possible in specific cases \citep{Boj81,KaSa17}. In general, this corresponds to a non-convex problem with many local minima \citep{HiOe16,Oett17}.
While \cite{RaGh03} proposed to sample i.i.d. nodes from the reference measure $\mathrm{d}\omega$, greedy minimization approaches have also been proposed \citep{HuDu12,Oett17}. In particular, \emph{kernel herding} \citep{ChWeSm10} corresponds to uniform weights and greedily minimizing the RHS in \eqref{eq:integral_bound_mean_element}. This leads to a fast rate in $\mathcal{O}(1/N)$, but only when the integrand is in a finite-dimensional RKHS. Kernel herding and similar forms of \emph{sequential} Bayesian quadrature are actually linked to the Frank-Wolfe algorithm \citep{BaLaOb12,BrOaGiOs15,HuDu12}. Beside the difficulty of proving fast convergence rates, these greedy approaches still require heuristics in practice.

\subsection{Leverage-score quadrature}
\label{s:bach}
In \cite{Bac17}, the author proposed to sample the nodes $(x_j)$ i.i.d. from some proposal distribution $q$, and then pick weights $\hat{\bm{w}}$ in \eqref{e:quadrature} that solve the optimization problem
\begin{equation}\label{eq:reg_kernel_opt_problem}
\min\limits_{w \in \mathbb{R}^{N}} \Big\| \mu_{g} - \sum\limits_{j \in [N]} \frac{w_{j}}{q(x_{j})^{1/2}} k(x_{j},.) \Big\|_{\mathcal{F}}^{2} + \lambda N \|w\|_{2}^{2},
\end{equation}
%
for some regularization parameter $\lambda>0$.
Proposition~\ref{p:bach} gives a bound on the resulting approximation error of the mean element for a specific choice of proposal pdf, namely the  leverage scores
\begin{equation}
q_{\lambda}^*(x) \propto \langle k(x,.), \bm{\Sigma}^{-1/2}(\bm{\Sigma}+ \lambda \mathbb{I}_{\mathbb{L}_{2}(\mathrm{d}\omega)})^{-1}\bm{\Sigma}^{-1/2} k(x,.) \rangle_{\mathbb{L}_{2}(\mathrm{d}\omega)} = \sum\limits_{m \in \mathbb{N}} \frac{\sigma_{m}}{\sigma_{m}+\lambda} e_{m}(x)^{2}.
\label{e:proposalBach}
\end{equation}

\begin{proposition}[Proposition 2 in \cite{Bac17}]
\label{p:bach}
Let $\delta \in [0,1]$, and $d_\lambda = \Tr \bm{\Sigma}(\bm{\Sigma} + \lambda \bm{I})^{-1}$. Assume that $N \geq 5 d_\lambda \log(16 d_\lambda / \delta)$, then
\begin{equation}
\Prb \bigg( \sup\limits_{\|g\|_{\mathrm{d}\omega} \leq 1} \inf\limits_{\|\bm{w}\|^{2}\leq \frac{4}{N}} \Big\| \mu_{g} - \sum\limits_{j \in [N]} \frac{w_{j}}{q_\lambda(x_{j})^{1/2}} k(x_{j},.)\Big\|_{\mathcal{F}}^{2} \leq 4\lambda \bigg) \geq 1- \delta .
\end{equation}
\end{proposition}

In other words, Proposition~\ref{p:bach} gives a uniform  control on the approximation error $\mu_{g}$ by the subspace spanned by the $k(x_{j},.)$ for $g$ belonging to the unit ball of $\mathbb{L}_{2}(\mathrm{d}\omega)$, where the $(x_{j})$ are sampled i.i.d. from $q_{\lambda}^*$. The required number of nodes is equal to $\mathcal{O}(d_{\lambda}\log d_{\lambda})$ for a given approximation error $\lambda$.
However, for fixed $\lambda$, the approximation error in Proposition~\ref{p:bach} does not go to zero when $N$ increases. One theoretical  workaround is to make $\lambda=\lambda(N)$ decrease with $N$. However, the coupling of $N$ and $\lambda$ through $d_\lambda$ makes it very intricate to derive a convergence rate from Proposition~\ref{p:bach}.
Moreover, the optimal density $q_{\lambda}^*$ is in general only available as the limit \eqref{e:proposalBach}, which makes sampling and evaluation difficult. Finally, we note that Proposition~\ref{p:bach} highlights the fundamental role played by the spectral decomposition of the operator $\bm{\Sigma}$ in designing and analyzing kernel quadrature rules.
\section{Projection determinantal point processes}
\label{s:DPPs}
%
%
Let $N \in \mathbb{N}^{*}$ and $(\psi_{n})_{n \in [N]}$ an orthonormal family of $\mathbb{L}_{2}(\mathrm{d}\omega)$, and assume for simplicity that $\mathcal{X}\subset\mathbb{R}^d$ and that $\rm{d}\omega$ has density $\omega$ with respect to the Lebesgue measure. Define the repulsion kernel
\begin{equation}
	\KDPP(x,y) = \sum\limits_{n \in [N]} \psi_{n}(x)\psi_{n}(y),
\end{equation}
not to be mistaken for the RKHS kernel $k$. One can show \cite[Lemma 21]{HKPV06} that
\begin{equation}\label{eq:pdpp_definition}
\frac{1}{N!} \Det(\KDPP(x_{i},x_{j})_{i,j \in [N]}) \prod_{i \in [N]} \omega(x_{i})
\end{equation}
is a probability density over $\mathcal{X}^N$. When $x_1,\dots,x_N$ have distribution \eqref{eq:pdpp_definition}, the set $\bm{x} = \{x_{1}, \dots x_{N} \}$ is said to be a projection DPP\footnote{In the finite case, more common in ML, projection DPPs are also called \emph{elementary} DPPs \cite{KuTa12}.} with reference measure $\mathrm{d}\omega$ and kernel $\KDPP$. Note that the kernel $\KDPP$ is a positive definite kernel so that the determinant in \eqref{eq:pdpp_definition} is non-negative.
Equation \eqref{eq:pdpp_definition} is key to understanding DPPs.
First, loosely speaking, the probability of seeing a point of $\bm{x}$ in an infinitesimal volume around $x_1$ is $\KDPP(x_1,x_1)\omega(x_1)\mathrm{d} x_1$. Note that when $d=1$ and $(\psi_{n})$ are the family of orthonormal polynomials with respect to $\mathrm{d}\omega$, this marginal probability is related to the optimal proposal $q_\lambda$ in Section~\ref{s:bach}; see Appendix~\ref{app:christoffel}. Second, the probability of simultaneously seeing a point of $\bm{x}$ in an infinitesimal volume around $x_1$ and one around $x_2$ is
\begin{align*}
\Big[\KDPP(x_1,x_1)\KDPP(x_2,x_2)  - &\KDPP(x_1,x_2)^2 \Big] \omega(x_1) \omega(x_2)\, \mathrm{d} x_1 \mathrm{d} x_2 \\
&\leq \left[\KDPP(x_1,x_1)\omega(x_1)\mathrm{d}x_1\right]\left[ \KDPP(x_2,x_2)\omega(x_2)\mathrm{d} x_2\right].
\end{align*}
The probability of co-occurrence is thus always smaller than that of a Poisson process with the same intensity. In this sense, a projection DPP with symmetric kernel is a \emph{repulsive} distribution, and $\KDPP$ encodes its repulsiveness.

One advantage of DPPs is that they can be sampled exactly. Because of the orthonormality of $(\psi_n)$, one can write the chain rule for \eqref{eq:pdpp_definition}; see \cite{HKPV06}. Sampling each conditional in turn, using e.g. rejection sampling \cite{RoCa04}, then yields an exact sampling algorithm. Rejection sampling aside, the cost of this algorithm is cubic in $N$ without further assumptions on the kernel. Simplifying assumptions can take many forms. In particular, when $d=1$, and $\omega$ is a Gaussian, gamma \cite{DuEd02}, or beta \cite{KiNe04} pdf, and $(\psi_n)$ are the orthonormal polynomials with respect to $\omega$, the corresponding DPP can be sampled by tridiagonalizing a matrix with independent entries, which takes the cost to $\mathcal{O}(N^2)$ and bypasses the need for rejection sampling. For further information on DPPs see \cite{Joh05, Sos00}.


\section{Kernel quadrature with projection DPPs}\label{s:DPPKQ}
We follow in the footsteps of \cite{Bac17}, see Section~\ref{s:bach}, but using a projection DPP rather than independent sampling to obtain the nodes.
In a nutshell, we consider nodes $(x_{j})_{j \in [N]}$ that are drawn from the projection DPP with reference measure $\mathrm{d}\omega$ and repulsion kernel
\begin{equation}
  \KDPP(x,y) = \sum\limits_{n \in [N]} e_{n}(x)e_{n}(y),
  \label{e:kernel}
\end{equation}
where we recall that $(e_n)$ are the normalized eigenfunctions of the integral operator $\bm{\Sigma}$.
The weights $\bm{w}$ are obtained by solving the optimization problem
\begin{equation}\label{eq:unreg_opt_problem}
\min\limits_{w \in \mathbb{R}^{N}} \| \mu_{g} - \bm{\Phi} \bm{w} \|_{\mathcal{F}}^{2},
\end{equation}
where
\begin{equation}
	\bm{\Phi}:(w_{j})_{j \in [N]} \mapsto \sum_{j \in [N]} w_{j} k(x_{j},.)
\end{equation}
 is the reconstruction operator\footnote{The reconstruction operator $\bm{\Phi}$ depends on the nodes $x_{j}$, although our notation doesn't reflect it for simplicity.}. In Section~\ref{subsec:unreg_opt_problem} we prove that \eqref{eq:unreg_opt_problem} almost surely has a unique solution $\hat{\bm{w}}$ and state our main result, an upper bound on the expected approximation error $\|\mu_{g} - \bm{\Phi}\hat{\bm{w}}\|^{2}_{\mathcal{F}}$ under the proposed Projection DPP. Section~\ref{subsec:dpp_quadrature_error} gives a sketch of the proof of this bound.

\subsection{Main result}
\label{subsec:unreg_opt_problem}
Assuming that nodes $(x_{j})_{j \in [N]}$ are known, we first need to solve the optimization problem \eqref{eq:unreg_opt_problem} that relates to problem~\eqref{eq:reg_kernel_opt_problem} without regularization ($\lambda = 0$).
Let $\bm{x} = (x_{1}, \dots, x_{N}) \in \mathcal{X}^{N}$, then
\begin{equation}
\| \mu_{g} - \bm{\Phi} \bm{w} \|_{\mathcal{F}}^{2} = \|\mu_{g}\|_{\mathcal{F}}^{2} - 2 \bm{w}^{\Tran} \mu_{g}(x_{j})_{j \in [N]} + \bm{w}^{\Tran} \bm{K}(\bm{x}) \bm{w},
\label{e:quadratic}
\end{equation}
where $\bm{K}(\bm{x}) = (k(x_{i},x_{j}))_{i,j \in [N]}$. The right-hand side of \eqref{e:quadratic} is quadratic in $\bm{w}$, so that the optimization problem \eqref{eq:unreg_opt_problem}
admits a unique solution $\hat{\bm{w}}$ if and only if $\bm{K}(\bm{x})$ is invertible. In this case, the solution is given by $\hat{\bm{w}} = \bm{K}(\bm{x})^{-1}\mu_{g}(x_{j})_{j \in [N]}$. A sufficient condition for the invertibility of $\bm{K}(\bm{x})$ is given in the following proposition.
\begin{proposition}\label{prop:K_N_non_singular}
Assume that the matrix
$\bm{E}(\bm{x}) = (e_{i}(x_{j}))_{ i,j \in [N]}$ is invertible, then $\bm{K}(\bm{x})$ is invertible.
\end{proposition}
The proof of Proposition~\ref{prop:K_N_non_singular} is given in Appendix~\ref{app:K_N_non_singular}.
Since the pdf \eqref{eq:pdpp_definition} of the projection DPP with kernel \eqref{e:kernel} is proportional to $\Det^2\bm{E}(\bm{x})$, the following corollary immediately follows.

\begin{corollary}
  Let $\bm{x} = \{x_{1}, \dots , x_{N}\}$ be a projection DPP with reference measure $\mathrm{d}\omega$ and kernel \eqref{e:kernel}. Then $\bm{K}(\bm{x})$ is a.s. invertible, so that \eqref{eq:unreg_opt_problem} has unique solution $\hat{\bm{w}} = \bm{K}(\bm{x})^{-1}\mu_{g}(x_{j})_{j \in [N]}$ a.s.
\label{c:regularization}
\end{corollary}
We now give our main result that uses nodes $(x_{j})_{j \in [N]}$ drawn from a well-chosen projection DPP.
\begin{theorem}\label{thm:main_theorem}
Let $\bm{x} = \{x_{1}, \dots , x_{N}\}$ be a projection DPP with reference measure $\mathrm{d}\omega$ and kernel \eqref{e:kernel}. Let $\hat{\bm{w}}$ be the unique solution to \eqref{eq:unreg_opt_problem} and define $\displaystyle \|g\|_{\mathrm{d}\omega,1} = \sum\limits_{n \in [N]} |\langle e_{n},g \rangle_{d\omega}|$. Assume that $\|g\|_{\mathrm{d}\omega} \leq 1$ and define $r_{N} = \sum\limits_{m \geq N+1} \sigma_{m}$, then
\begin{equation}\label{eq:main_result}
\EX_{\DPP} \|\mu_{g} - \bm{\Phi}\hat{\bm{w}}\|_{\mathcal{F}}^{2}  \leq
2\sigma_{N+1} +2\|g\|_{\mathrm{d}\omega,1}^{2}\left(   N r_{N} + \sum\limits_{\ell =2}^{N} \frac{\sigma_{1}}{\ell!^{2}}\left(\frac{Nr_{N}}{\sigma_{1}}\right)^{\ell} \right).
\end{equation}
\end{theorem}
In particular, if $Nr_{N} = o(1)$, then the right-hand side of \eqref{eq:main_result} is $N r_{N} + o(N r_{N})$. For example, take $\mathcal{X} = [0,1]$, $\mathrm{d}\omega$ the uniform measure on $\mathcal{X}$, and $\mathcal{F}$ the $s$-Sobolev space, then $\sigma_{m} = m^{-2s}$ \cite{BeTh11}. Now, if $s > 1$, the expected worst case quadrature error is bounded by $Nr_{N} = \mathcal{O}(N^{2-2s}) = o(1)$. Another example is the case of the Gaussian measure on $\mathcal{X} = \mathbb{R}$, with the Gaussian kernel. In this case $\sigma_{m} = \beta \alpha^{m}$ with $0 < \alpha < 1$ and $\beta >0$ \cite{RaWi06} so that $Nr_{N} = N\frac{\beta}{1-\alpha}\alpha^{N+1}=  o(1)$.

We have assumed that $\mathcal{F}$ is dense in $\mathbb{L}_{2}(\mathrm{d}\omega)$ but Theorem~\ref{thm:main_theorem} is valid also when $\mathcal{F}$ is finite-dimensional. In this case, denote $N_{0} = \dim \mathcal{F}$. Then, for $n > N_{0}$, $\sigma_{n} = 0$ and $r_{N_{0}} = 0$, so that \eqref{eq:main_result} implies
\begin{equation}
  \| \mu_{g} - \bm{\Phi}\hat{\bm{w}}\|_{\mathcal{F}}^{2} = 0 \:\: \text{a.s.}
\end{equation}
This compares favourably with herding, for instance, which comes with a rate in $\mathcal{O}(\frac{1}{N})$ for the quadrature based on herding with uniform weights \cite{BaLaOb12,ChWeSm10}.

The constant $\|g\|_{\mathrm{d}\omega,1}$ in $\eqref{eq:main_result}$ is the $\ell_1$ norm of the coefficients of projection of $g$ onto $\Span(e_{n})_{n \in [N]}$ in $\mathbb{L}_{2}(\mathrm{d}\omega)$. For example, for $g = e_{n}$, $\|g\|_{\mathrm{d}\omega,1} = 1$ if $n \in [N]$ and $\|g\|_{\mathrm{d}\omega,1} = 0$ if $n \geq N+1$. In the worst case, $\|g\|_{\mathrm{d}\omega,1} \leq \sqrt{N} \|g\|_{\mathrm{d}\omega} \leq \sqrt{N}$.
Thus, we can obtain a uniform bound for $\|g\|_{\mathrm{d}\omega}\leq 1$ in the spirit of Proposition~\ref{p:bach}, but with a supplementary factor $N$ in the upper bound in \eqref{eq:main_result}.

\subsection{Bounding the approximation error under the DPP}
\label{subsec:dpp_quadrature_error}

In this section, we give the skeleton of the proof of Theorem~\ref{thm:main_theorem}, referring to the appendices for technical details. The proof is in two steps. First, we give an upper bound for the approximation error $\|\mu_{g} - \bm{\Phi}\hat{\bm{w}}\|^{2}_{\mathcal{F}}$ that involves the maximal principal angle between the functional subspaces of $\mathcal{F}$
$$\mathcal{E}_N^{\mathcal{F}} = \Span(e_{n}^{\mathcal{F}})_{n \in [N]} \quad\text{ and }\quad \mathcal{T}(\bm{x}) = \Span(k(x_{j},.))_{j \in [N]}.$$
DPPs allow closed form expressions for the expectation of trigonometric functions of such angles; see \cite{BeBaCh18Sub} and Appendix~\ref{app:geometric_intuition} for the geometric intuition behind the proof.
The second step thus consists in developing the expectation of the bound under the DPP.

\subsubsection{Bounding the approximation error using principal angles}
\label{subsec:approximation_error_principal_angle}
 Let $\bm{x} =  (x_{1}, \dots , x_{N}) \in \mathcal{X}^{N} $ be such that $\Det \bm{E}(\bm{x}) \neq 0$. By Proposition~\ref{prop:K_N_non_singular}, $\bm{K}(\bm{x})$ is non singular and $\dim \mathcal{T}(\bm{x}) = N$. The optimal approximation error writes
\begin{equation}
\|\mu_{g} - \bm{\Phi}\hat{\bm{w}}\|^{2}_{\mathcal{F}} =\|\mu_{g} - \bm{\Pi}_{\mathcal{T}(\bm{x})}\mu_{g}\|^{2}_{\mathcal{F}},\label{e:finalTool}
\end{equation}
where $\bm{\Pi}_{\mathcal{T}(\bm{x})} = \bm{\Phi}(\bm{\Phi}^{*}\bm{\Phi})^{-1}\bm{\Phi}^{*}$ is the orthogonal projection onto $\mathcal{T}(\bm{x})$ with $\bm{\Phi}^{*}$ the dual\footnote{For $\mu \in \mathcal{F}$,$\bm{\Phi}^{*}\mu = (\mu(x_{j}))_{j \in [N]}$. $\bm{\Phi}^{*}\bm{\Phi}$ is an operator from $\mathbb{R}^{N}$ to $\mathbb{R}^{N}$ that can be identified with $\bm{K}(\bm{x})$.} of $\bm{\Phi}$.

In other words, \eqref{e:finalTool} equates the approximation error to $\|\bm{\Pi}_{\mathcal{T}(\bm{x})^{\perp}}\mu_g\|^{2}_{\mathcal{F}}$, where $\bm{\Pi}_{\mathcal{T}(\bm{x})^{\perp}}$ is the orthogonal projection onto $\mathcal{T}(\bm{x})^{\perp}$.
Now we have the following lemma.

\begin{lemma}\label{lemma:approximation_error_spectral_bound}
Assume that $\|g\|_{\mathrm{d}\omega} \leq 1$ then $\| \bm{\Sigma}^{-1/2} \mu_{g} \|_{\mathcal{F}} \leq 1$ and
\begin{equation}
	\|\bm{\Pi}_{\mathcal{T}(\bm{x})^{\perp}}\mu_{g}\|^{2}_{\mathcal{F}} \leq 2 \bigg( \sigma_{N+1} + \|g\|_{\mathrm{d}\omega,1}^2  \max\limits_{ n \in [N]} \sigma_{n}\|\bm{\Pi}_{\mathcal{T}(\bm{x})^{\perp}} e_{n}^{\mathcal{F}}\|_{\mathcal{F}}^{2} \bigg) .
\label{e:boundingOrthogonalProjection}
\end{equation}
\end{lemma}


Now, to upper bound the right-hand side of \eqref{e:boundingOrthogonalProjection}, we note that $\sigma_{n}\|\bm{\Pi}_{\mathcal{T}(\bm{x})^{\perp}} e_{n}^{\mathcal{F}}\|_{\mathcal{F}}^{2}$ is the product of two terms: $\sigma_{n}$ is a decreasing function of $n$ while $\|\bm{\Pi}_{\mathcal{T}(\bm{x})^{\perp}} e_{n}^{\mathcal{F}}\|_{\mathcal{F}}^{2}$ is the interpolation error of the eigenfunction $e_{n}^{\mathcal{F}}$, measured in the $\|.\|_{\mathcal{F}}$ norm.
We can bound the latter interpolation error uniformly in $n\in [N]$ using the geometric notion of maximal principal angle between $\mathcal{T}(\bm{x})$ and $\mathcal{E}^{\mathcal{F}}_{N} = \Span(e_{n}^{\mathcal{F}})_{ n \in [N]}$.
This maximal principal angle is defined through its cosine
\begin{equation}
	\cos^{2} \theta_{N}(\mathcal{T}(\bm{x}),\mathcal{E}^{\mathcal{F}}_{N}) = \inf\limits_{\substack{\bm{u} \in \mathcal{T}(\bm{x}), \bm{v} \in \mathcal{E}^{\mathcal{F}}_{N}\\ \|u\|_{\mathcal{F}} = 1, \|v\|_{\mathcal{F}} = 1}} \langle \bm{u}, \bm{v} \rangle_{\mathcal{F}}.
\end{equation}
Similarly, we can define the $N$ principal angles $\theta_{n}(\mathcal{T}(\bm{x}),\mathcal{E}^{\mathcal{F}}_{N}) \in \left[0, \frac{\pi}{2}\right]$ for $ n\in [N]$ between the subspaces $\mathcal{E}^{\mathcal{F}}_{N}$ and $\mathcal{T}(\bm{x})$. These angles quantify the relative position of the two subspaces. See Appendix~\ref{subsec:principalangles} for more details about principal angles.
Now, we have the following lemma.
\begin{lemma}\label{lemma:max_error_cos}
Let $\bm{x} = (x_{1}, \dots , x_{N}) \in \mathcal{X}^{N}$ such that $\Det \bm{E}(\bm{x}) \neq 0$. Then
\begin{equation}\label{eq:max_error_as_function_of_cos}
	\max_{ n \in [N]}\|\bm{\Pi}_{\mathcal{T}(\bm{x})^{\perp}} e_{n}^{\mathcal{F}}\|_{\mathcal{F}}^{2} \leq \frac{1}{\cos^{2} \theta_{N}(\mathcal{T}(\bm{x}),\mathcal{E}^{\mathcal{F}}_{N})} - 1  \leq \prod\limits_{n \in [N]}\frac{1}{\cos^{2} \theta_{n}(\mathcal{T}(\bm{x}),\mathcal{E}^{\mathcal{F}}_{N})} - 1.
\end{equation}
\end{lemma}
To sum up, we have so far bounded the approximation error by the geometric quantity in the right-hand side of \eqref{eq:max_error_as_function_of_cos}. Where projection DPPs shine is in taking expectations of such geometric quantities.

\subsubsection{Taking the expectation under the DPP}
\label{subsec:expectation_under_the_DPP}
The analysis in Section~\ref{subsec:approximation_error_principal_angle} is valid whenever $\Det \bm{E}(\bm{x}) \neq 0$. As seen in Corollary~\ref{c:regularization}, this condition is satisfied almost surely when $\bm{x}$ is drawn from the projection DPP of Theorem~\ref{thm:main_theorem}. Furthermore, the expectation of the right-hand side of \eqref{eq:max_error_as_function_of_cos} can be written in terms of the eigenvalues of the kernel $k$.
\begin{proposition}\label{prop:expected_value_of_product_of_cos}
Let $\bm{x}$ be a projection DPP with reference measure $\mathrm{d}\omega$ and kernel \eqref{e:kernel}. Then,
\begin{equation}\label{eq:expected_value_of_product_of_cos}
\EX_{\DPP} \prod\limits_{n \in [N]} \frac{1}{\cos^{2} \theta_{n}\bigg(\mathcal{T}(\bm{x}), \mathcal{E}^{\mathcal{F}}_{N} \bigg)}   = \sum\limits_{\substack{T \subset \mathbb{N}^{*} \\ |T| = N}} \frac{ \prod\limits_{t \in T}\sigma_{t}}{\prod\limits_{n \in [N]} \sigma_{n}} \:.
\end{equation}
\end{proposition}


The bound of Proposition~\ref{prop:expected_value_of_product_of_cos}, once reported in  Lemma~\ref{lemma:max_error_cos} and Lemma~\ref{lemma:approximation_error_spectral_bound}, already yields Theorem~\ref{thm:main_theorem} in the special case where $\sigma_{1} = \dots = \sigma_{N}$. This seems a very restrictive condition, but next Proposition~\ref{prop:kernel_perturbation_inequality} shows that we can always reduce the analysis to that case.
In fact, let the kernel $\tilde{k}$ be defined by
\begin{equation}\label{eq:tilde_k_kernel_definition}
\tilde{k}(x,y) = \sum\limits_{n \in [N]} \sigma_{1}e_{n}(x)e_{n}(y) + \sum\limits_{n \geq N+1} \sigma_{n}e_{n}(x)e_{n}(y) = \sum\limits_{n \in \mathbb{N}^{*}} \tilde{\sigma}_{n}e_{n}(x)e_{n}(y),
\end{equation}
and let $\tilde{\mathcal{F}}$ be the corresponding RKHS. Then one has the following inequality.
\begin{proposition}\label{prop:kernel_perturbation_inequality}
Let $ \tilde{\mathcal{T}}(\bm{x}) = \Span \left( \tilde{k}(x_{j},.) \right)_{j \in [N]}$ and $\bm{\Pi}_{\tilde{\mathcal{T}}(\bm{x})^{\perp}}$ the orthogonal projection onto $\tilde{\mathcal{T}}(\bm{x})^{\perp}$ in $(\tilde{\mathcal{F}}, \langle .,.\rangle_{\tilde{\mathcal{F}}})$. Then,
\begin{equation}\label{eq:kernel_perturbation_inequality}
	\forall n \in [N], \:\: \sigma_{n} \|\bm{\Pi}_{\mathcal{T}(\bm{x})^{\perp}} e_{n}^{\mathcal{F}}\|_{\mathcal{F}}^{2} \leq \sigma_{1}   \|\bm{\Pi}_{\tilde{\mathcal{T}}(\bm{x})^{\perp}} e_{n}^{\tilde{\mathcal{F}}}\|_{\tilde{\mathcal{F}}}^{2}.
\end{equation}
\end{proposition}
Simply put, capping the first eigenvalues of $k$ yields a new kernel $\tilde{k}$ that captures the interaction between the terms $\sigma_{n}$ and $\|\bm{\Pi}_{\mathcal{T}(\bm{x})^{\perp}} e_{n}^{\mathcal{F}}\|_{\mathcal{F}}^{2}$ such that we only have to deal with the term $\|\bm{\Pi}_{\tilde{\mathcal{T}}(\bm{x})^{\perp}} e_{n}^{\tilde{\mathcal{F}}}\|_{\tilde{\mathcal{F}}}^{2}$.
Combining  Proposition~\ref{prop:expected_value_of_product_of_cos} with Proposition~\ref{prop:kernel_perturbation_inequality} applied to the kernel $\tilde{k}$ yields Theorem~\ref{thm:main_theorem}.


\subsection{Discussion}
\label{s:discussion}
We have arbitrarily introduced a product in the right-hand side of \eqref{eq:max_error_as_function_of_cos}, which is a rather loose majorization. Our motivation is that the expected value of this symmetric quantity is tractable under the DPP. Getting rid of the product could make the bound much tighter. Intuitively, taking the upper bound in \eqref{eq:expected_value_of_product_of_cos} to the power $1/N$ results in a term in $\mathcal{O}(r_{N})$ for the RKHS $\tilde{\mathcal{F}}$. Improving the bound in \eqref{eq:expected_value_of_product_of_cos} would require a de-symmetrization by comparing the maximum of the $1/\cos^{2} \theta_{\ell}(\mathcal{T}(\bm{x}),\mathcal{E}^{\mathcal{F}}_{N})$ to their geometric mean.
An easier route than de-symmetrization could be to replace the product in \eqref{eq:max_error_as_function_of_cos} by a sum, but this is beyond the scope of this article.

In comparison with \cite{Bac17}, we emphasize that the dependence of our bound on the eigenvalues of the kernel $k$, via $r_{N}$, is explicit. This is in contrast with Proposition~\ref{p:bach} that depends on the eigenvalues of $\bm{\Sigma}$ through the degree of freedom $d_{\lambda}$ so that the necessary number of samples $N$ diverges when $\lambda \rightarrow 0$. On the contrary, our quadrature requires a finite number of points for $\lambda =0$. It would be interesting to extend the analysis of our quadrature in the regime $\lambda >0$.


\section{Numerical simulations}\label{s:numsims}


\subsection{The periodic Sobolev space and the Korobov space}\label{s:sobolev_numsim}
Let $\mathrm{d}\omega$ be the uniform measure on $\mathcal{X} = [0,1]$, and let the RKHS kernel be \citep{BeTh11}
$$k_{s}(x,y) = 1+ \sum\limits_{m \in \mathbb{N}^{*}} \frac{1}{m^{2s}} \cos(2\pi m(x-y)),$$
so that $\mathcal{F}=\mathcal{F}_{s}$ is the Sobolev space of order $s$ on $[0,1]$. Note that $k_{s}$ can be expressed in closed form using Bernoulli polynomials \cite{Wah90}. We take $g\equiv 1$ in \eqref{e:quadrature}, so that the mean element $\mu_{g} \equiv 1$. We compare the following algorithms: $(i)$ the quadrature rule DPPKQ we propose in Theorem~\ref{thm:main_theorem}, $(ii)$ the quadrature rule DPPUQ based on the same projection DPP but with uniform weights, implicitly studied in \cite{Joh97}, $(iii)$ the kernel quadrature rule \eqref{eq:reg_kernel_opt_problem} of \cite{Bac17}, which we denote LVSQ for \emph{leverage score quadrature}, with regularization parameter $\lambda \in \{0,0.1,0.2\}$ (note that the optimal proposal is $q_{\lambda}^* \equiv 1$), $(iv)$ herding with uniform weights \citep{BaLaOb12,ChWeSm10}, $(v)$ sequential Bayesian quadrature (SBQ) \citep{HuDu12} with regularization to avoid numerical instability, and $(vi)$ Bayesian quadrature on the uniform grid (UGBQ). We take $N \in [5,50]$. Figures~\ref{fig:a} and \ref{fig:b} show log-log plots of the worst case quadrature error w.r.t. $N$, averaged over 50 samples for each point, for $s \in \{1,3\}$.

 We observe that the approximation errors of all first four quadratures converge to $0$ with different rates. Both UGBQ and DPPKQ converge to $0$ with a rate of $\mathcal{O}(N^{-2s})$, which indicates that our $\mathcal{O}(N^{2-2s})$ bound in Theorem~\ref{thm:main_theorem} is not tight in the Sobolev case. Meanwhile, the rate of DPPUQ is $\mathcal{O}(N^{-2})$ across the three values of $s$: it does not adapt to the regularity of the integrands. This corresponds to the CLT proven in \cite{Joh97}.
  LVSQ without regularization converges to $0$ slightly slower than $\mathcal{O}(N^{-2s})$. Augmenting $\lambda$ further slows down convergence. Herding converges at an empirical rate of $\mathcal{O}(N^{-2})$, which is faster than the rate $\mathcal{O}(N^{-1})$ predicted by the theoretical analysis in \cite{BaLaOb12,ChWeSm10}. SBQ is the only one that seems to plateau for $s = 3$, although it consistently has the best performance for low $N$.
Overall, in the Sobolev case, DPPKQ and UGBQ have the best convergence rate. UGBQ --~known to be optimal in this case \cite{Boj81}~-- has a better constant.

Now, for a multidimensional example, consider the ``Korobov" kernel $k_{s}$ defined on $[0,1]^{d}$ by
\begin{equation}
\forall \bm{x}, \bm{y} \in [0,1]^{d}, \:\:k_{s,d}(\bm{x},\bm{y}) = \prod\limits_{i \in [d]} k_{s}(x_{i},y_{i}).
\end{equation}
We still take $g\equiv 1$ in \eqref{e:quadrature} so that $\mu_{g} \equiv 1$. We compare $(i)$ our DPPKQ, $(ii)$ LVSQ without regularization ($\lambda =0$), $(iii)$ the kernel quadrature based on the uniform grid UGBQ, $(iv)$ the kernel quadrature SGBQ based on the sparse grid from \citep{Smo63}, $(v)$ the kernel quadrature based on the Halton sequence HaltonBQ \citep{Hal64}. We take $N \in [5,1000]$ and $s =1$. The results are shown in Figure~\ref{fig:c}. This time, UGBQ suffers from the dimension with a rate in $\mathcal{O}(N^{-2s/d})$, while DPPKQ, HaltonBQ and LVSQ $(\lambda = 0)$ all perform similarly well. They scale as $\mathcal{O}((\log N)^{2s(d-1)} N^{-2s})$, which is a tight upper bound on $\sigma_{N+1}$, see \citep{Bac17} and Appendix~\ref{app:sup_num_sim}.
 SGBQ seems to lag slightly behind with a rate $\mathcal{O}((\log N)^{2(s+1)(d-1)} N^{-2s})$ 
 \citep{Hol08,Smo63}.

\subsection{The Gaussian kernel}\label{s:gaussian_numsim}
We now consider $\rm{d}\omega$ to be the Gaussian measure on $\mathcal{X} = \mathbb{R}$ along with the RKHS kernel $\displaystyle k_{\gamma}(x,y) = \exp[-(x-y)^{2}/2\gamma^{2}]$, and again $g\equiv 1$. Figure~\ref{fig:d} compares the empirical performance of DPPKQ to the theoretical bound of Theorem~\ref{thm:main_theorem}, herding, crude Monte Carlo with i.i.d. sampling from $\rm{d}\omega$, and sequential Bayesian Quadrature, where we again average over $50$ samples.
We take $N \in [5,50]$ and $\gamma = \frac{1}{2}$. Note that, this time, only the $y$-axis is on the log scale for better display, and that LVSQ is not plotted since we don't know how to sample from $q_\lambda$ in \eqref{e:proposalBach} in this case.
We observe that the approximation error of DPPKQ converges to $0$ as $\mathcal{O}(\alpha^{N})$, while the discussion below Theorem~\ref{thm:main_theorem} let us expect a slightly slower $\mathcal{O}(N\alpha^{N})$. Herding improves slightly upon Monte Carlo that converges as $\mathcal{O}(N^{-1})$. Similarly to Sobolev spaces, the convergence of sequential Bayesian quadrature plateaus even if it has the smallest error for small $N$.
We also conclude that DPPKQ is a close runner-up to SBQ and definitely takes the lead for large enough $N$.

\begin{figure}
    \centering
\begin{subfigure}[]{\twofig}\label{fig:sobolev_1}
\includegraphics[width=\textwidth]{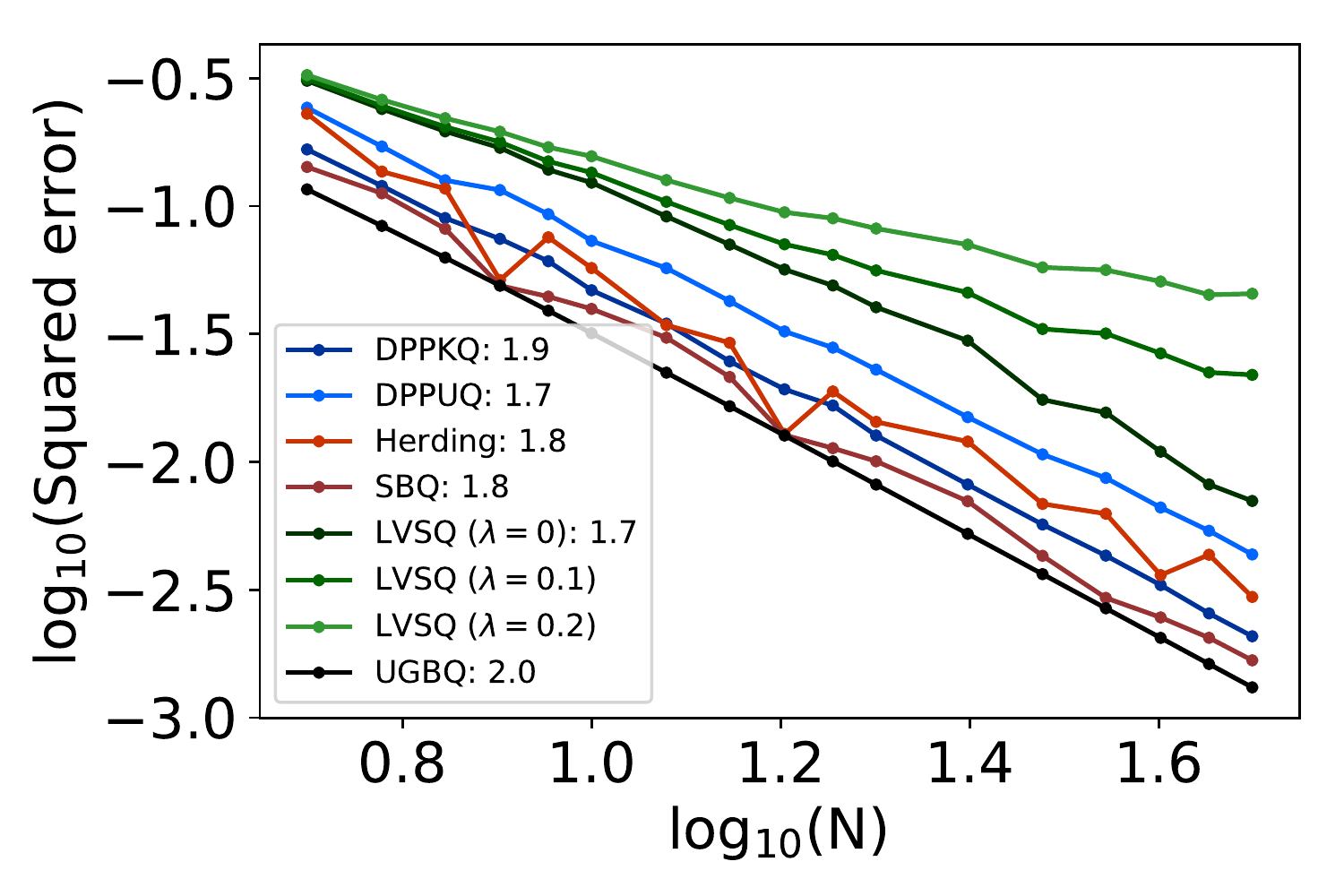}
\caption{\label{fig:a} Sobolev space, $d=1$, $s=1$}
\end{subfigure}
\hfill
\begin{subfigure}[]{\twofig}
\includegraphics[width=\textwidth]{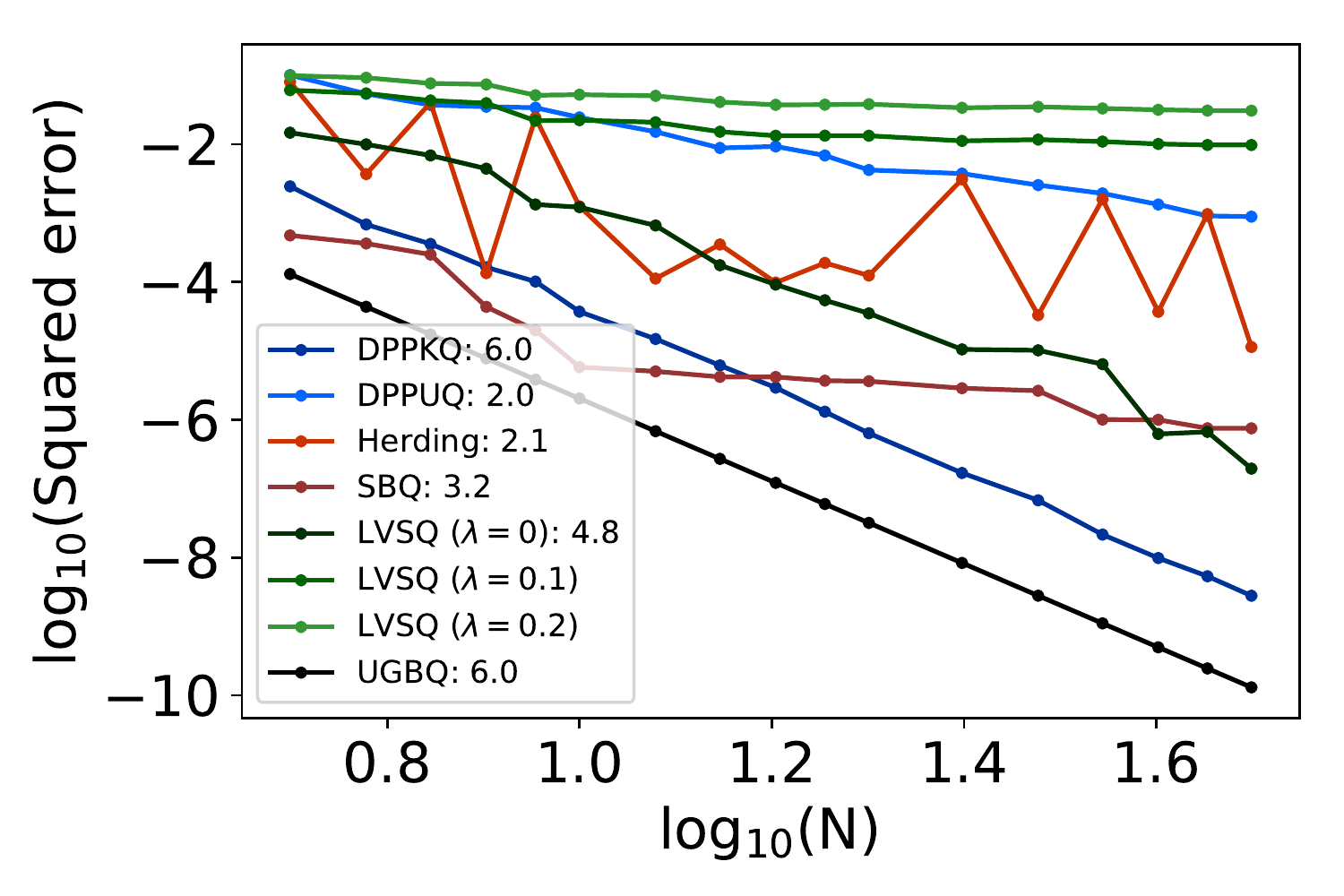}
\caption{\label{fig:b} Sobolev space, $d=1$, $s=3$}
\end{subfigure}
\\
\begin{subfigure}[]{\twofig}
\includegraphics[width=\textwidth]{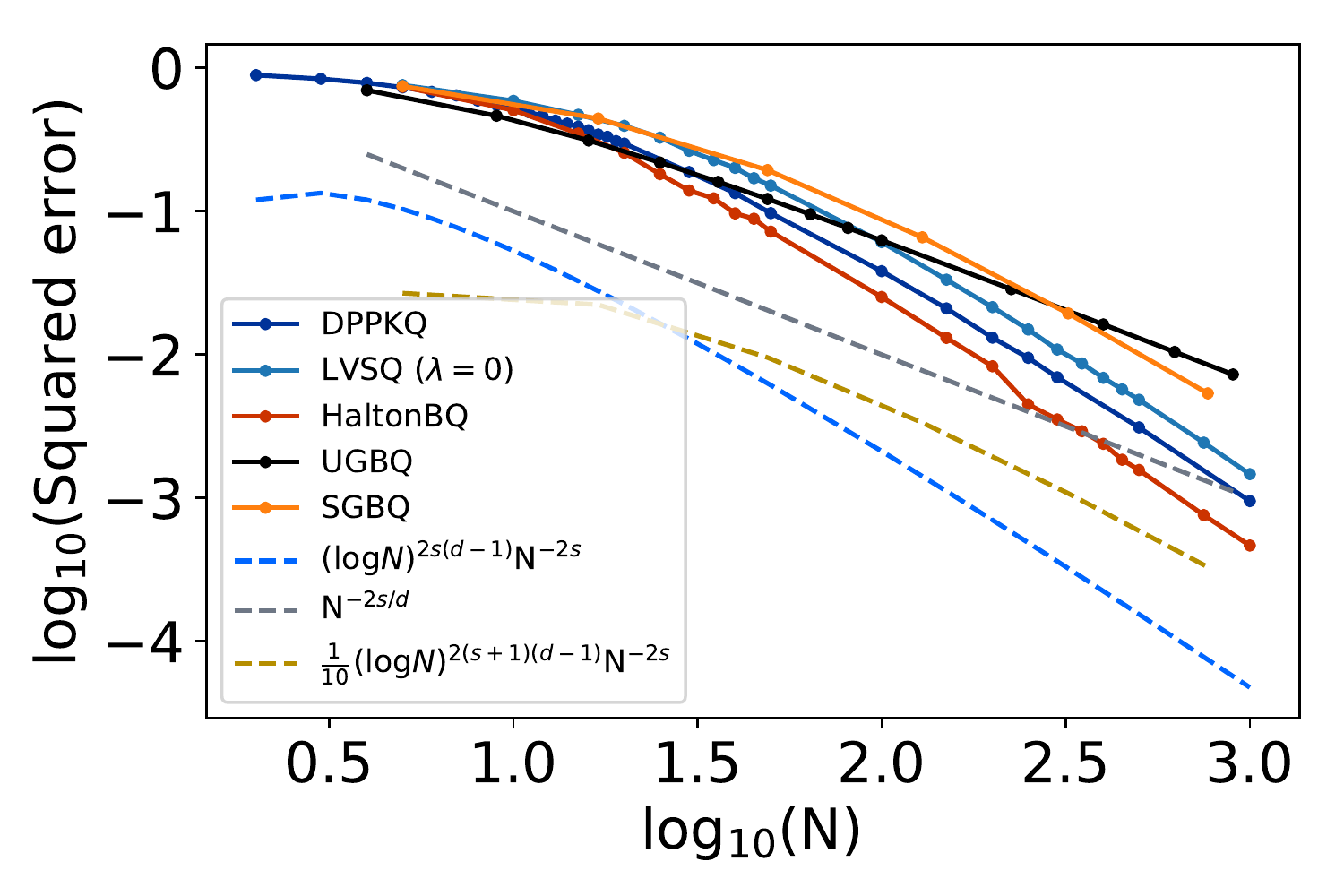}
\caption{\label{fig:c} Korobov space, $d=2$, $s=1$}
\end{subfigure}
\hfill
\begin{subfigure}[]{\twofig}
\includegraphics[width=\textwidth]{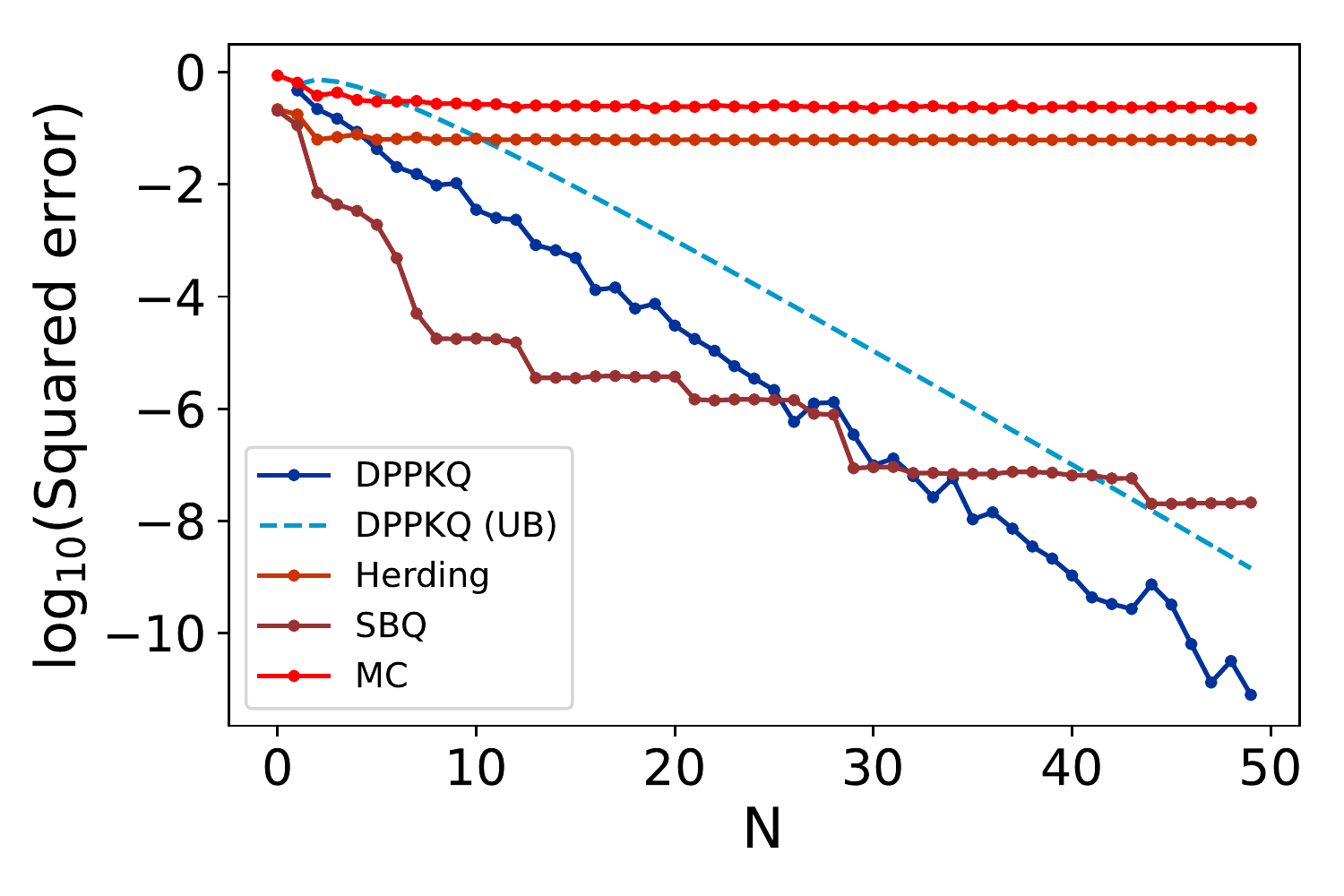}
\caption{\label{fig:d} Gaussian kernel, $d=1$}
\end{subfigure}

\caption{Squared error vs. number of nodes $N$ for various kernels.
\label{fig:results}}
\end{figure}

\vspace*{-2mm}

\section{Conclusion}
In this article, we proposed a quadrature rule for functions living in a RKHS. The nodes are drawn from a DPP tailored to the RKHS kernel, while the weights are the solution to a tractable, non-regularized optimization problem.
We proved that the expected value of the squared worst case error is bounded by a quantity that depends on the eigenvalues of the integral operator associated to the RKHS kernel, thus preserving the natural feel and the generality of the bounds for kernel quadrature \cite{Bac17}.
Key intermediate quantities further have clear geometric interpretations in the ambient RKHS.
Experimental comparisons suggest that DPP quadrature favourably compares with existing kernel-based quadratures. In specific cases where an optimal quadrature is known, such as the uniform grid for 1D periodic Sobolev spaces, DPPKQ seems to have the optimal convergence rate. However, our generic error bound does not reflect this optimality in the Sobolev case, and must thus be sharpened.

We have discussed room for improvement in our proofs. Further work should also address exact sampling algorithms, which do not exist yet when the spectral decomposition of the integral operator is not known. Approximate algorithms would also suffice, as long as the error bound is preserved.

\subsubsection*{Acknowledgments}
We acknowledge support from ANR grant BoB (ANR-16-CE23-0003) and région Hauts-de-France.
We also thank Adrien Hardy and the reviewers for their detailed and insightful comments.

\bibliography{bibliography,stats,learning}

\newpage

\appendix

\section{Implementation details}
In this section, we give details on the repulsion kernels in each example of the main paper, and explain how we sampled from the corresponding DPPs. In short, we relied on matrix models for univariate cases, and vanilla DPP sampling \cite{HKPV06} for multivariate settings.

\subsection{The one-dimensional periodic Sobolev space}
Consider the kernel $k_{s} : [0,1]\times [0,1] \rightarrow \mathbb{R}_{+}$ defined by
\begin{equation}
k_{s}(x,y) = 1+ \sum\limits_{m \in \mathbb{N}^{*}} \frac{1}{m^{2s}} \cos(2\pi m(x-y)).
\end{equation}
The Mercer decomposition of $k_{s}$ associated to the uniform measure $\mathrm{d}\omega$ on $[0,1]$ writes
\begin{equation}
k_{s}(x,y) = \sum\limits_{k \in \mathbb{Z}} \frac{1}{\max(1,|k|)^{2s}} \,e^{2\pi i k x}e^{-2\pi i k y}.
\end{equation}
The corresponding repulsion kernel is
\begin{equation}
\KDPP(x,y)  = e^{\pi i N(x-y) } \sum\limits_{m = -N/2}^{N/2}e^{2 \pi i m x} e^{-2 \pi i m y} = \sum\limits_{m = 0}^{N} e^{2\pi i m x}e^{-2\pi i m y},
\end{equation}
if $N$ is even and
\begin{equation}
\KDPP(x,y)  = e^{\pi i (N-1)(x-y) } \sum\limits_{m = -(N-1)/2}^{(N+1)/2}e^{2 \pi i m x} e^{-2 \pi i m y} = \sum\limits_{m = 0}^{N} e^{2\pi i m x}e^{-2\pi i m y},
\end{equation}
if not.
The projection DPP with kernel $\KDPP$ and reference measure $\mathrm{d}\omega$ can be sampled through a matrix model. Indeed this DPP is also the distribution of the arguments (normalized by $2\pi$) of the eigenvalues of a random unitary matrix drawn from the Haar measure on $\mathbb{U}_{N+1}$ \citep{Wey46}. Sampling such matrices can be done, e.g., using the QR decomposition of a matrix with i.i.d. unit complex Gaussians as coefficients \citep{Mez06}.

\subsection{The one-dimensional Gaussian kernel}

Let $k_{\gamma}:\mathbb{R}\times \mathbb{R}\rightarrow \mathbb{R}_{+}$ and the reference measure $\mathrm{d}\omega$ be defined by
\begin{equation}
k_{\gamma}(x,y) = e^{-\frac{(x-y)^{2}}{2\gamma^{2}}}, \quad \mathrm{d}\omega(x) = \frac{1}{\sqrt{2 \pi} \sigma}e^{-\frac{x^{2}}{2\sigma^2}}.
\end{equation}
For notational convenience, we further let
\begin{equation}
a = \frac{1}{4 \sigma^{2}}, \quad b = \frac{1}{2 \gamma^{2}}, \quad c = \sqrt{a^{2} + 2ab},
\end{equation}
and
\begin{equation}
A = a + b + c,  \quad B = b/A.
\end{equation}
Now, the Mercer decomposition of $k_{\gamma}$ reads \cite{Ras03}
\begin{equation}
k_{\gamma}(x,y) = \sum\limits_{m \in \mathbb{N}} \sigma_{m}e_{m}(x)e_{m}(y),
\end{equation}
where
\begin{equation}
\sigma_{m} = \sqrt{\frac{2a}{A}} B^{m}, \quad e_{m} = e^{-(c-a)x^{2}}H_{m}(\sqrt{2c}x),
\end{equation}
and $H_{m}$ is the $m$-th Hermite polynomial (i.e., orthonormal polynomials for the pdf of a unit Gaussian). Now, denote
\begin{equation}
\tilde{e}_{m}(x) := H_{m}(\sqrt{2c}x),
\end{equation}
and the measure
\begin{equation}
\mathrm{d}\tilde{\omega} = \frac{1}{\sqrt{2 \pi} \sigma}e^{-2cx^{2}}.
\end{equation}
The rescaled polynomials $(\tilde{e}_{m})_{m \in \mathbb{N}}$ are orthonormal with respect to the measure $\mathrm{d}\tilde{\omega}$. Moreover, for $x \in \mathbb{R}$,
\begin{align}
 e_{m}(x)e_{m'}(x) \mathrm{d}\omega(x)
 & = e^{-(c-a)x^{2}}H_{m}(\sqrt{2c}x)e^{-(c-a)x^{2}}H_{m'}(\sqrt{2c}x) \frac{1}{\sqrt{2 \pi} \sigma}e^{-2a x^{2}}\\
 & = H_{m}(\sqrt{2c}x)H_{m'}(\sqrt{2c}x) e^{-2c x^{2}}\\
 & = \tilde{e}_{m}(x)\tilde{e}_{m'}(x) \mathrm{d}\tilde{\omega}(x).
\end{align}
Thus, for $\bm{x} = (x_{i})_{i \in [N]} \in \mathbb{R}^{N}$, we have
\begin{equation}
\Det \bm{E}(\bm{x}) \otimes_{i \in [N]} \mathrm{d}\omega(x_{i}) = \Det \tilde{\bm{E}}(\bm{x}) \otimes_{i \in [N]} \mathrm{d}\tilde{\omega}(x_{i}).
\end{equation}
In other words, the projection DPP associated to the orthonormal family $(e_{n})_{n \in [N]}$ and the reference measure $\mathrm{d}\omega$ is equivalent to the projection DPP associated to the orthonormal family $(\tilde{e}_{n})_{n \in [N]}$ and the reference measure $\mathrm{d}\tilde{\omega}$. The latter DPP is known to be the distribution of the eigenvalues of a symmetrized matrix with i.i.d. Gaussian entries \citep{Meh04}, which is easily implemented.

\subsection{The case of a tensor product of RKHSs}
\label{s:multivariate}

We consider the case where $\mathcal{F}$ writes as a tensor product of RKHSs, with the associated kernel
%
\begin{equation}
k(\bm{x}, \bm{y}) = \prod\limits_{\ell \in [L]} k_{\ell}(x_{\ell},y_{\ell}),
\end{equation}
with $k_{\ell}: \mathcal{X}_{\ell}\times \mathcal{X}_{\ell} \rightarrow \mathbb{R}$.

\subsubsection{The multivariate integral operator}
The integral operator becomes
\begin{equation}
\bm{\Sigma} f(\bm{x}) = \int_{\mathcal{X}} f(\bm{x})k(\bm{x},\bm{y}) \mathrm{d}\omega(\bm{y}) = \prod\limits_{\ell \in [L]} \int_{\mathcal{X}_{\ell}} f_{\ell}(x_{\ell}) k_{\ell}(x_{\ell},y_{\ell}) \mathrm{d}\omega_{\ell}(y_{\ell}).
\end{equation}

In the main paper, we considered for instance the Korobov space $\mathbb{K}^{d}_{s}([0,1])$, defined as the tensor product of unidimensional periodic Sobolev spaces.
Note that an element $f$ of $\mathbb{K}^{d}_{s}([0,1])$ is such that
$$ \frac{\partial^{u_{1} + \dots + u_{d}}}{\partial{x_{1}^{u_{1}}} \dots \partial{x_{d}^{u_{d}}}}f \in \mathbb{L}_{2}([0,1]^{d}), \quad \forall u_{1}, \dots u_{d} \in \{0, \dots , s \}.$$
This implies that $\mathbb{K}^{d}_{s}([0,1])$ is included in the multidimensional Sobolev space, which corresponds to the same requirement, but only for multi-indices such that $\Vert u_{i}\Vert_1 \leq s$.
Another example, featured in this supplementary material, is the multidimensional Gaussian space associated to the Gaussian kernel on $\mathcal{X}_\ell = \mathbb{R}$ and the multidimensional Gaussian measure. In this case, the kernel $k_{\gamma,d}$ can be written as the tensor product of the Gaussian kernels on $\mathbb{R}$:
\begin{equation}
\forall \bm{x}, \bm{y} \in \mathbb{R}^{d}, \:\: k_{\gamma,d}(\bm{x},\bm{y}) = \prod\limits_{i \in [d]} k_{\gamma}(x_{i},y_{i}).
\end{equation}

In general, the eigenpairs of the integral operator are the tensor products of the eigenpairs of the integral operators $\Sigma_{\ell}$ corresponding to the spaces $\mathcal{F}_{\ell}$ and measures $\mathrm{d}\omega_{\ell}$. In other words, for $\bm{u} \in (\mathbb{N}\smallsetminus \{0\})^{d}$,
\begin{equation}
\bm{\Sigma} \otimes_{i \in [d]} e_{\ell,u_{i}} = \prod\limits_{i \in [d]} \sigma_{\ell,u_{i}} \otimes_{i \in [d]} e_{\ell,u_{i}}.
\end{equation}
\subsubsection{Fixing an order on multi-indices}
The definition of the projection DPP and its kernel $\frak K$ now require that we fix an order on multi-indices. We choose an order $\prec$ that keeps eigenvalues decreasing, as in the univariate case where $\sigma_{1} \geq \sigma_{2}\geq\dots$. Whenever the univariate eigenvalues take the form $\sigma_{i} = \frac{1}{(1+i)^{\eta}}$ with $\eta >0$, such as in the Korobov case, it holds
\begin{align}
\prod\limits_{i \in [d]} \sigma_{u_{i}} \leq \prod\limits_{i \in [d]} \sigma_{v_{i}}  & \Leftrightarrow \left(\prod\limits_{i \in [d]} \frac{1}{1+u_{i}} \right)^{\eta} \leq \left(\prod\limits_{i \in [d]} \frac{1}{1+v_{i}} \right)^{\eta} \\
& \Leftrightarrow \sum\limits_{i \in [d]} \log(1+v_{i})  \leq \sum\limits_{i \in [d]} \log(1+u_{i}) .
\label{e:orderSobolev}
\end{align}
Now, if the eigenvalues takes the form $\displaystyle \sigma_{i} = {\eta^{-i}}$, with $\eta >1$, as in the Gaussian case,
\begin{align}
\prod\limits_{i \in [d]} \sigma_{u_{i}} \leq \prod\limits_{i \in [d]} \sigma_{v_{i}}  & \Leftrightarrow \prod\limits_{i \in [d]} \frac{1}{\eta^{u_{i}}} \leq \prod\limits_{i \in [d]} \frac{1}{\eta^{v_{i}}} \\
& \Leftrightarrow \left(\sum\limits_{i \in [d]} v_{i} \right)\log \eta \leq \left(\sum\limits_{i \in [d]} u_{i} \right)\log \eta \\
& \Leftrightarrow \sum\limits_{i \in [d]} v_{i}  \leq \sum\limits_{i \in [d]} u_{i} .
\label{e:orderGaussian}
\end{align}
In the multivariate Korobov and the Gaussian cases, we thus define in this work $\bm{u}\prec \bm{v}$ as \eqref{e:orderSobolev} or \eqref{e:orderGaussian}, respectively.

Now, for $N \in \mathbb{N}$, let $\textbf{u}_N = (\textbf{u}_{1,N},\dots,\textbf{u}_{d,N})\in\mathbb{N}^d$ be the $N$-th multi-index according to $\prec$. The repulsion kernel is defined as
\begin{equation}
\KDPP(x,y) = \sum\limits_{n \in [N]} \prod\limits_{i \in [d]}e_{\textbf{u}_{i, N}} (x_{i}) e_{\textbf{u}_{i, N}} (y_{i}), \quad x,y\in\mathbb{R}^d.
\end{equation}
We sampled from the corresponding DPP using the generic sampling algorithm in \cite{HKPV06}, using the uniform and Gaussian distributions as proposal in the successive rejection sampling steps for the Korobov and Gaussian cases, respectively.

\section{Supplementary simulations}\label{app:sup_num_sim}
In this section, we give more plots of the convergence of the quadrature error. Before that, we experimentally assess whether the upper bounds given in \citep{Bac17} are sharp. The author proved upper bounds for $\sigma_{N+1}$ in cases where the univariate eigenvalues $\sigma_{\ell,N}$ decrease polynomially or geometrically in $N$. In particular, for the Korobov spaces of dimension $d$ and regularity $s$, we have
\begin{equation}
\sigma_{N+1} = \mathcal{O}\left((\log N)^{2s(d-1)}N^{-2s} \right).
\label{eq:sigma_N_sobolev}
\end{equation}
For the Gaussian RKHS in dimension $d$, it holds
\begin{equation}\label{eq:sigma_N_1_gaussian_1}
\sigma_{N+1} = \mathcal{O}\left(\beta^{d}e^{-\delta d!^{1/d}N^{1/d}} \right),
\end{equation}
where $\beta \in ]0,1[$ and $\delta > 0$ are constants depending on the scale parameters of the kernel and the measure $\mathrm{d}\omega$.
In our experiments, we compare the errors of various quadratures to the two rates \eqref{eq:sigma_N_sobolev} and \eqref{eq:sigma_N_1_gaussian_1}. We mean these rates to be proxies for plotting $\sigma_{\textbf{u}_N}$, where $\textbf{u}_N$ refers to the order introduced in Section~\ref{s:multivariate}. Figure~\ref{f:rates} shows that in the Korobov case, the rate \eqref{eq:sigma_N_sobolev} is indeed close to the corresponding eigenvalue for large values of $N$. The value of $(\log N)^{2s(d-1)}N^{-2s}$ could be larger than $1$ for $d \geq 4$ and small values of $N$. As for the Gaussian case, Figure~\ref{f:rates} shows that the rate \eqref{eq:sigma_N_1_gaussian_1} is also close to the corresponding eigenvalue for all values of $N$.  
\begin{figure}[]
    \centering
\includegraphics[width= 0.45\textwidth]{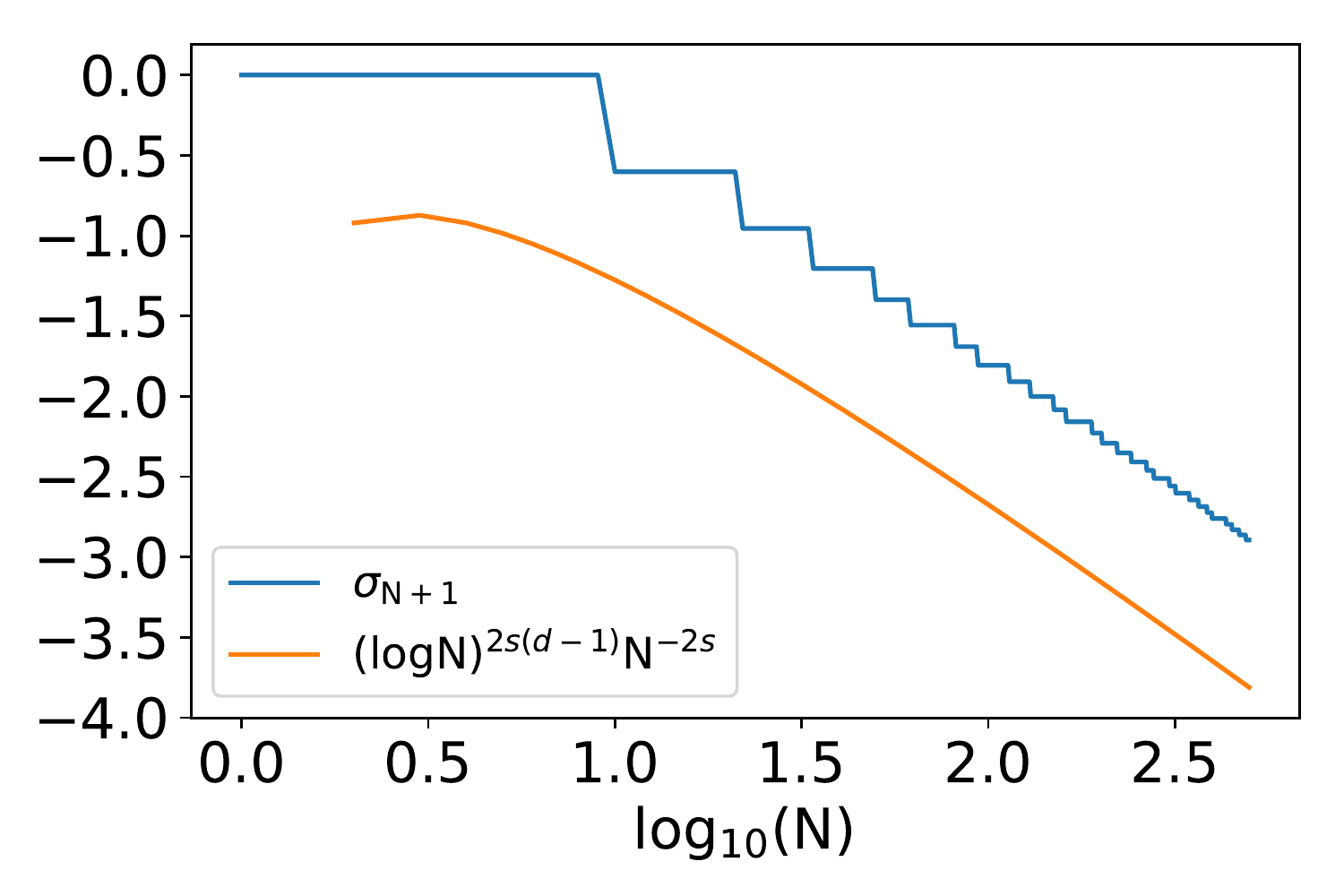} 
\includegraphics[width= 0.45\textwidth]{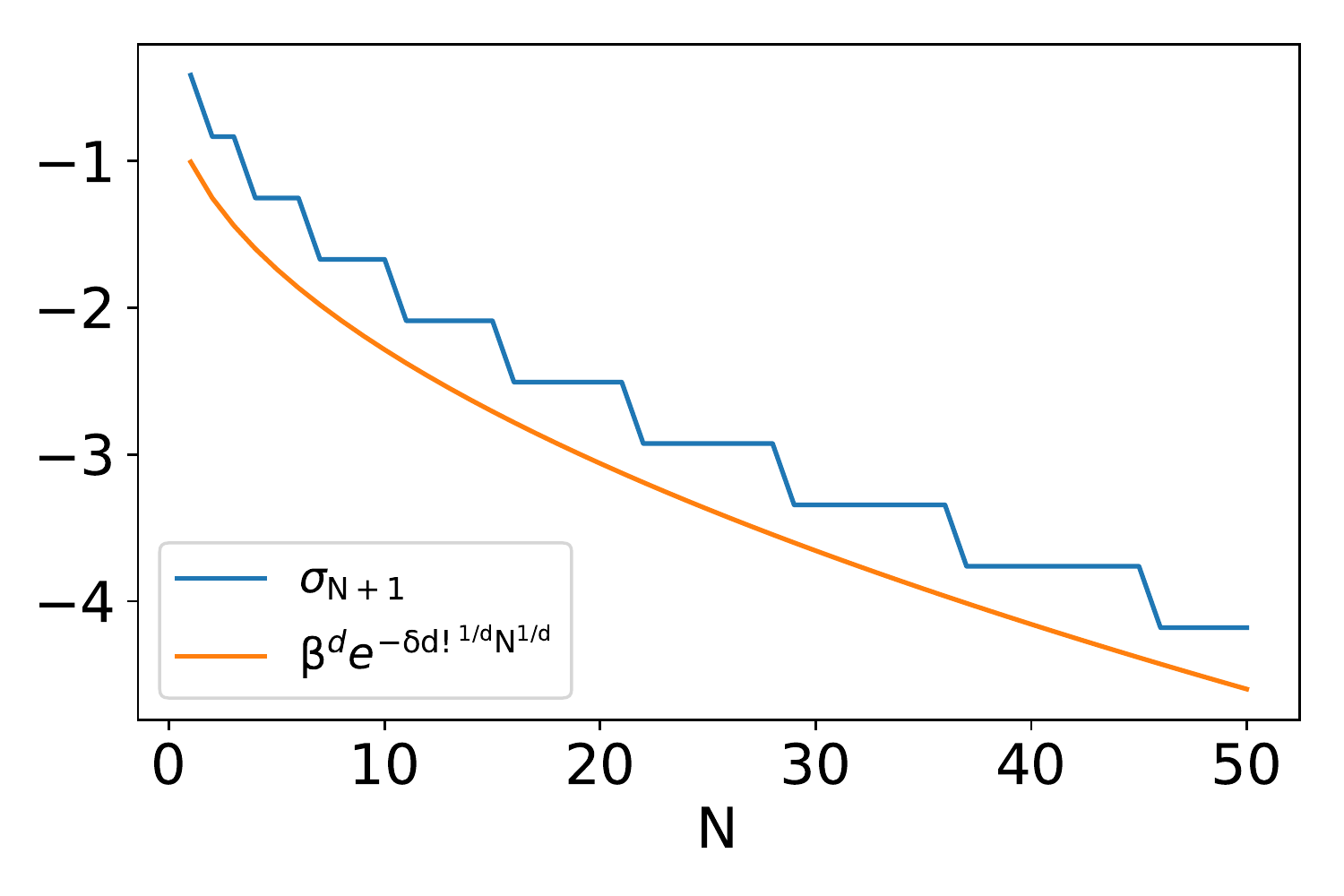} \\
\includegraphics[width= 0.45\textwidth]{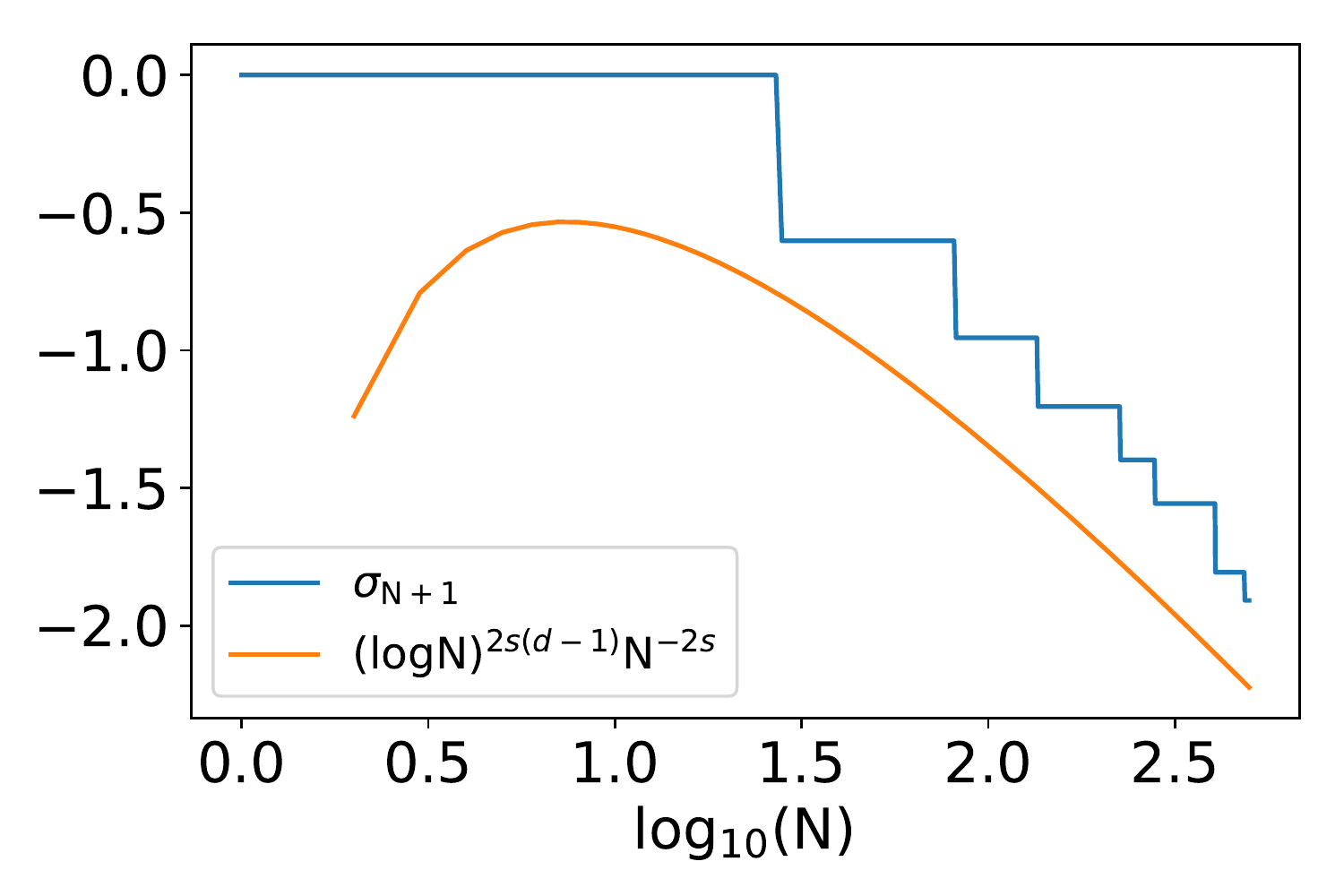} 
\includegraphics[width= 0.45\textwidth]{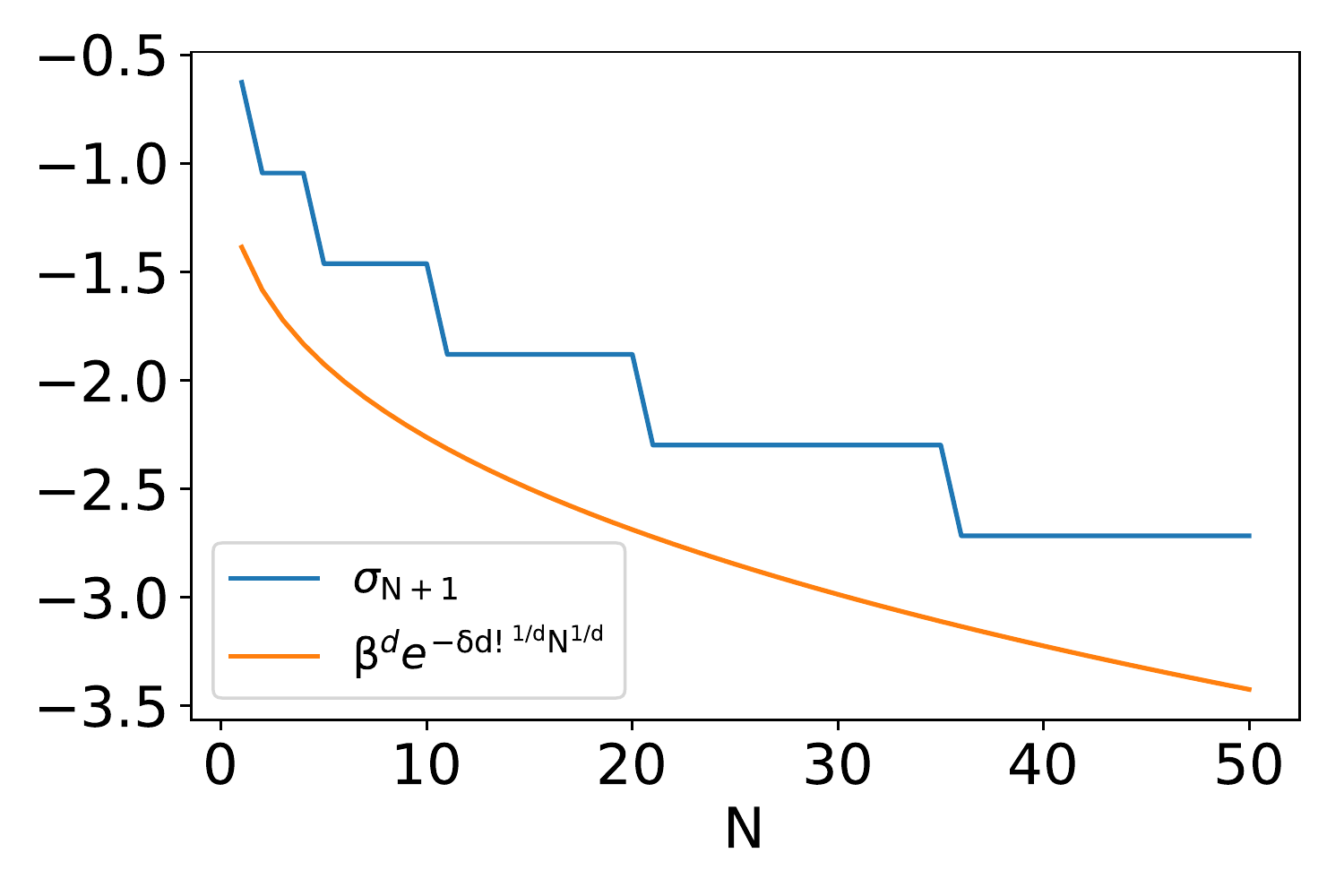}\\
\includegraphics[width= 0.45\textwidth]{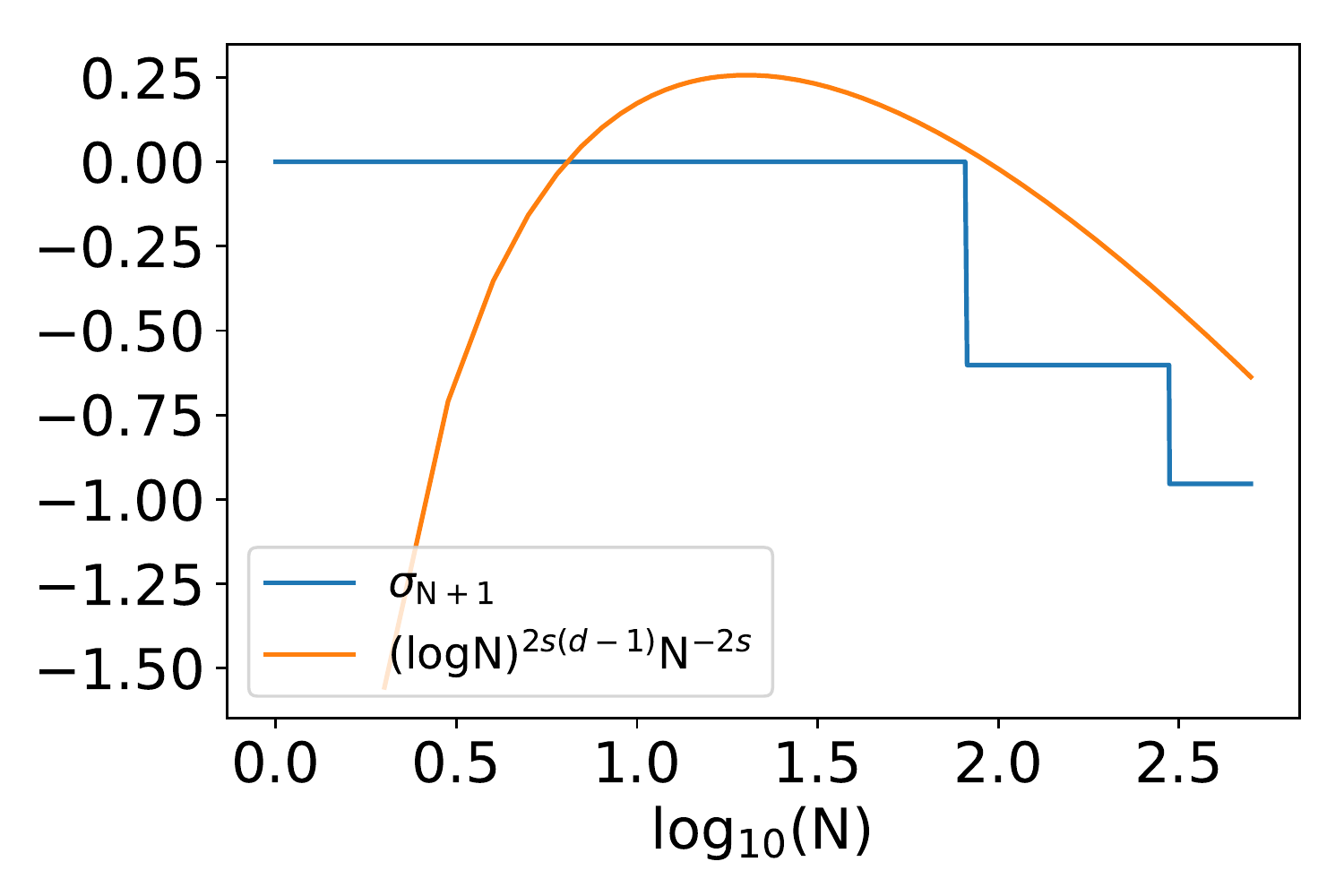} 
\includegraphics[width= 0.45\textwidth]{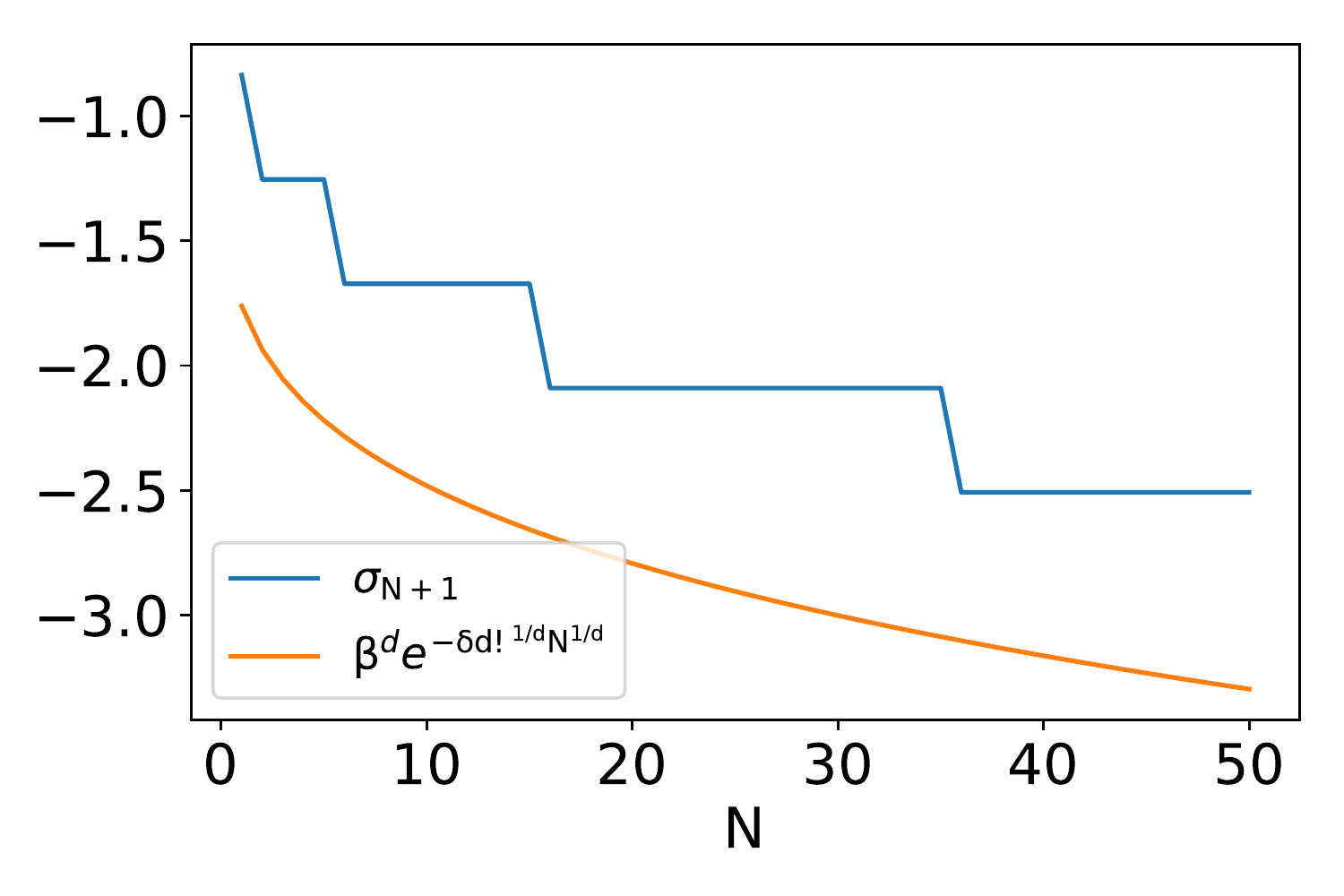}
\\
\caption{(Left): comparison of $\sigma_{N+1}$ in the Korobov case according to the spectral order and $(\log N)^{2s(d-1)}N^{-2s}$ for $d \in \{2,3,4\}$ and $s=1$, (Right): comparison of $\sigma_{N+1}$ in the Gaussian case according to the spectral order and $\beta^{d}e^{-\delta d!^{1/d}N^{1/d}}$ for $d \in \{2,3,4\}$ and $\gamma = 1$.}
\label{f:rates}
\end{figure}

\subsection{The multi Fourier ensemble and Korobov RKHS}
We consider the case of Korobov spaces with $d \in \{2,3\}$ and $s \in \{1,2\}$ and compare the quadrature error of the same algorithms as in \ref{s:sobolev_numsim}. The results are compiled in Figure~\ref{fig:KorobovResult}. The numerical simulations confirm the dependencies of the theoretical bounds of the different algorithms to the dimension $d$ and the regularity $s$. In particular, UGBQ have better performance for high values of $s$ and low values of $N$ while its asymptotic behaviour is still the same $\mathcal{O}(N^{-2s/d})$. Moreover, the empirical rate of SGBQ is similar to its theoretical rate $\mathcal{O}((\log N)^{2(s+1)(d-1)} N^{-2s})$ \citep{Hol08,Smo63}. Finally, the rate $\mathcal{O}((\log N)^{2s(d-1)} N^{-2s})$ is confirmed also for the algorithms DPPKQ, LVSQ $(\lambda = 0)$ and HaltonBQ.
\begin{figure}[]
    \centering
\includegraphics[width= 0.45\textwidth]{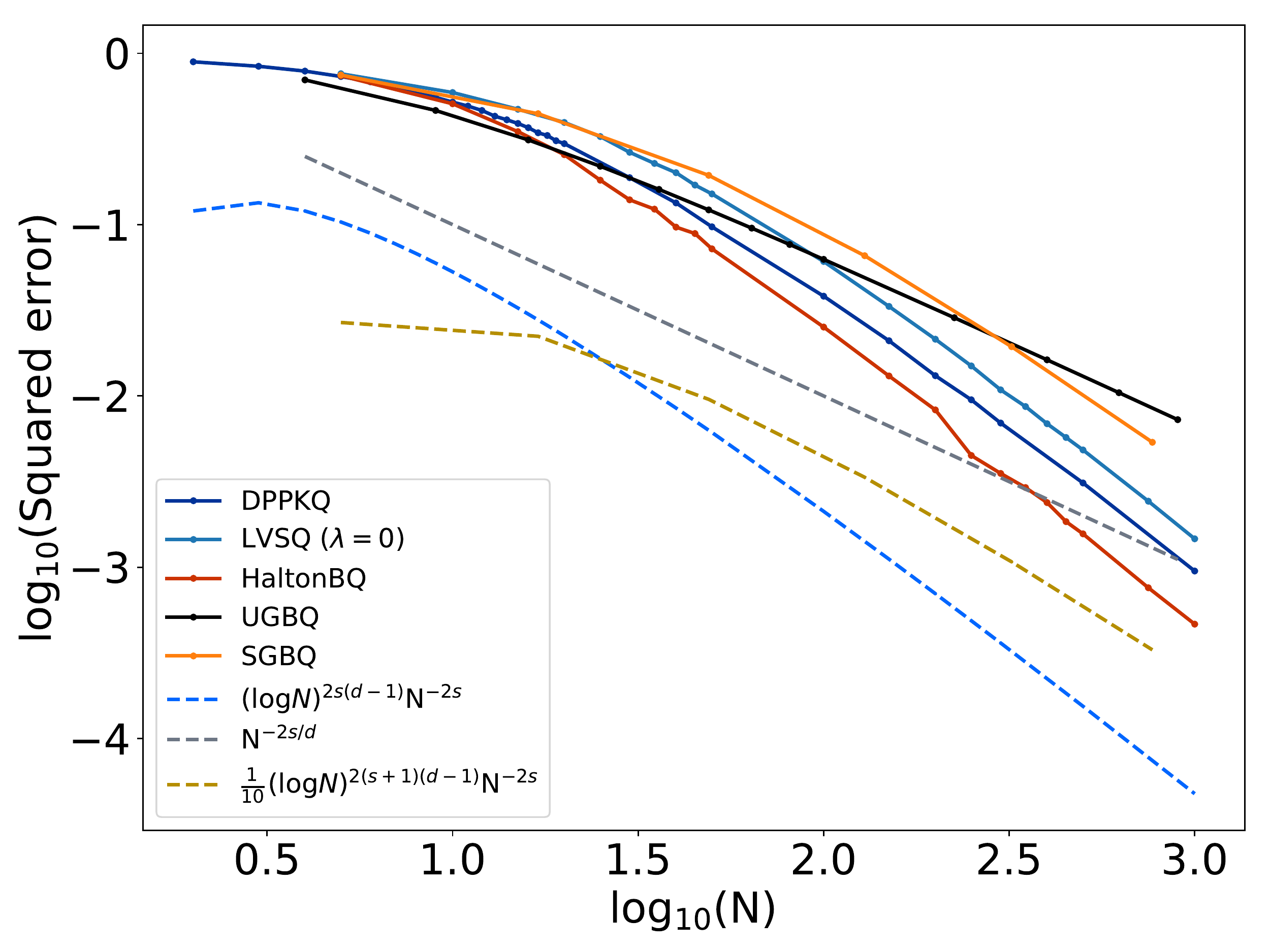}
\includegraphics[width= 0.45\textwidth]{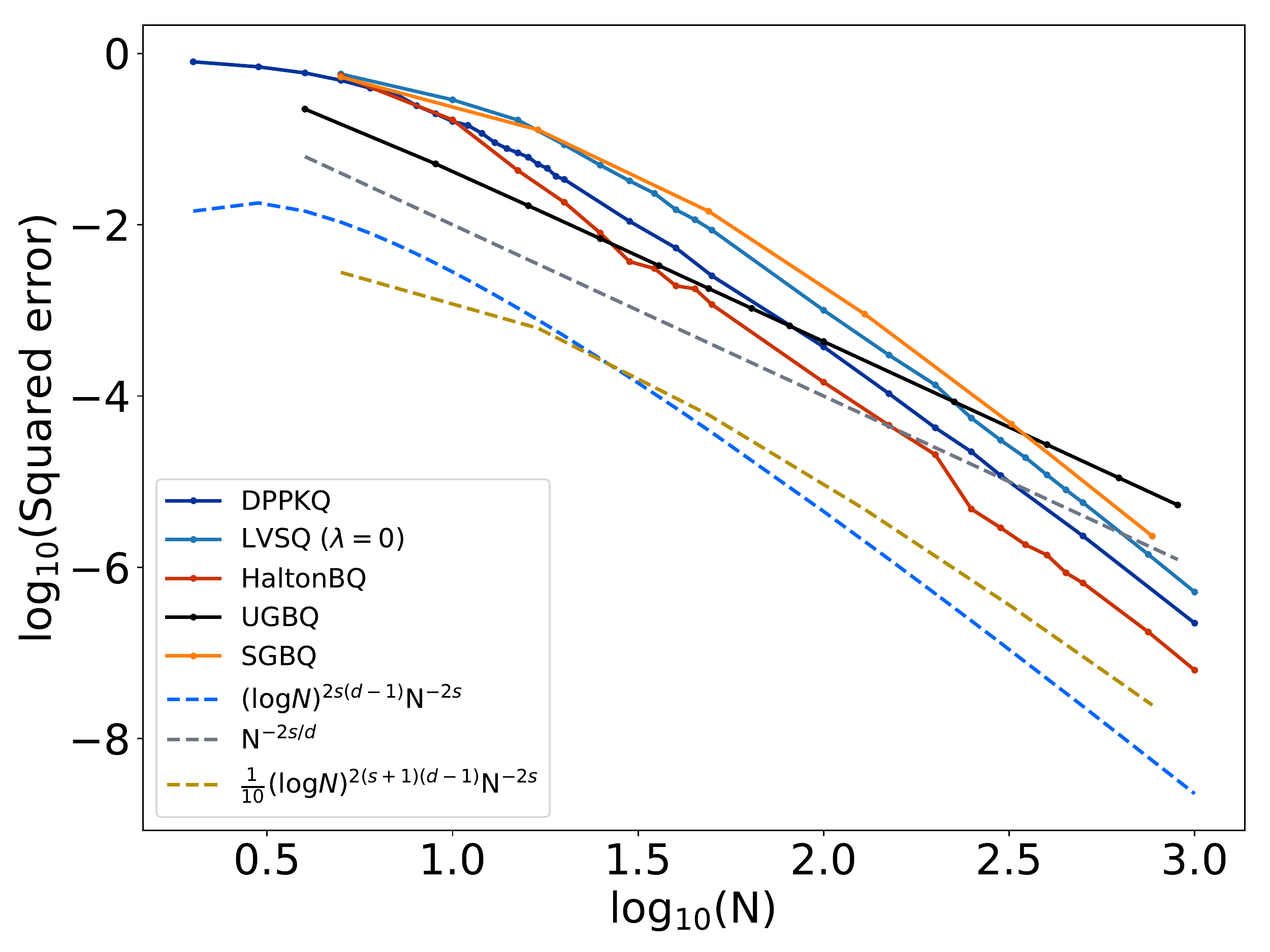} \\

\includegraphics[width= 0.45\textwidth]{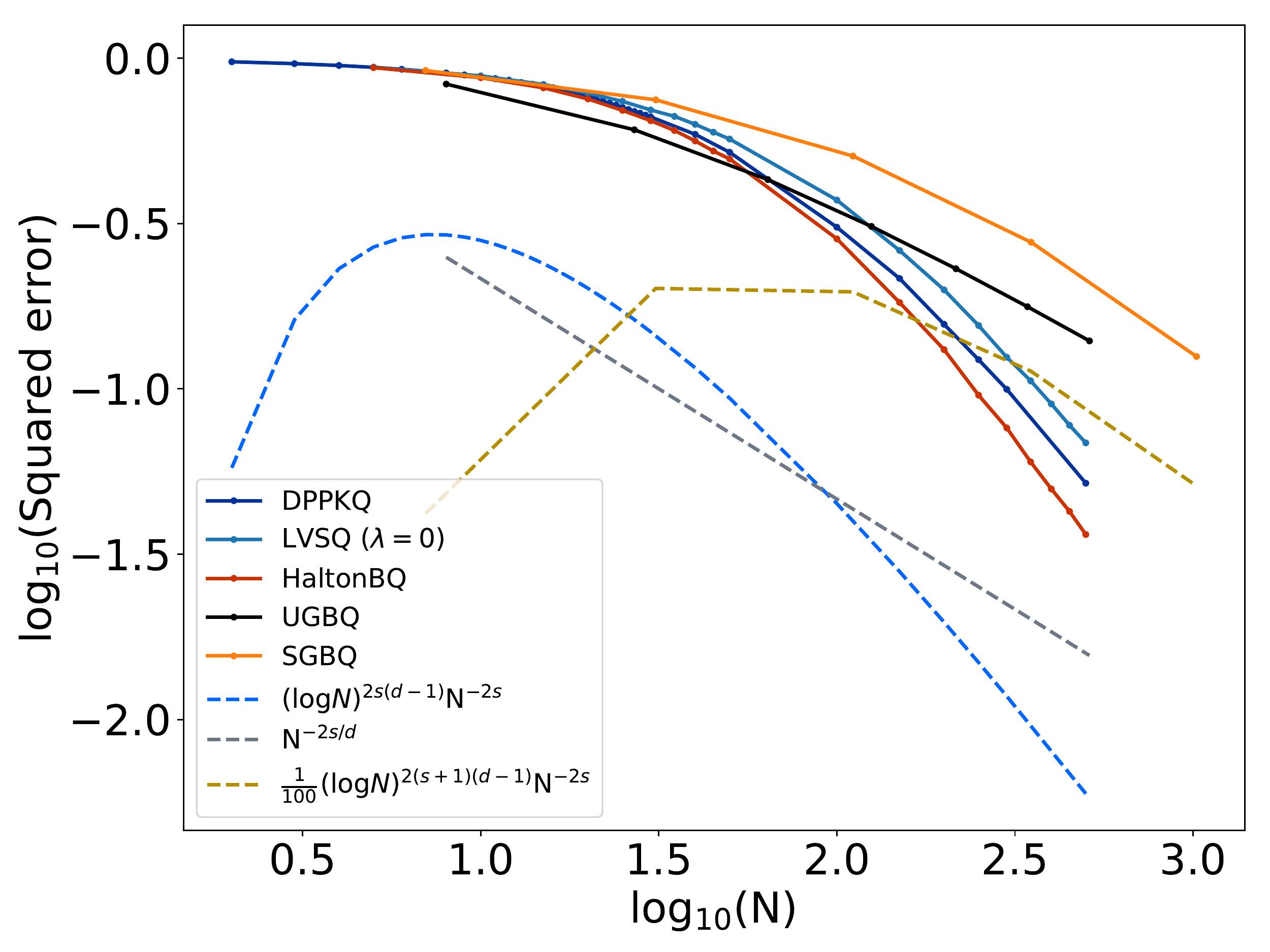}
\includegraphics[width= 0.45\textwidth]{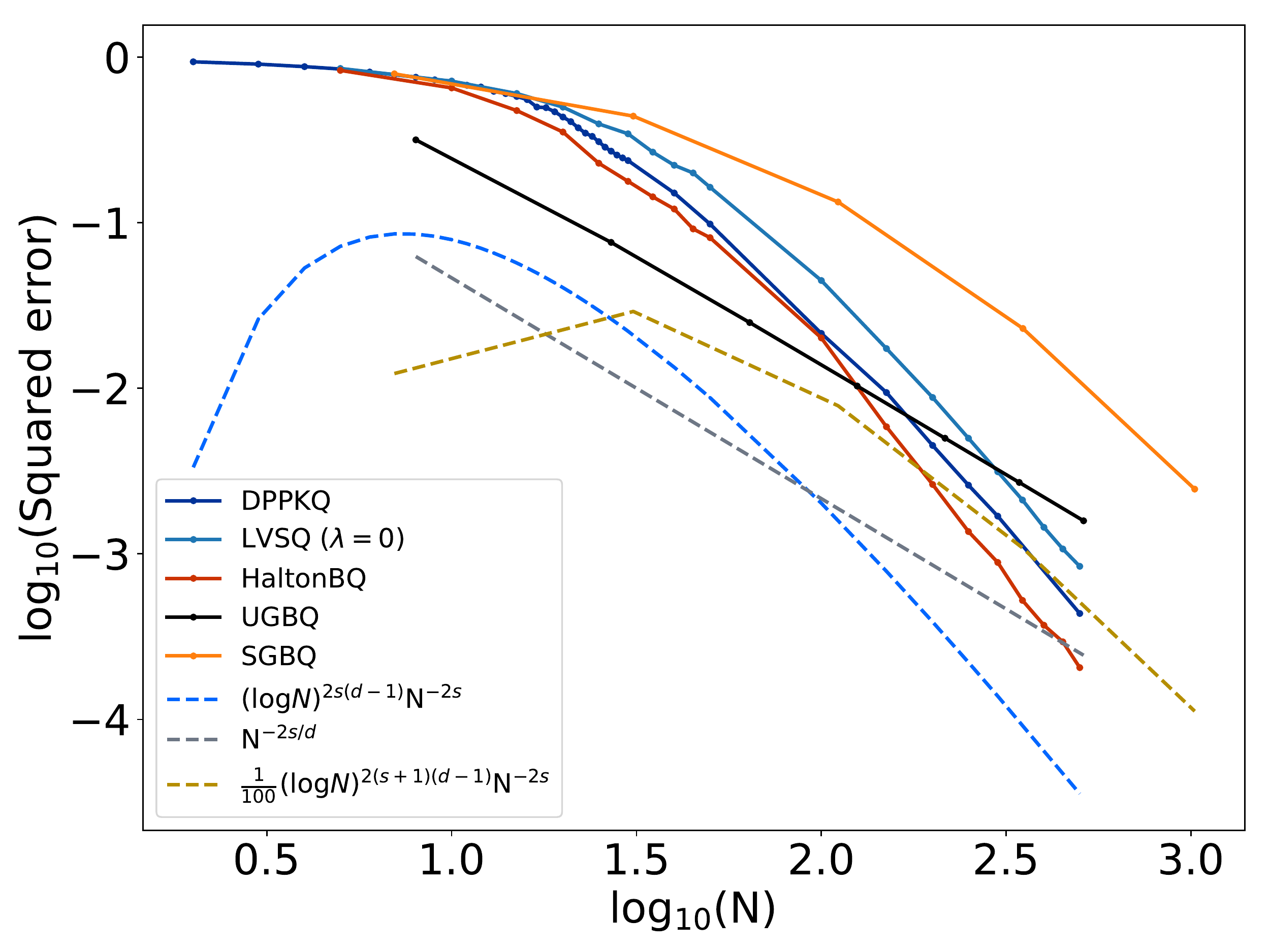} \\
\caption{The squared error for the Korobov space ($d \in \{2,3\}$, $s\in \{1,2\}$): (Top) the results for the dimension $d = 2$, (Left) the results for the regularity $s = 1$. \label{fig:KorobovResult}}
\end{figure}

\subsection{The multi Gaussian ensemble}
We consider the case of Gaussian spaces with $d \in \{2,3\}$. The kernel $k_{\gamma,d}$ and the reference measure are the tensor product of respectively the same kernel and the same measure used in Section~\ref{s:gaussian_numsim}. We compare DPPKQ and Bayesian quadrature based on the tensor product of Gauss-Hermite nodes noted GHBQ. Note that a variant of this algorithm was proposed in \citep{KaSa19}: the quadrature nodes are the tensor product of the Gauss-Hermite nodes however the weights were calculated differently. The authors proved under an assumption on the stability of the weights (that was verified empirically) that the rate of convergence is $\mathcal{O}(dr^{d}\beta'^{d}e^{-\delta' d N^{1/d}})$, where $r$ is a constant that quantify the stability of the weights, and $\beta',\delta'$ are constants that depend simultaneously on the the stability of the weights and length scales of the kernel and the measure. The results are compiled in Figure~\ref{fig:GaussianResult}.

The numerical simulations shows that the empirical rate of DPPKQ is $\mathcal{O}(e^{-\delta d N^{1/d}})$ that is slightly better than its theoretical rate $\mathcal{O}(e^{-\delta d!^{1/d} N^{1/d}})$. Moreover, we observe that the empirical rate of DPPKQ is better than the empirical rate of HGBQ.

\begin{figure}[]
    \centering
\includegraphics[width= 0.45\textwidth]{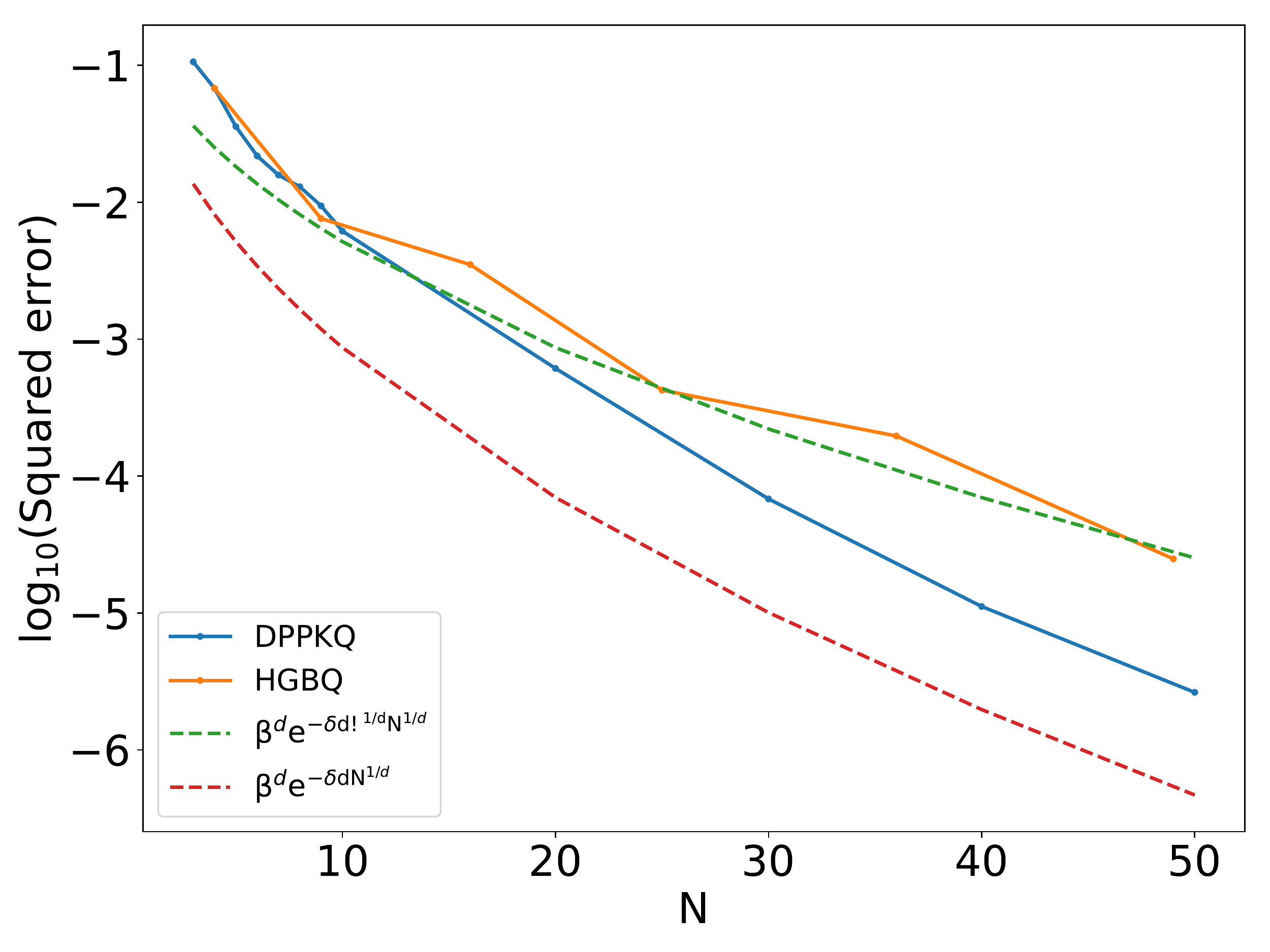}
\includegraphics[width= 0.45\textwidth]{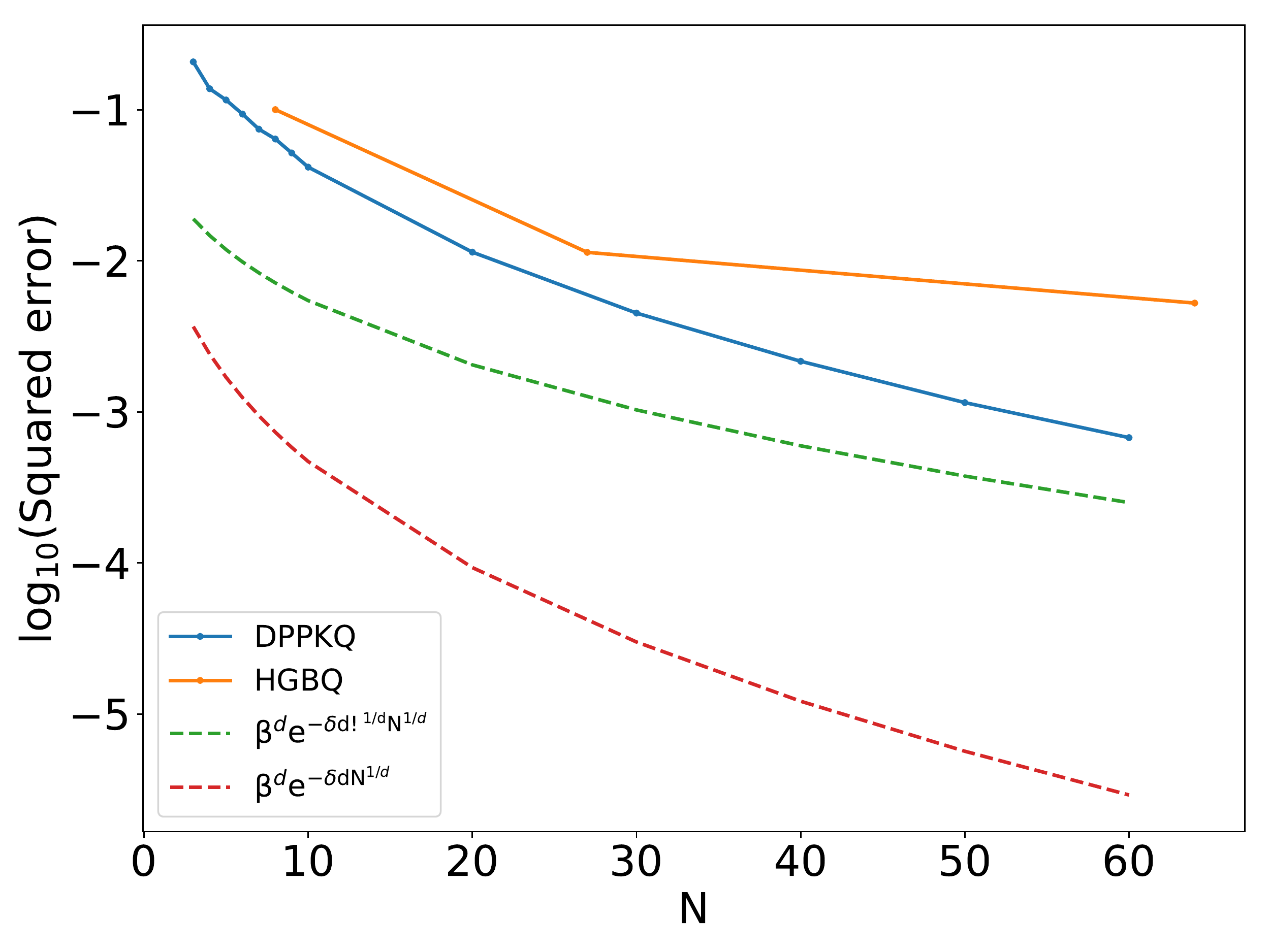} \\
\caption{The squared error for the Gaussian space ($d \in \{2,3\}$, $\gamma=1$). \label{fig:GaussianResult}}
\end{figure}

\section{Mercer's theorem, leverage scores, and principal angles}

For the sake of completeness, this section gathers some known results, which will be used to prove our own. We will need a general version of Mercer's theorem, as usual for kernel methods, see Section~\ref{subsec:mercer}. On a more technical ground, we will also need formulas for leverage score changes under rank 1 updates, see Section~\ref{subsec:lv_score_updates}. Finally, Section~\ref{subsec:principalangles} covers principal angles between subspaces of a Hilbert space, which bridge the gap between pairs of Hilbert subspaces and determinants, and facilitate taking expectations in Theorem~\ref{thm:main_theorem}.

\subsection{Mercer decomposition in non-compact subspaces}
\label{subsec:mercer}
In this section we recall Mercer's theorem and its extensions to non-compact spaces. Let $\mathcal{X}$ be a measurable space and $\mathrm{d}\omega$ a measure on $\mathcal{X}$. Assume $k$ is a positive definite kernel on $\mathcal{X}$. Whenever it is well-defined, we consider the operator
\begin{equation}
\bm{\Sigma} f(x) = \int\limits_{\mathcal{X}} k(x,y)f(y) \mathrm{d}\omega(y).
\end{equation}
\begin{theorem}\label{thm:Mercer_for_compact}
Assume that $\mathcal{X}$ is a compact space and $\mathrm{d}\omega$ is a finite Borel measure on $\mathcal{X}$.
Then, there exists an orthonormal basis $(e_{n})_{n \in \mathbb{N}^{*}}$ of $\mathbb{L}_{2}(\mathrm{d}\omega)$ consisting of eigenfunctions of $\bm{\Sigma}$, and the corresponding eigenvalues are non-negative.
The eigenfunctions corresponding to non-vanishing eigenvalues can be taken to be continuous, and the kernel $k$ writes
\begin{equation}
k(x,y) = \sum\limits_{n \in \mathbb{N}^{*}} \sigma_{n} e_{n}(x)e_{n}(y),
\end{equation}
where the convergence is absolute and uniform.
\end{theorem}
Theorem~\ref{thm:Mercer_for_compact} was first proven when $\mathcal{X} = [0,1]$ and $\mathrm{d}\omega$ is the Lebesgue measure in \cite{Mer1909}. A modern proof can be found in \cite{Lax02}, while the proof in the general case can be found in \cite{CuZh07}.  Note, however, that the compactness assumption in Theorem~\ref{thm:Mercer_for_compact} excludes kernels such as the Gaussian or the Laplace kernels.
Hence, extensions to non-compact spaces are usually required in ML.
In \cite{Sun05}, the author extended Theorem~\ref{thm:Mercer_for_compact} to $X = \cup_{i \in \mathbb{N}} X_{i}$, with the $X_{i}$s compact and $\mathrm{d}\omega(X_{i})<\infty$. One can also extend Mercer's theorem under a \textit{compact embedding} assumption \citep{StSc12}: the RKHS $\mathcal{F}$ associated to $k$ is said to be compactly embedded in $\mathbb{L}_{2}(\mathrm{d}\omega)$ if the application
\begin{align*}
  I_{\mathcal{F}}: \mathcal{F}&\longrightarrow \mathbb{L}_{2}(\mathrm{d}\omega) \\
  f &\longmapsto f
\end{align*}
is compact.
A sufficient condition for this assumption is the integrability of the diagonal (Lemma 2.3, \citep{StSc12}):
\begin{equation}
\int_{\mathcal{X}} k(x,x) \mathrm{d}\omega(x) < \infty.
\end{equation}
Note that this condition is not necessary (Example 2.9, \citep{StSc12}). Now, under the compact embedding assumption, the pointwise convergence of the Mercer decomposition to the kernel $k$ is equivalent to the injectivity of the embedding $I_{\mathcal{F}}$ (Theorem 3.1, \citep{StSc12}).

\subsection{Leverage score changes under rank 1 updates}\label{subsec:lv_score_updates}
In this section we prove a lemma inspired from Lemma 5 in \cite{Coh15}. This lemma concerns the changes of leverage scores under rank 1 updates.

We start by recalling the definition of leverage scores, which play an important role in randomized linear algebra \citep{DrMaMu06}. Let $N,M \in \mathbb{N}^{*}$, $M \geq N$. Let $\bm{A} \in \mathbb{R}^{N \times M}$ be a matrix of full rank.
For $i \in [M]$, denote $\bm{a}_{i}$ the $i$-th column of the matrix $\bm{A}$. Now, the $i$-th leverage score of the matrix $\bm{A}$ is defined by
\begin{equation}\label{eq:def_matrix_leverage_score}
\tau_{i}(\bm{A}) = \bm{a}_{i}^{\Tran} (\bm{A}\bm{A}^{\Tran})^{-1}\bm{a}_{i},
\end{equation}
while the cross-leverage score between the $i$-th column and the $j$-th column is defined by
\begin{equation}
\tau_{i,j}(\bm{A}) = \bm{a}_{i}^{\Tran} (\bm{A}\bm{A}^{\Tran})^{-1}\bm{a}_{j}.
\end{equation}
It holds \citep{DrMaMu06}
\begin{equation}\label{eq:bounds_on_lv_scores}
\forall i \in [M], \: \tau_{i}(\bm{A}) \in [0,1],
\end{equation}
and we have the following result.
\begin{lemma}\label{lemma:rank1_leverage_score}
Let $N,M \in \mathbb{N}^{*}$, $M \geq N$. Let $\bm{A} \in \mathbb{R}^{N \times M}$ of full rank and $\rho \in \mathbb{R}_{+}^{*}$ and $i \in [M]$. Let $\bm{W} \in \mathbb{R}^{M \times M}$ a diagonal matrix such that $\bm{W}_{i,i} = \sqrt{1+\rho}$ and $\bm{W}_{j,j} = 1$ for $j \neq i$. Then
\begin{equation}\label{eq:lv_score_update}
\tau_{i}(\bm{A}\bm{W}) = \frac{(1+\rho)\tau_{i}(\bm{A})}{1+\rho \tau_{i}(\bm{A})} \geq \tau_{i}(\bm{A}),
\end{equation}
and
\begin{equation}\label{eq:cross_lv_score_update}
\forall j \in [M]-\{i\}, \: \tau_{j}(\bm{A}\bm{W}) = \tau_{j}(\bm{A}) -\frac{\rho \tau_{i,j}(\bm{A})^{2}}{1+\rho \tau_{i}(\bm{A})} \leq \tau_{j}(\bm{A}).
\end{equation}
\end{lemma}
The proof of this lemma is similar to Lemma 5 in \cite{Coh15}. We recall the proof for completeness.
\begin{proof}(Adapted from \cite{Coh15})
The Sherman-Morrison formula applied to $\bm{A}\bm{W}\bm{W}^{\Tran}\bm{A}^{\Tran}$ and the vector $\sqrt{\rho} \bm{a}_{i}$ yields
\begin{align}
(\bm{A}\bm{W}\bm{W}^{\Tran}\bm{A}^{\Tran})^{-1} &  = (\bm{A}\bm{A}^{\Tran} + \rho \bm{a}_{i}\bm{a}_{i}^{\Tran})^{-1} \\
& = (\bm{A}\bm{A}^{\Tran})^{-1} - \frac{(\bm{A}\bm{A}^{\Tran})^{-1}\rho \bm{a}_{i}\bm{a}_{i}^{\Tran} (\bm{A}\bm{A}^{\Tran})^{-1}}{1+ \rho \bm{a}_{i}^{\Tran}(\bm{A}\bm{A}^{\Tran})^{-1}\bm{a}_{i}}.
\end{align}
By definition of $\tau_{i}(\bm{A}\bm{W})$
\begin{align}
\tau_{i}(\bm{A}\bm{W}) & = \sqrt{1 + \rho} \bm{a}_{i}^{\Tran}(\bm{A}\bm{W}\bm{W}^{\Tran}\bm{A}^{\Tran})^{-1}\bm{a}_{i}\sqrt{1 + \rho}\\
& = (1+\rho)\bm{a}_{i}^{\Tran} \left( (\bm{A}\bm{A}^{\Tran})^{-1} - \frac{(\bm{A}\bm{A}^{\Tran})^{-1}\rho \bm{a}_{i}\bm{a}_{i}^{\Tran} (\bm{A}\bm{A}^{\Tran})^{-1}}{1+ \rho \bm{a}_{i}^{\Tran}(\bm{A}\bm{A}^{\Tran})^{-1}\bm{a}_{i}} \right) \bm{a}_{i} \nonumber\\
& = (1+\rho) \left(\tau_{i}(\bm{A}) - \frac{\rho \tau_{i}(\bm{A})^{2}}{1+\rho \tau_{i}(\bm{A})} \right) \nonumber\\
& = (1+\rho)\frac{\tau_{i}(\bm{A})}{1+\rho \tau_{i}(\bm{A})} \nonumber.
\end{align}
Now let $j \in [M]-\{i\}$. By definition of $\tau_{j}(\bm{A}\bm{W})$
\begin{align}
\tau_{j}(\bm{A}\bm{W}) & =  \bm{a}_{j}^{\Tran}(\bm{A}\bm{W}\bm{W}^{\Tran}\bm{A}^{\Tran})^{-1}\bm{a}_{j}\\
& = \bm{a}_{j}^{\Tran} \left( (\bm{A}\bm{A}^{\Tran})^{-1} - \frac{(\bm{A}\bm{A}^{\Tran})^{-1}\rho \bm{a}_{i}\bm{a}_{i}^{\Tran} (\bm{A}\bm{A}^{\Tran})^{-1}}{1+ \rho \bm{a}_{i}^{\Tran}(\bm{A}\bm{A}^{\Tran})^{-1}\bm{a}_{i}} \right) \bm{a}_{j} \nonumber\\
& =  \tau_{j}(\bm{A}) - \frac{\rho \tau_{i,j}(\bm{A})^{2}}{1+\rho \tau_{i}(\bm{A})} \nonumber\\
& \leq \tau_{j}(\bm{A}) \nonumber.
\end{align}
\end{proof}

\subsection{Principal angles between subspaces in Hilbert spaces}
\label{subsec:principalangles}
We recall in this section the definition of principal angles between subspaces in Hilbert spaces and connect them to the determinant of the Gramian matrix of their orthonormal bases.


\begin{proposition}\label{prop:cos_between_subspaces}
Let $\mathcal{H}$ be a Hilbert space. Let $\mathcal{P}_{1}$ and $\mathcal{P}_{2}$ be two finite-dimensional subspaces of $\mathcal{H}$ with $N = \dim \mathcal{P}_{1} = \dim \mathcal{P}_{2}$. Denote $\bm{\Pi}_{\mathcal{P}_{1}}$ and $\bm{\Pi}_{\mathcal{P}_{2}}$ the orthogonal projections of $\mathcal{H}$ onto these two subspaces. There exist two orthonormal bases for $\mathcal{P}_{1}$ and $\mathcal{P}_{2}$ denoted $(\bm{v}_{i}^{1})_{i \in [N]}$ and $(\bm{v}_{i}^{2})_{i \in [N]}$, and a set of angles $\theta_{i}(\mathcal{P}_{1},\mathcal{P}_{2}) \in [0,\frac{\pi}{2}]$ such that
\begin{equation}
\cos \theta_{N}(\mathcal{P}_{1},\mathcal{P}_{2}) \leq \dots \leq \cos \theta_{1}(\mathcal{P}_{1},\mathcal{P}_{2}) ,
\end{equation}

and for $i \in [1,...,N]$
\begin{equation}
\langle \bm{v}_{i}^{1}, \bm{v}_{i}^{2} \rangle_{\mathcal{H}} = \cos \theta_{i}(\mathcal{P}_{1},\mathcal{P}_{2}),
\end{equation}
and
\begin{equation}\label{eq:principal_vectors_projection_relationship}
 \bm{\Pi}_{\mathcal{P}_{1}}\bm{v}_{i}^{2} = \cos \theta_{i}(\mathcal{P}_{1},\mathcal{P}_{2}) \bm{v}_{i}^{1},
\end{equation}
and
\begin{equation}
 \bm{\Pi}_{\mathcal{P}_{2}}\bm{v}_{i}^{1} = \cos \theta_{i}(\mathcal{P}_{1},\mathcal{P}_{2}) \bm{v}_{i}^{2}.
\end{equation}
In particular
\begin{equation}\label{eq:costhetaN}
\cos \theta_{N}(\mathcal{P}_{1},\mathcal{P}_{2}) = \inf\limits_{\bm{v} \in \mathcal{P}_{1},\|\bm{v}\|_{\mathcal{H}} = 1} \|\bm{\Pi}_{\mathcal{P}_{2}}\bm{v}\|_{\mathcal{H}}.
\end{equation}
\end{proposition}
We refer to \cite{GoVa12} for the proof in the finite-dimensional case and \cite{DaKa70} for the general case.
The following result shows that the principal angles are somewhat independent of the choice of orthonormal bases. It can be found in \cite{BjGo73,MiBe92} for the finite dimensional case. We give here the proof for the general case, for the sake of completeness.
\begin{corollary}\label{cor:cos_det_relationship}
Let $(\bm{w}^{1}_{i})_{i \in [N]}$ be any orthonormal basis of $\mathcal{P}_{1}$ and $(\bm{w}^{2}_{i})_{i \in [N]}$ be any orthonormal basis of $\mathcal{P}_{2}$, and let $\bm{W} = (\langle \bm{w}^{1}_{i},\bm{w}^{2}_{j} \rangle_{\mathcal{H}})_{1\leq i,j \leq N}$ and $\bm{G} = \bm{W}\bm{W}^{\Tran}$. Then
the eigenvalues of $\bm{G}$ are the $\cos^{2} \theta_{i}(\mathcal{P}_{1},\mathcal{P}_{2})$. In particular, $\Det^{2} \bm{W} = \Det \bm{G} =  \prod\limits_{i \in [N]} \cos^{2} \theta_{i}(\mathcal{P}_{1},\mathcal{P}_{2})$.
\end{corollary}

\begin{proof}
Let $(\bm{v}^{i}_{i})_{i \in [N]}$, $i\in\{1,2\}$, be the bases of Proposition~\ref{prop:cos_between_subspaces}. Let $\bm{U}^{1} \in \mathbb{O}_{N}(\mathbb{R})$ be such that
\begin{equation}
\forall i \in [N], \: \bm{w}^{1}_{i} = \sum\limits_{j \in [N]} u^{1}_{i,j} \bm{v}^{1}_{j}.
\end{equation}
Similarly, there exists a matrix $\bm{U}^{2} \in \mathbb{O}_{N}(\mathbb{R})$ such that
\begin{equation}
\forall i \in [N], \: \bm{w}^{2}_{i} = \sum\limits_{j \in [N]} u^{2}_{i,j} \bm{v}^{2}_{j}.
\end{equation}
This implies that
\begin{equation}
\bm{W} = \bm{U}^{1}\bm{V}\bm{U}^{2 \Tran},
\end{equation}
where $\displaystyle \bm{V} = (\langle \bm{v}^{1}_{i}, \bm{v}^{2}_{j} \rangle_{\mathcal{H}})_{1 \leq i,j \leq N}$. Then
\begin{equation}
\bm{G} = \bm{W}\bm{W}^{\Tran} = \bm{U}^{1}\bm{V}\bm{V}^{\Tran}\bm{U}^{1 \Tran}.
\end{equation}
Thus the eigenvalues of $\bm{G}$ are the eigenvalues of $\bm{V}\bm{V}^{\Tran}$. By Proposition~\ref{prop:cos_between_subspaces}, the diagonal elements of $\bm{V}$ are
\begin{equation}\label{eq:diagonal_of_V}
v_{i,i} = \langle \bm{v}^{1}_{i}, \bm{v}^{2}_{i} \rangle_{\mathcal{H}} = \cos \theta_{i}(\mathcal{P}_{1},\mathcal{P}_{2}).
\end{equation}
We finish the proof by showing that the anti-diagonal elements satisfy
\begin{equation}\label{eq:anti_diagonal_of_V}
v_{i,j} = \langle \bm{v}^{1}_{i}, \bm{v}^{2}_{j} \rangle_{\mathcal{H}} = 0.
\end{equation}
By \eqref{eq:principal_vectors_projection_relationship},
\begin{equation}
\forall i \in [N], \:\sum\limits_{j \in [N]} \langle \bm{v}_{i}^{2}, \bm{v}_{j}^{1}\rangle_{\mathcal{H}}^{2} = \|\bm{\Pi}_{\mathcal{P}_{1}}\bm{v}_{i}^{2}\|_{\mathcal{H}}^{2} = \cos^{2} \theta_{i}(\mathcal{P}_{1},\mathcal{P}_{2}).
\end{equation}
Then
\begin{equation}
 \sum\limits_{i \in [N]}\sum\limits_{j \in [N]} \langle \bm{v}_{i}^{2}, \bm{v}_{j}^{1}\rangle_{\mathcal{H}}^{2} = \sum\limits_{i \in [N]} \cos^{2} \theta_{i}(\mathcal{P}_{1},\mathcal{P}_{2}) = \sum\limits_{i \in [N]} \langle \bm{v}_{i}^{2}, \bm{v}_{i}^{1}\rangle_{\mathcal{H}}^{2} .
\end{equation}
Thus
\begin{equation}
\sum\limits_{\substack{i,j \in [N]\\ i \neq j}} \langle \bm{v}_{i}^{2}, \bm{v}_{i}^{1}\rangle_{\mathcal{H}}^{2} = 0 .
\end{equation}
Finally, $\bm{V}$ is a diagonal matrix and the eigenvalues of $\bm{G}$ are the $\cos^{2} \theta_{i}(\mathcal{P}_{1},\mathcal{P}_{2})$.
\end{proof}

\section{Proofs of our results}
\label{s:proofs}
Section~\ref{app:K_N_non_singular} contains the proof of Proposition~\ref{prop:K_N_non_singular}. In the main paper, we use it under the form of Corollary~\ref{c:regularization} to ensure that $\bm{K}(\bm{x})$ is almost surely invertible when $\bm{x} = \{x_{1}, \dots , x_{N}\}$ is a projection DPP with reference measure $\mathrm{d}\omega$ and kernel \eqref{e:kernel}. This allows computing the quadrature weights.

The rest of Section~\ref{s:proofs} deals with Theorem~\ref{thm:main_theorem}, our upper bound on the approximation error of DPP-based kernel quadrature. The proof is rather long, but can be decomposed in four steps, which we now introduce for ease of reading.

First, we prove Lemma~\ref{lemma:approximation_error_spectral_bound}, which separates the search for an upper bound into examining the contribution of the three terms in~\eqref{e:boundingOrthogonalProjection}; this is Section~\ref{s:proofOfSeparationLemma}. The first two terms of \eqref{e:boundingOrthogonalProjection} only depend on the function $g$ in~\eqref{e:quadrature}, and we leave them be. The third term is more geometric, and relates to the approximation error of the space spanned by $(e^{\cal F}_n)_{n\in[N]}$ by the (random) subspace ${\cal T}(\bm{x})$.

Second, in Section~\ref{s:proofOfPerturbationInequality}, we bound this geometric term for a fixed DPP realization $\bm{x}$. We pay attention to obtain a bound that will later yield a tractable expectation under that DPP. This is done in Proposition~\ref{prop:kernel_perturbation_inequality}, which in turn requires two intermediate results, Lemma~\ref{lemma:interpolation_error_matricial_form} and Proposition~\ref{prop:graded_RKHS}.

Third, we take the expectation of the bound in Proposition~\ref{prop:kernel_perturbation_inequality} under the proposed DPP. This is done in Proposition~\ref{prop:expected_value_of_product_of_cos}, which is proven thanks to Proposition~\ref{prop:K_N_non_singular}, Lemmas \ref{lemma:max_error_cos}, \ref{lemma:cos_ratio_det} \& \ref{lemma:truncated_Fredholm_formula}. This is Section~\ref{s:proofOfExpectedProduct}.

Fourth, Theorem~\ref{thm:main_theorem} is obtained in Section~\ref{s:finalBound}, using the results of the previous steps, and an argument to reduce the proof to RKHSs with flat initial spectrum.

\subsection{Proof of Proposition~\ref{prop:K_N_non_singular}}\label{app:K_N_non_singular}

\begin{proof}
Recall the Mercer decomposition of $k$:
\begin{equation}\label{eq:Mercer_decomposition_1}
k(x,y) = \sum\limits_{m \in \mathbb{N}^{*}} \sigma_{m} e_{m}(x)e_{m}(y),
\end{equation}
where the convergence is point-wise on $\mathcal{X}$.
Define for $M \in \mathbb{N}^{*}, \: M \geq N$ the $M$-th truncated kernel
\begin{equation}
k_{M}(x,y) = \sum\limits_{m \in [M]} \sigma_{m} e_{m}(x)e_{m}(y).
\end{equation}
By \eqref{eq:Mercer_decomposition_1}
\begin{equation}\label{eq:truncated_kernel_limit}
\forall x,y \in \mathcal{X}, \:\lim\limits_{M \rightarrow \infty} k_{M}(x,y) = k(x,y).
\end{equation}
Let $\bm{x} = (x_{1}, \dots, x_{N}) \in \mathcal{X}^{N}$ such that $\Det \bm{E}(\bm{x}) \neq 0$, and define
\begin{equation}
\bm{K}_{M}(\bm{x}) = (k_{M}(x_{i},x_{j}))_{i,j \in [N]} .
\end{equation}
By the continuity of the function $\bm{M} \in \mathbb{R}^{N \times N} \mapsto \Det \bm{M}$ and by \eqref{eq:truncated_kernel_limit}
\begin{equation}
  \lim\limits_{M \rightarrow \infty} \Det \bm{K}_{M}(\bm{x}) = \Det \bm{K}(\bm{x}).
\end{equation}
Thus to prove that $\Det \bm{K}(\bm{x}) > 0$, it is enough to prove that the $\Det \bm{K}_{M}(\bm{x})$ is larger than a positive real number for $M$ large enough. We write
\begin{equation}
  \bm{K}_{M}(\bm{x}) = \bm{F}_{M}(\bm{x})^{\Tran}\bm{\Sigma}_{M}\bm{F}_{M}(\bm{x}),
\end{equation}
with $\bm{F}_{M}(\bm{x}) = (e_{i}(x_{j}))_{(i,j) \in [M]\times[N]}$ and $\Sigma_{M}$ is a diagonal matrix containing the first $M$ eigenvalues $(\sigma_{m})$. The Cauchy-Binet identity yields
\begin{align}
  \Det \bm{K}_{M}(\bm{x}) & = \sum\limits_{T \subset [M], |T| = N} \prod\limits_{i \in T} \sigma_{i} \Det^{2} (e_{i}(x_{j}))_{(i,j)\in T \times [N]} \\
& \geq \prod\limits_{i \in [N]} \sigma_{i} \Det^{2}\bm{E}(\bm{x})>0.
\end{align}
Therefore,
\begin{equation}
\Det \bm{K}(\bm{x}) = \lim\limits_{M \rightarrow \infty} \Det \bm{K}_{M}(\bm{x}) \geq \prod\limits_{i \in [N]} \sigma_{i} \Det^{2} \bm{E}(\bm{x}) > 0.
\end{equation}
so that $\bm{K}(\bm{x})$ is a.s. invertible.
\end{proof}

\subsection{Proof of Lemma~\ref{lemma:approximation_error_spectral_bound}}
\label{s:proofOfSeparationLemma}
\begin{proof}
First, we prove that
\begin{equation}
  \|\bm{\Sigma}^{-1/2}\mu_{g}\|_{\mathcal{F}}^{2} = \|g\|_{\mathrm{d}\omega}^{2} \leq 1.
\end{equation} 
 Recall that
\begin{equation}
\mu_{g} = \int_{\mathcal{X}} g(y)k(.,y) \mathrm{d}\omega(y),
\end{equation}
and that we assumed in Section~\ref{s:intro} that $\mathcal{F}$ is dense in $\mathbb{L}_{2}(\mathrm{d}\omega)$, so that $(e_{m})_{m \in \mathbb{N}}$ is an orthonormal basis of $\mathbb{L}_{2}(\mathrm{d}\omega)$ and the eigenvalues $\sigma_{n}$ are strictly positive. Now let $\bm{\Sigma}^{-1/2} : \mathcal{F} \rightarrow \mathbb{L}_{2}(\mathrm{d}\omega)$ and $\bm{\Sigma}^{1/2} : \mathbb{L}_{2}(\mathrm{d}\omega) \rightarrow \mathcal{F}$ be defined by
\begin{equation}
\bm{\Sigma}^{-1/2} e_{m}^{\mathcal{F}} = e_{m}, \: \forall m\in \mathbb{N}^*,
\end{equation}
\begin{equation}
\bm{\Sigma}^{1/2} e_{m} =  e_{m}^{\mathcal{F}}, \: \forall m\in \mathbb{N}^*.
\end{equation}
Observe that $\bm{\Sigma}^{-1/2} \mu_{g} = \bm{\Sigma}^{-1/2} \bm{\Sigma} g = \bm{\Sigma}^{1/2} g \in \mathcal{F}$. Now, for $m \in \mathbb{N}^{*}$,
\begin{align}
\langle e_{m}^{\mathcal{F}}, \bm{\Sigma}^{-1/2}\mu_{g} \rangle_{\mathcal{F}} & =\langle e_{m}^{\mathcal{F}}, \bm{\Sigma}^{1/2}g \rangle_{\mathcal{F}} \\
& = \langle e_{m}^{\mathcal{F}}, \bm{\Sigma}^{1/2} \sum\limits_{n \in \mathbb{N}^{*}} \langle g, e_{n} \rangle_{\mathrm{d}\omega} e_{n} \rangle_{\mathcal{F}}\\
& = \sum\limits_{n \in \mathbb{N}^{*}} \langle g, e_{n} \rangle_{\mathrm{d}\omega} \langle e_{m}^{\mathcal{F}}, \bm{\Sigma}^{1/2} e_{n} \rangle_{\mathcal{F}} \\
& = \sum\limits_{n \in \mathbb{N}^{*}} \langle g, e_{n} \rangle_{\mathrm{d}\omega} \langle e_{m}^{\mathcal{F}},  e_{n}^{\mathcal{F}} \rangle_{\mathcal{F}} \\
& = \langle g,e_{m} \rangle_{\mathrm{d}\omega} \nonumber.
\end{align}
As a consequence,
\begin{equation}\label{eq:bound_on_norm_mean_element}
  \|\bm{\Sigma}^{-1/2}\mu_{g}\|_{\mathcal{F}}^{2} = \|g\|_{\mathrm{d}\omega}^{2} \leq 1.
\end{equation} 
Now we turn to proving \eqref{e:boundingOrthogonalProjection} from the main text. Define first the operators $\bm{\Sigma}_{N}, \bm{\Sigma}_{N}^{1/2}, \bm{\Sigma}_{N}^{\perp}, \bm{\Sigma}_{N}^{\perp 1/2}: \mathbb{L}_{2}(\mathrm{d}\omega) \rightarrow \mathcal{F}$, $\bm{\Sigma}_{N}^{1/2} : \mathbb{L}_{2}(\mathrm{d}\omega) \rightarrow \mathcal{F}$ and $\bm{\Sigma}_{N}^{\perp} : \mathbb{L}_{2}(\mathrm{d}\omega) \rightarrow \mathcal{F}$ by
\begin{equation}
 \bm{\Sigma}_{N} e_{m} = \left\{
    \begin{array}{ll}
        \sigma_{m}e_{m} & \mbox{if} \: m \in [N] \\
        0 & \mbox{else}
    \end{array}
\right. ,
\end{equation}
\begin{equation}
 \bm{\Sigma}_{N}^{1/2} e_{m} = \left\{
    \begin{array}{ll}
        \sqrt{\sigma_{m}}e_{m} & \mbox{if} \: m \in [N] \\
        0 & \mbox{else}
    \end{array}
\right. ,
\end{equation}
\begin{equation}
 \bm{\Sigma}_{N}^{\perp} e_{m} = \left\{
    \begin{array}{ll}
        0      & \mbox{if} \: m \in [N] \\
        \sigma_{m}e_{m} & \mbox{if} \: m \geq N+1 \\
    \end{array}
\right. ,
\end{equation}
\begin{equation}
 \bm{\Sigma}_{N}^{\perp 1/2} e_{m} = \left\{
    \begin{array}{ll}
        0      & \mbox{if} \: m \in [N] \\
        \sqrt{\sigma_{m}}e_{m} & \mbox{if} \: m \geq N+1 \\
    \end{array}
\right. ,
\end{equation}
Note that $\bm{\Sigma}^{1/2} = \bm{\Sigma}_{N}^{1/2} + \bm{\Sigma}_{N}^{\perp 1/2}$ and
\begin{equation}\label{eq:sup_sigma_N}
  \sup\limits_{\|\mu\|_{\mathcal{F}} \leq 1} \|\bm{\Sigma}_{N}^{\perp 1/2} \mu\|_{\mathcal{F}}^{2}  = \sigma_{N+1}.
\end{equation}
Using \eqref{eq:bound_on_norm_mean_element}, there exists $\tilde{\mu}_{g} \in \mathcal{F}$ such that $\|\tilde{\mu}_{g}\|_{\mathcal{F}} \leq 1$ and $\mu_{g} = \bm{\Sigma}^{1/2}\tilde{\mu}_{g}$. Now, the approximation error writes
\begin{align}
  \|\bm{\Pi}_{\mathcal{T}(\bm{x})^{\perp}}\mu_{g}\|^{2}_{\mathcal{F}}& = \| \bm{\Pi}_{\mathcal{T}(\bm{x})^{\perp}}\bm{\Sigma}^{1/2} \tilde{\mu}_{g}\| _{\mathcal{F}}^{2}\\
  & = \| \bm{\Pi}_{\mathcal{T}(\bm{x})^{\perp}}(\bm{\Sigma}^{1/2}_{N} + \bm{\Sigma}^{\perp 1/2}_{N}) \tilde{\mu}_{g}\| _{\mathcal{F}}^{2} \nonumber\\
  & = \| \bm{\Pi}_{\mathcal{T}(\bm{x})^{\perp}}\bm{\Sigma}^{1/2}_{N} \tilde{\mu}_{g}\|_{\mathcal{F}}^{2} + \| \bm{\Pi}_{\mathcal{T}(\bm{x})^{\perp}}\bm{\Sigma}^{\perp 1/2}_{N} \tilde{\mu}_{g}\| _{\mathcal{F}}^{2} \\
  & \qquad\qquad + 2 \langle \bm{\Pi}_{\mathcal{T}(\bm{x})^{\perp}}\bm{\Sigma}^{1/2}_{N} \tilde{\mu}_{g}, \bm{\Pi}_{\mathcal{T}(\bm{x})^{\perp}}\bm{\Sigma}^{\perp 1/2}_{N} \tilde{\mu}_{g}\rangle_{\mathcal{F}} \nonumber\\
  & \leq 2 \left(\| \bm{\Pi}_{\mathcal{T}(\bm{x})^{\perp}}\bm{\Sigma}^{1/2}_{N} \tilde{\mu}_{g}\| _{\mathcal{F}}^{2} + \| \bm{\Pi}_{\mathcal{T}(\bm{x})^{\perp}}\bm{\Sigma}^{\perp 1/2}_{N} \tilde{\mu}_{g}\| _{\mathcal{F}}^{2}\right) \nonumber.
\end{align}
The operator $\bm{\Pi}_{\mathcal{T}(\bm{x})^{\perp}}$ is an orthogonal projection and $\|\tilde{\mu}_{g}\|_{\mathcal{F}} \leq 1$ so that by \eqref{eq:sup_sigma_N}
\begin{equation}
\| \bm{\Pi}_{\mathcal{T}(\bm{x})^{\perp}}\bm{\Sigma}^{\perp 1/2}_{N} \tilde{\mu}_{g}\| _{\mathcal{F}}^{2} \leq \| \bm{\Sigma}^{\perp 1/2}_{N} \tilde{\mu}_{g}\| _{\mathcal{F}}^{2} \leq \sigma_{N+1}.
\end{equation}
Now, recall that the $(e_{n}^{\mathcal{F}})_{n \in [N]}$ is orthonormal. Moreover for $n \in [N]$, $e_{n}^{\mathcal{F}}$ is an eigenfunction of $\bm{\Sigma}_{N}^{1/2}$ and the corresponding eigenvalue is $\sqrt{\sigma}_{n}$. Thus
 \begin{equation}
\bm{\Pi}_{\mathcal{T}(\bm{x})^{\perp}}\bm{\Sigma}^{1/2}_{N} \tilde{\mu}_{g} = \bm{\Pi}_{\mathcal{T}(\bm{x})^{\perp}}\sum\limits_{n \in [N]} \sqrt{\sigma_{n}} \langle \tilde{\mu}_{g}, e_{n}^{\mathcal{F}} \rangle_{\mathcal{F}} e_{n}^{\mathcal{F}} = \sum\limits_{n \in [N]} \langle \tilde{\mu}_{g}, e_{n}^{\mathcal{F}} \rangle_{\mathcal{F}} \sqrt{\sigma_{n}} \bm{\Pi}_{\mathcal{T}(\bm{x})^{\perp}}e_{n}^{\mathcal{F}}.
\end{equation}
Then
\begin{align}
    \|\bm{\Pi}_{\mathcal{T}(\bm{x})^{\perp}}\bm{\Sigma}^{1/2}_{N} \tilde{\mu}_{g}\|_{\mathcal{F}}^{2} & = \| \sum\limits_{n \in [N]} \langle \tilde{\mu}_{g}, e_{n}^{\mathcal{F}} \rangle_{\mathcal{F}} \sqrt{\sigma_{n}} \bm{\Pi}_{\mathcal{T}(\bm{x})^{\perp}}e_{n}^{\mathcal{F}} \|_{\mathcal{F}}^{2}\\
    & = \sum\limits_{n \in [N]} \sum\limits_{m \in [N]}\langle \tilde{\mu}_{g}, e_{n}^{\mathcal{F}} \rangle_{\mathcal{F}} \langle \tilde{\mu}_{g}, e_{m}^{\mathcal{F}} \rangle_{\mathcal{F}} \sqrt{\sigma_{n}}\sqrt{\sigma_{m}}  \langle  \bm{\Pi}_{\mathcal{T}(\bm{x})^{\perp}}e_{n}^{\mathcal{F}}, \bm{\Pi}_{\mathcal{T}(\bm{x})^{\perp}}e_{m}^{\mathcal{F}}\rangle_{\mathcal{F}} \nonumber\\
    & \leq \sum\limits_{n \in [N]} \sum\limits_{m \in [N]}\langle \tilde{\mu}_{g}, e_{n}^{\mathcal{F}} \rangle_{\mathcal{F}} \langle \tilde{\mu}_{g}, e_{m}^{\mathcal{F}} \rangle_{\mathcal{F}} \sqrt{\sigma_{n}}\sqrt{\sigma_{m}}   \|\bm{\Pi}_{\mathcal{T}(\bm{x})^{\perp}}e_{n}^{\mathcal{F}}\|_{\mathcal{F}}\|\bm{\Pi}_{\mathcal{T}(\bm{x})^{\perp}}e_{m}^{\mathcal{F}}\|_{\mathcal{F}} \nonumber\\
    & \leq \left( \sum\limits_{n \in [N]} \sum\limits_{m \in [N]} |\langle \tilde{\mu}_{g}, e_{n}^{\mathcal{F}} \rangle_{\mathcal{F}}|\cdot | \langle \tilde{\mu}_{g}, e_{m}^{\mathcal{F}} \rangle_{\mathcal{F}}| \right) \max\limits_{n \in [N]}\sigma_{n}  \|\bm{\Pi}_{\mathcal{T}(\bm{x})^{\perp}}e_{n}^{\mathcal{F}}\|_{\mathcal{F}}^{2} \nonumber\\
    & \leq \left( \sum\limits_{n \in [N]} |\langle \tilde{\mu}_{g}, e_{n}^{\mathcal{F}} \rangle_{\mathcal{F}}| \right)^{2} \max\limits_{n \in [N]}\sigma_{n}  \|\bm{\Pi}_{\mathcal{T}(\bm{x})^{\perp}}e_{n}^{\mathcal{F}}\|_{\mathcal{F}}^{2} .
\end{align}
Remarking that
$ \|g \|_{d\omega,1} = \sum\limits_{n \in [N]} |\langle \tilde{\mu}_{g}, e_{n}^{\mathcal{F}} \rangle_{\mathcal{F}}|
$
concludes the proof of \eqref{e:boundingOrthogonalProjection} and therefore Lemma~\ref{lemma:approximation_error_spectral_bound}.
\end{proof}

\subsection{Proof of Proposition~\ref{prop:kernel_perturbation_inequality}}
\label{s:proofOfPerturbationInequality}
Proposition~\ref{prop:kernel_perturbation_inequality} gives an upper bound to the term $\max\limits_{n \in [N]}\sigma_{n}  \|\bm{\Pi}_{\mathcal{T}(\bm{x})^{\perp}}e_{n}^{\mathcal{F}}\|_{\mathcal{F}}^{2}$ that appears in Lemma~\ref{lemma:approximation_error_spectral_bound}.
We first prove a technical result, Lemma~\ref{lemma:interpolation_error_matricial_form}, and then combine it with  Proposition~\ref{prop:graded_RKHS} to finish the proof. We conclude with the proof of Proposition~\ref{prop:graded_RKHS}.

%

\subsubsection{A preliminary lemma}

Let $\bm{x} = (x_{1}, \dots, x_{N}) \in \mathcal{X}^{N}$. Recall that $\bm{K}(\bm{x}) = (k(x_{i},x_{j}))_{1 \leq i,j \leq N}$ and denote $\tilde{\bm{K}}(\bm{x}) = (\tilde{k}(x_{i},x_{j}))_{1 \leq i,j \leq N}$, see section \ref{subsec:expectation_under_the_DPP}.
%
In the following, we define  
\begin{eqnarray}
  \Delta_n^{\cal F}(\bm{x}) & = & e_{n}^{\mathcal{F}}(\bm{x})^{\Tran}\bm{K}(\bm{x})^{-1}e_{n}^{\mathcal{F}}(\bm{x})\\
  \Delta_n^{\tilde{\cal F}}(\bm{x}) & = & e_{n}^{\tilde{\mathcal{F}}}(\bm{x})^{\Tran}\tilde{\bm{K}}(\bm{x})^{-1}e_{n}^{\tilde{\mathcal{F}}}(\bm{x})
\end{eqnarray}

Lemma~\ref{lemma:interpolation_error_matricial_form} below shows that each term of the form $\Delta_n^{\cal F}(\bm{x})$ measures the squared norm of the projection of $e_{n}^{\mathcal{F}}$ on ${\cal T}(\bm{x})$. The same holds for $\Delta_n^{\tilde{\cal F}}(\bm{x})$ and the projection of $e_{n}^{\tilde{\mathcal{F}}}$ onto $\tilde{{\cal T}}(\bm{x})$.

Indeed, $\displaystyle \|\bm{\Pi}_{\mathcal{T}(\bm{x})^{\perp}}e_{n}^{\mathcal{F}}\|^{2}_{\mathcal{F}} = 1- \|\bm{\Pi}_{\mathcal{T}(\bm{x})}e_{n}^{\mathcal{F}}\|^{2}_{\mathcal{F}}$ since $\|e_{n}^{\mathcal{F}}\|^{2}_{\mathcal{F}} = 1$.
Thus it is sufficient to prove that $\|\bm{\Pi}_{\mathcal{T}(\bm{x})}e_{n}^{\mathcal{F}}\|^{2}_{\mathcal{F}} = \Delta_n^{\cal F}(\bm{x})$. This boils down to showing that $\bm{K}(\bm{x})^{-1}$ is the matrix of the inner product $\langle \cdot,\cdot \rangle_{\cal F}$ restricted to ${\mathcal{T}(\bm{x})}$.
\begin{lemma}\label{lemma:interpolation_error_matricial_form}
For $n \in \mathbb{N}^{*}$, let $e_{n}^{\mathcal{F}}(\bm{x}), e_{n}^{\tilde{\mathcal{F}}}(\bm{x})\in \mathbb{R}^{N}$ the vectors of the evaluations of $e_{n}^{\mathcal{F}}$ and $e_{n}^{\tilde{\mathcal{F}}}$ on the elements of $\bm{x}$ respectively. Then
\begin{eqnarray}
  \|\bm{\Pi}_{\mathcal{T}(\bm{x})^{\perp}}e_{n}^{\mathcal{F}}\|^{2}_{\mathcal{F}} & = & 1- \Delta_n^{\cal F}(\bm{x}), \label{eq:basis_element_interpolation_error_1}\\
  \|\bm{\Pi}_{\tilde{\mathcal{T}}_{N}(\bm{x})^{\perp}}e_{n}^{\tilde{\mathcal{F}}}\|^{2}_{\tilde{\mathcal{F}}} & = & 1- \Delta_n^{\tilde{\cal F}} (\bm{x}). \label{eq:basis_element_interpolation_error_2}
\end{eqnarray}
\end{lemma}\label{lemma:interpolation_error_kernel_formula}
We give the proof of \eqref{eq:basis_element_interpolation_error_1}; the proof of \eqref{eq:basis_element_interpolation_error_2} follows the same lines.
\begin{proof}
Let us write
\begin{equation}
\bm{\Pi}_{\mathcal{T}(\bm{x})}e_{n}^{\mathcal{F}} = \sum\limits_{i \in [N]} c_{i}k(x_{i},.),
\end{equation}
where the $c_{i}$ are the elements of the vector $\bm{c} = \bm{K}(\bm{x})^{-1}e_{n}^{\mathcal{F}}(\bm{x})$. Then
\begin{align}
\bm{\Pi}_{\mathcal{T}(\bm{x})}e_{n}^{\mathcal{F}} & =  \sum\limits_{i \in [N]}c_{i}\sum\limits_{m \in \mathbb{N}^{*}} \sigma_{m}e_{m}(x_{i})e_{m}(.)\\
& = \sum\limits_{m \in \mathbb{N}^{*}} \sqrt{\sigma_{m}}\left(\sum\limits_{i \in [N]}c_{i} e_{m}(x_{i}) \right) e_{m}^{\mathcal{F}}(.) \nonumber.
\end{align}
Since $(e_m^{\cal F})_{m\in \mathbb{N}^{*}}$ is orthonormal,
\begin{align}\label{eq:interpolation_error_and_vector_c}
\|\bm{\Pi}_{\mathcal{T}(\bm{x})}e_{n}^{\mathcal{F}}\|^{2}_{\mathcal{F}} & = \sum\limits_{m \in \mathbb{N}^{*}} \sigma_{m} \left(\sum\limits_{i \in [N]}c_{i} e_{m}(x_{i}) \right)^{2}\\
& = \sum\limits_{m \in \mathbb{N}^{*}} \sigma_{m} \sum\limits_{i \in [N]} \sum\limits_{j \in [N]}c_{i}c_{j} e_{m}(x_{i})e_{m}(x_{j}) \nonumber\\
& = \sum\limits_{m \in \mathbb{N}^{*}} \bm{c}^{\Tran} e_{m}^{\mathcal{F}}(\bm{x})e_{m}^{\mathcal{F}}(\bm{x})^{\Tran} \bm{c} \nonumber\\
& = \bm{c}^{\Tran} \sum\limits_{m \in \mathbb{N}^{*}}  e_{m}^{\mathcal{F}}(\bm{x})e_{m}^{\mathcal{F}}(\bm{x})^{\Tran} \bm{c} \nonumber.
\end{align}
Using Mercer's theorem, see \eqref{eq:truncated_kernel_limit},
\begin{equation}\label{eq:matricial_mercer_decomposition}
  \bm{K}(\bm{x}) = \sum\limits_{m \in \mathbb{N}^{*}}  e_{m}^{\mathcal{F}}(\bm{x})e_{m}^{\mathcal{F}}(\bm{x})^{\Tran}.
\end{equation}
Combining \eqref{eq:interpolation_error_and_vector_c} and \eqref{eq:matricial_mercer_decomposition} along with the definition of the vector $\bm{c} = \bm{K}(\bm{x})^{-1}e_{n}^{\mathcal{F}}(\bm{x})$ yields
\begin{align}
\|\bm{\Pi}_{\mathcal{T}(\bm{x})}e_{n}^{\mathcal{F}}\|^{2}_{\mathcal{F}} & = \bm{c}^{\Tran} \bm{K}(\bm{x}) \bm{c}\\
& = e_{n}^{\mathcal{F}}(\bm{x})^{\Tran} \bm{K}(\bm{x})^{-1} \bm{K}(\bm{x}) \bm{K}(\bm{x})^{-1}e_{n}^{\mathcal{F}}(\bm{x}) \nonumber\\
& = e_{n}^{\mathcal{F}}(\bm{x})^{\Tran} \bm{K}(\bm{x})^{-1}e_{n}^{\mathcal{F}}(\bm{x}) \nonumber\\
& = \Delta_n^{\cal F}(\bm{x}) \nonumber.
\end{align}
\end{proof}

%

%
\subsubsection{End of the proof of Proposition~~\ref{prop:kernel_perturbation_inequality}}
\begin{proof}
By Lemma~\ref{lemma:interpolation_error_matricial_form}, the inequality~\eqref{eq:kernel_perturbation_inequality} in Proposition~\ref{prop:kernel_perturbation_inequality} is equivalent to
\begin{equation}\label{eq:kernel_perturbation_inequality_reformulated}
\forall n \in [N], \: \sigma_{n} \left(1- \Delta_n^{\cal F}(\bm{x}) \right) \leq \sigma_{1} \left(1- \Delta_n^{\tilde{\cal F}}(\bm{x})\right) .
\end{equation}
As an intermediate remark, note that in the special case $n =1$, by construction
\begin{equation}
  \bm{K}(\bm{x}) \prec \tilde{\bm{K}}(\bm{x}),
\end{equation}
where $\prec$ is the Loewner order, the partial order defined by the convex cone of positive semi-definite matrices. Thus
\begin{equation}\label{eq:inequality_between_K_K_tilde}
  \tilde{\bm{K}}(\bm{x})^{-1} \prec \bm{K}(\bm{x})^{-1}.
\end{equation}
Noting that $\tilde{\sigma}_{1} = \sigma_{1}$ and that
\begin{equation}\label{eq:trivial_equality_between_eigenfunctions}
  e_{1}^{\mathcal{F}} = \sqrt{\sigma_{1}}e_{1} = \sqrt{\tilde{\sigma}_{1}}e_{1} = e_{1}^{\tilde{\mathcal{F}}}.
\end{equation}
yields \eqref{eq:kernel_perturbation_inequality_reformulated} for $n=1$:
\begin{equation}
  1-e_{1}^{\mathcal{F}}(\bm{x})^{\Tran}\bm{K}(\bm{x})^{-1}e_{1}^{\mathcal{F}}(\bm{x})  \leq  1- e_{1}^{\tilde{\mathcal{F}}}(\bm{x})^{\Tran}\tilde{\bm{K}}(\bm{x})^{-1}e_{1}^{\tilde{\mathcal{F}}}(\bm{x}).
\end{equation}
For $n \neq 1$, the proof is much more subtle. Indeed, a naive application of the inequality \eqref{eq:inequality_between_K_K_tilde} would lead to the following inequality
\begin{equation}
  1-e_{n}^{\tilde{\mathcal{F}}}(\bm{x})^{\Tran}\bm{K}(\bm{x})^{-1}e_{n}^{\tilde{\mathcal{F}}}(\bm{x})  \leq  1- e_{n}^{\tilde{\mathcal{F}}}(\bm{x})^{\Tran}\tilde{\bm{K}}(\bm{x})^{-1}e_{n}^{\tilde{\mathcal{F}}}(\bm{x}).
\end{equation}
Since $\forall n\in\mathbb{N}$, $e_{n}^{\tilde{\mathcal{F}}} = \sqrt{\sigma_1/\sigma_n} e_{n}^{\mathcal{F}}$, we get
\begin{equation}
  1-\sigma_1 e_{n}^{\mathcal{F}}(\bm{x})^{\Tran}\bm{K}(\bm{x})^{-1}e_{n}^{\mathcal{F}}(\bm{x})  \leq  1- \sigma_n e_{n}^{\tilde{\mathcal{F}}}(\bm{x})^{\Tran}\tilde{\bm{K}}(\bm{x})^{-1}e_{n}^{\tilde{\mathcal{F}}}(\bm{x}),
\end{equation}
and hence the unsatisfactory inequality
\begin{equation}
  1- \sigma_1 \Delta_n^{\cal F}(\bm{x})  \leq  1- \sigma_n \Delta_n^{\tilde{\cal F}}(\bm{x})
\end{equation}

We can prove a better inequality by applying a sequence of rank-one updates to the kernel $k$ to build $N$ intermediate kernels $k^{(\ell)}$ that lead to $N$ inequalities sharp enough to prove \eqref{eq:kernel_perturbation_inequality_reformulated} for $n \neq 1$.
Then inequality~(\ref{eq:kernel_perturbation_inequality_reformulated}) will result as a corollary of Proposition~\ref{prop:graded_RKHS} below. To this aim, we define $N$ RKHS $\tilde{\mathcal{F}}_{\ell}$, $1\leq \ell\leq N$, that interpolate between $\mathcal{F}$ and $\tilde{\mathcal{F}}$. For $\ell \in [N]$, define the kernel $\tilde{k}^{(\ell)}$ by 
\begin{equation}
  \tilde{k}^{(\ell)}(x,y) = \sum\limits_{m \in [\ell]} \sigma_{1}e_{m}(x)e_{m}(y) + \sum\limits_{m \geq \ell+1 }\sigma_{m}e_{m}(x)e_{m}(y),
\end{equation}
and let $\tilde{\mathcal{F}}_{\ell}$ the RKHS corresponding to the kernel $\tilde{k}^{(\ell)}$. For $\bm{x} \in \mathcal{X}^{N}$, define $\tilde{\bm{K}}^{(\ell)}(\bm{x}) = (\tilde{k}^{(\ell)}(x_{i},x_{j}))_{1 \leq i,j \leq N}$. Similar to previous notations, we define as well
\begin{equation}
  \Delta_n^{\tilde{\cal F}_\ell}(\bm{x}) = e_{n}^{\tilde{\mathcal{F}}_{\ell}}(\bm{x})^{\Tran}\tilde{\bm{K}}^{(\ell)}(\bm{x})^{-1}e_{n}  ^{\tilde{\mathcal{F}}_{\ell}}(\bm{x}).
\end{equation}
Now we have the following useful proposition.
%
\begin{proposition}\label{prop:graded_RKHS}
For $n \in [N]\smallsetminus\{1\}$, we have
\begin{equation}\label{eq:graded_RKHS_inequality_1}
\sigma_{n} \left(1-\Delta_n^{\tilde{\cal F}_{n-1}}(\bm{x}) \right) \leq \sigma_{1} \left(1- \Delta_n^{\tilde{\cal F}_{n}}(\bm{x}) \right),
\end{equation}
and
\begin{equation}\label{eq:graded_RKHS_inequality_2}
  \forall \: \ell \in [N]\smallsetminus\{1,n\}, \:\: 1-\Delta_n^{\tilde{\cal F}_{\ell-1}}(\bm{x}) \leq  1- \Delta_n^{\tilde{\cal F}_{\ell}}(\bm{x}) .
\end{equation}
\end{proposition}
%
%
For ease of reading, we first show that inequality~\eqref{eq:kernel_perturbation_inequality_reformulated} and therefore Proposition~\ref{prop:kernel_perturbation_inequality} is easily deduced from this Proposition~\ref{prop:graded_RKHS} and then give its proof.

Let $n \in [N]$ such that $n \neq 1$.
We first remark that $\mathcal{F}=\tilde{\mathcal{F}}_{1} $ and use $(n-2)$ times inequality \eqref{eq:graded_RKHS_inequality_2} of Proposition~\ref{prop:graded_RKHS}:
\begin{align}
  \sigma_{n} \left(1- \Delta_n^{\cal F}(\bm{x}) \right) & = \sigma_{n}  \left(1- \Delta_n^{\tilde{\cal F}_{1}}(\bm{x}) \right) \\
  & \leq  \sigma_{n} \left(1-\Delta_n^{\tilde{\cal F}_{n-1}}(\bm{x}) \right) \nonumber
\end{align}
Then we use \eqref{eq:graded_RKHS_inequality_1} that is connected to the rank-one update from the kernel $k^{(n-1)}$ to $k^{(n)}$ so that
\begin{equation}
  \sigma_{n} \left(1-\Delta_n^{\tilde{\cal F}_{n-1}}(\bm{x}) \right)
  \leq
  \sigma_{1} \left(1-\Delta_n^{\tilde{\cal F}_{n}}(\bm{x}) \right)
\end{equation}
Then we apply \eqref{eq:graded_RKHS_inequality_2} to the r.h.s. again $N-n-1$ times to finally get:
\begin{align}
\sigma_{n} \left(1-\Delta_n^{\cal F}(\bm{x}) \right)
& \leq \sigma_{1} \left(1-\Delta_n^{\tilde{\cal F}_{N}}(\bm{x}) \right) \\
& \leq \sigma_{1} \left(1-\Delta_n^{\tilde{\cal F}}(\bm{x}) \right) \nonumber,
\end{align}
since $\tilde{k}^{(N)}=\tilde{k}$ and $\tilde{\mathcal{F}}_{N} = \tilde{\mathcal{F}}$. This concludes the proof of the desired inequality \eqref{eq:kernel_perturbation_inequality_reformulated} and therefore of Proposition~\ref{prop:kernel_perturbation_inequality}.
\end{proof}

\subsubsection{Proof of Proposition~\ref{prop:graded_RKHS}}
%
\begin{proof}(Proposition~\ref{prop:graded_RKHS})
Let $n \in [N]\smallsetminus \{1\}$, and $M \in \mathbb{N}$ such that $M \geq N$. Let $\bm{A}_{\ell} \in \mathbb{R}^{N \times M}$ defined by
\begin{equation}
\forall (i,m) \in [N]\times[M], \: (\bm{A}_{\ell})_{i,m} = e_{m}^{\tilde{\mathcal{F}}_{\ell}}(x_{i}). \footnote{The matrix $\bm{A}_{\ell}$ depends on $\bm{x}$.}
\end{equation}
For $\ell \in [N]$ define
\begin{equation}
\tilde{\bm{K}}_{M}^{(\ell)}(\bm{x}) = \bm{A}_{\ell}^{\Tran}\bm{A}_{\ell}.
\end{equation}
Let  $\bm{W}_{\ell} \in \mathbb{R}^{M \times M}$ the diagonal matrix defined by
\begin{equation}
\bm{W}_{\ell} = \mathrm{diag}(\underbrace{1,...,1}_{\ell-1},\sqrt{\frac{\sigma_{1}}{\sigma_{\ell}}}, 1...,1)
\end{equation}
Then one has the simple relation
\begin{equation}\label{eq:perturbation_of_A}
\bm{A}_{\ell+1} = \bm{A}_{\ell}\bm{W}_{\ell},
\end{equation}
which prepares the use of Lemma~\ref{lemma:rank1_leverage_score} in Section~\ref{subsec:lv_score_updates}.
By definition of the $n$-th leverage score of the matrix $\bm{A}$, see \eqref{eq:def_matrix_leverage_score} in Section~\ref{subsec:lv_score_updates},
\begin{equation}
e_{n}^{\tilde{\mathcal{F}}_{\ell}}(\bm{x})^{\Tran} \tilde{\bm{K}}_{M}^{(\ell)}(\bm{x})^{-1}e_{n}^{\tilde{\mathcal{F}}_{\ell}}(\bm{x})  = e_{n}^{\tilde{\mathcal{F}}_{\ell}}(\bm{x})^{\Tran} \left(\bm{A}_{\ell}^{\Tran}\bm{A}_{\ell} \right)^{-1}e_{n}^{\tilde{\mathcal{F}}_{\ell}}(\bm{x}) = \tau_{n} \left(\bm{A}_{\ell} \right).
\end{equation}
Define similarly $\Delta_{n,M}^{\tilde{\cal F}_{\ell}}(\bm{x}) = e_{n}^{\tilde{\mathcal{F}}_{\ell}}(\bm{x})^{\Tran} \tilde{\bm{K}}_{M}^{(\ell)}(\bm{x})^{-1}e_{n}^{\tilde{\mathcal{F}}_{\ell}}(\bm{x})$.
Thanks to \eqref{eq:lv_score_update} of Lemma~\ref{lemma:rank1_leverage_score} and \eqref{eq:perturbation_of_A} and for $\ell = n$
\begin{equation}
\tau_{n} \Big(\bm{A}_{n}\Big) = \tau_{n} \left( \bm{A}_{n-1}\bm{W}_{n} \right) = \frac{(1+\rho_{n})\tau_{n} \Big(\bm{A}_{n-1} \Big)}{1+\rho_{n} \tau_{n} \Big( \bm{A}_{n-1} \Big)},
\end{equation}
where $\displaystyle \rho_{n} = \frac{\sigma_{1}}{\sigma_{n}} - 1$.
Thus
\begin{equation}
  1-\tau_{n} \Big(\bm{A}_{n} \Big)  =1- \frac{(1+\rho_{n})\tau_{n} \Big(\bm{A}_{n-1} \Big)}{1+\rho_{n} \tau_{n} \Big( \bm{A}_{n-1} \Big)} = \frac{1-\tau_{n}  \Big( \bm{A}_{n-1} \Big)}{1+\rho_{n} \tau_{n} \Big( \bm{A}_{n-1} \Big) }.
\end{equation}
Then
\begin{align}
    \sigma_{1} \left(1-\tau_{n} \Big(\bm{A}_{n} \Big)\right) & = \sigma_{1} \frac{1-\tau_{n} \Big( \bm{A}_{n-1} \Big)}{1+\rho_{n} \tau_{n} \Big( \bm{A}_{n-1} \Big)} \\
    & = \sigma_{n}(1+ \rho_{n})\frac{1-\tau_{n} \Big( \bm{A}_{n-1} \Big)}{1+\rho_{n} \tau_{n} \Big( \bm{A}_{n-1} \Big)} \nonumber\\
    & = \frac{1+ \rho_{n}}{1+\rho_{n} \tau_{n} \Big( \bm{A}_{n-1} \Big)} \sigma_{n}\left(1-\tau_{n} \Big( \bm{A}_{n-1} \Big)\right) \nonumber\\
    & \geq \sigma_{n}\left(1-\tau_{n} \Big( \bm{A}_{n-1} \Big)\right) \nonumber,
\end{align}
since $\rho_{n} \geq 0$ and $\tau_{n} \Big( \bm{A}_{n-1} \Big) \in [0,1]$ thanks to \eqref{eq:bounds_on_lv_scores}.
This proves that for $M \in \mathbb{N}^{*}$ such that $M \geq N$,
\begin{equation}\label{eq:graded_RKHS_truncated_inequality_1}
\sigma_{n} \left(1-\Delta_{n,M}^{\tilde{\cal F}_{n-1}}(\bm{x}) \right) \leq \sigma_{1} \left(1- \Delta_{n,M}^{\tilde{\cal F}_{n}}(\bm{x}) \right).
\end{equation}
Now,
\begin{eqnarray}
  \lim\limits_{M \rightarrow \infty} \tilde{\bm{K}}_{M}^{(n+1)}(\bm{x}) & = & \tilde{\bm{K}}^{(n+1)}(\bm{x}),\\
  \lim\limits_{M \rightarrow \infty} \tilde{\bm{K}}_{M}^{(n)}(\bm{x}) & =  & \tilde{\bm{K}}^{(n)}(\bm{x}).
\end{eqnarray}
Moreover the application $\bm{X} \mapsto \bm{X}^{-1}$ is continuous in $GL_{N}(\mathbb{R})$.
This proves the inequality~\eqref{eq:graded_RKHS_inequality_1} of Proposition~\ref{prop:graded_RKHS}.
To  prove the inequality~\eqref{eq:graded_RKHS_inequality_2}, we start by using~\eqref{eq:cross_lv_score_update}:
\begin{equation}
  \forall \ell \in [N]\smallsetminus\{1,n\}, \:\: \tau_{n} \Big(\bm{A}( \ell) \Big) = \tau_{n} \left(\bm{A}_{\ell-1} \bm{W}_{\ell} \right) \leq \tau_{n} \Big(\bm{A}_{\ell-1} \Big).
\end{equation}
which implies that
\begin{equation}
  \forall \ell \in [N]\smallsetminus\{1,n\}, 1- \tau_{n} \Big(\bm{A}_{\ell-1} \Big) \leq 1-\tau_{n} \Big(\bm{A}_{\ell} \Big).
\end{equation}
Then for $M \geq N$,
\begin{equation}\label{eq:graded_RKHS_truncated_inequality_2}
  \forall \: \ell \in [N]\smallsetminus\{1,n\}, \:\: 1- \Delta_{n,M}^{\tilde{\cal F}_{\ell-1}}(\bm{x}) \leq  1- \Delta_{n,M}^{\tilde{\cal F}_{\ell}}(\bm{x}) .
\end{equation}
As above, we conclude the proof by considering the limit $M\to\infty$
\begin{equation}
  \forall \: \ell \in [N]\smallsetminus\{1,n\}, \:\:\lim\limits_{M \rightarrow \infty} \tilde{\bm{K}}_{M}^{(\ell)}(\bm{x}) =  \tilde{\bm{K}}^{(\ell)}(\bm{x}).
\end{equation}
This proves inequality~\eqref{eq:graded_RKHS_inequality_2} and concludes the proof of Proposition~\ref{prop:graded_RKHS}.
\end{proof}



\subsection{Proof of Proposition~\ref{prop:expected_value_of_product_of_cos}}
\label{s:proofOfExpectedProduct}
%

In this section, $\bm{x}  = (x_{1}, \dots , x_{N}) \in \mathcal{X}^{N}$ is the realization of the DPP of Theorem~\ref{thm:main_theorem}. Let $\bm{E}^{\mathcal{F}}(\bm{x}) = (e_{i}^{\mathcal{F}}(x_{j}))_{1 \leq i,j \leq N}$ and $\bm{E}(\bm{x}) = (e_{i}(x_{j}))_{1 \leq i,j \leq N}$, and $\bm{K}(\bm{x})= (k(x_{i},x_{j}))_{1 \leq i,j \leq N} $.
Moreover, let $\mathcal{E}^{\mathcal{F}}_{N} = \Span(e_{m}^{\mathcal{F}})_{m \in [N]}$ and $\mathcal{T}(\bm{x}) =  \Span \left( k(x_{i},.) \right)_{i \in [N]}$.

We first prove two lemmas that are necessary to prove Proposition~\ref{prop:expected_value_of_product_of_cos}.

\subsubsection{Two preliminary lemmas}

\begin{lemma}\label{lemma:cos_ratio_det}
Let $\bm{x} =  (x_{1}, \dots , x_{N}) \in \mathcal{X}^{N} $ such that $\Det^{2} \bm{E}(\bm{x}) \neq 0$. Then,
\begin{equation}
\prod\limits_{\ell \in [N]} \frac{1}{\cos^{2} \theta_{\ell} \left(\mathcal{E}^{\mathcal{F}}_{N}, \mathcal{T}(\bm{x}) \right)} = \frac{\Det \bm{K}(\bm{x})}{\Det^{2} \bm{E}^{\mathcal{F}}(\bm{x})}.
\end{equation}
\end{lemma}

\begin{proof}
The condition $\Det^{2} \bm{E}(\bm{x}) \neq 0$ yields by Proposition~\ref{prop:K_N_non_singular} that $\bm{K}(\bm{x})$ is non singular. Thus $\dim \mathcal{T}(\bm{x}) = N$. Let $(t_{i})_{i \in [N]}$ an orthonormal basis of $\mathcal{T}(\bm{x})$ with respect to $\langle ., . \rangle_{\mathcal{F}}$.
Using Corollary~\ref{cor:cos_det_relationship}, and the fact that $(e_{n}^{\mathcal{F}})_{n \in [N]}$ is an orthonormal basis of $\mathcal{E}^{\mathcal{F}}_{N}$ according to $\langle ., . \rangle_{\mathcal{F}}$,
\begin{equation}\label{eq:prod_cos_det_E}
\prod\limits_{\ell \in [N]} \cos^{2} \theta_{\ell} \left(\mathcal{E}^{\mathcal{F}}_{N}, \mathcal{T}(\bm{x}) \right) = \Det^{2} (\langle e_{n}^{\mathcal{F}}, t_{i} \rangle_{\mathcal{F}})_{(n,i) \in [N]\times[N]}.
\end{equation}
Now, write for $i \in [N]$,
\begin{equation}\label{eq:t_as_function_of_kx}
t_{i} = \sum\limits_{j \in [N]} c_{i,j} k(x_{j},.).
\end{equation}
Thus
\begin{align}
\langle e_{n}^{\mathcal{F}}, t_{i} \rangle_{\mathcal{F}} = & \sum\limits_{j \in [N]} c_{i,j} \langle e_{n}^{\mathcal{F}}, k(x_{j},.) \rangle_{\mathcal{F}} \\
= &\sum\limits_{j \in [N]} c_{i,j}  e_{n}^{\mathcal{F}}(x_{j}).
\end{align}
Then
\begin{equation}
(\langle e_{n}^{\mathcal{F}}, t_{i} \rangle_{\mathcal{F}})_{(n,i) \in [N]\times[N]} = \bm{E}^{\mathcal{F}}(\bm{x}) \bm{C}(\bm{x})^{\Tran} ,
\end{equation}
where
\begin{equation}
\bm{C}(\bm{x}) = (c_{i,j})_{1 \leq i,j \leq N}.
\end{equation}
Thus
\begin{equation}\label{eq:AN_times_EN}
\Det^{2} (\langle e_{n}^{\mathcal{F}}, t_{i} \rangle_{\mathcal{F}})_{(n,i) \in [N]\times[N]} = \Det^{2} \bm{C}(\bm{x}) \Det^{2} \bm{E}^{\mathcal{F}}(\bm{x}).
\end{equation}
Now, let $\bm{c}_{i}$ the columns of the matrix $\bm{C}(\bm{x})$. $(t_{i})_{i \in [N]}$ is an orthonormal basis of $\mathcal{T}(\bm{x})$ with respect to $\langle .,. \rangle_{\mathcal{F}}$, then by \eqref{eq:t_as_function_of_kx}
\begin{equation}
  \delta_{i,i'} = \langle t_{i}, t_{i'} \rangle_{\mathcal{F}} = \bm{c}_{i}^{\Tran} \bm{K}(\bm{x}) \bm{c}_{i'}  .
\end{equation}
Therefore
\begin{equation}
\bm{C}(\bm{x})^{\Tran} \bm{K}(\bm{x}) \bm{C}(\bm{x}) = \mathbb{I}_{N}.
\end{equation}
Thus
\begin{equation}\label{eq:det_AN_det_KN_relationship}
\Det^{2} \bm{C}(\bm{x}) = \frac{1}{\Det \bm{K}(\bm{x})}.
\end{equation}
Combining \eqref{eq:prod_cos_det_E}, \eqref{eq:AN_times_EN} and \eqref{eq:det_AN_det_KN_relationship} concludes the proof of Lemma~\ref{lemma:cos_ratio_det}:
\begin{equation}
\prod\limits_{\ell \in [N]} \frac{1}{\cos^{2} \theta_{\ell} \left(\mathcal{E}^{\mathcal{F}}_{N}, \mathcal{T}(\bm{x}) \right)} = \frac{\Det \bm{K}(\bm{x})}{\Det^{2} \bm{E}^{\mathcal{F}}(\bm{x})}.
\end{equation}
\end{proof}

\begin{lemma}\label{lemma:truncated_Fredholm_formula}
\begin{equation}
  \frac{1}{N!} \int_{\mathcal{X}^{N}}\Det \bm{K}(x_{1}, \dots, x_{N}) \otimes_{j \in [N]} \mathrm{d}\omega(x_{j})  = \sum\limits_{\substack{T \subset \mathbb{N}^{*} \\ |T| = N}}  \prod\limits_{t \in T}\sigma_{t}.
\end{equation}
\end{lemma}
\begin{proof}
Let $\bm{x} = (x_{1}, \dots, x_{N}) \in \mathcal{X}^{N}$. From \eqref{eq:truncated_kernel_limit}
\begin{equation}
\Det \bm{K}(\bm{x}) = \lim\limits_{M \rightarrow \infty} \Det \bm{K}_{M}(\bm{x}).
\end{equation}
Moreover,
\begin{equation}
\Det \bm{K}_{M}(\bm{x})  = \sum\limits_{T \subset [M], |T| = N} \prod\limits_{i \in T} \sigma_{i} \Det^{2} (e_{i}(x_{j}))_{(i,j)\in T \times [N]}.
\end{equation}
Now, for $T \subset [M]$ such that $|T| = N$, $(e_{t})_{t \in T}$ is an orthonormal family of $\mathbb{L}_{2}(\mathrm{d}\omega)$, then by \cite{HKPV06} Lemma 21:
\begin{equation}
\int_{\mathcal{X}^{N}}\Det^{2} (e_{t}(x_{j})) \otimes_{j \in [N]} \mathrm{d}\omega(x_{j}) = N! .
\end{equation}

Thus
\begin{align}
\frac{1}{N!} \int_{\mathcal{X}^{N}}\Det \bm{K}_{M}(\bm{x}) \otimes_{j \in [N]} \mathrm{d}\omega(x_{j}) &=  \frac{1}{N!} \sum\limits_{T \subset [M], |T| = N} \prod\limits_{t \in T} \sigma_{t} \int_{\mathcal{X}^{N}}\Det^{2} (e_{t}(x_{j})) \otimes_{j \in [N]} \mathrm{d}\omega(x_{j}) \\
& = \sum\limits_{T \subset [M], |T| = N} \prod\limits_{t \in T} \sigma_{t} \nonumber.
\end{align}
Now, $\displaystyle \sum\limits_{n \in \mathbb{N}^{*}} \sigma_{n} < \infty$ implies that $\displaystyle \sum\limits_{T \subset \mathbb{N}^{*}, |T| = N} \prod\limits_{t \in T} \sigma_{t} < \infty$. In fact, for $\ell \in [N]$ let $p_{\ell}$ the $\ell$-th symmetric polynomial. By Maclaurin’s inequality \citep{Ste04}, and for any vector $\bm{\nu} \in \mathbb{R}_{+}^{M}$
\begin{equation}
\left( \frac{p_{\ell}(\bm{\nu})}{{M\choose \ell}} \right)^{\frac{1}{\ell}} \leq \frac{p_{1}(\bm{\nu})}{M}.
\end{equation}
Thus
\begin{align}\label{eq:Maclaurin_inequality}
p_{\ell}(\bm{\nu}) & \leq  \frac{{M\choose \ell}}{M^{\ell}}p_{1}(\bm{\nu})^{\ell} \\
& \leq \frac{M!}{\ell!(M-\ell)! M^{\ell}}p_{1}(\bm{\nu})^{\ell} \nonumber\\
& \leq \frac{M (M-1) \dots (M-\ell +1)}{\ell! M^{\ell}}p_{1}(\bm{\nu})^{\ell} \nonumber\\
& \leq \frac{1}{\ell !}p_{1}(\bm{\nu})^{\ell} \nonumber.
\end{align}
This inequality is independent of the dimension $M$ thus it can be extended for $\bm{\nu} \in \mathbb{R}_{+}^{\mathbb{N}^{*}}$ with $\displaystyle \sum\limits_{n \in \mathbb{N}^{*}} \nu_{n} < \infty$. Therefore
\begin{equation}
\sum\limits_{T \subset \mathbb{N}^{*}, |T| = N} \prod\limits_{t \in T} \sigma_{t}  \leq \frac{1}{N!} (\sum\limits_{n \in \mathbb{N}^{*}} \sigma_{n})^{N} < \infty .
\end{equation}

Furthermore,
\begin{equation}
\forall M \in \mathbb{N}^{*}, \: \forall \bm{x} \in \mathcal{X}^{N}, \: 0 \leq \Det \bm{K}_{M}(\bm{x}) \leq \Det \bm{K}_{M+1,N}(\bm{x}).
\end{equation}
Then by monotone convergence theorem, $\displaystyle \bm{x} \mapsto \frac{1}{N!}\Det \bm{K}(\bm{x})$ is mesurable and
\begin{align}
\int_{\mathcal{X}^{N}} \frac{1}{N!}\Det \bm{K}(\bm{x}) \otimes_{j \in [N]} \mathrm{d}\omega(x_j) & = \lim\limits_{M \rightarrow \infty} \int_{\mathcal{X}^{N}}\frac{1}{N!} \Det \bm{K}_{M}(\bm{x})\otimes_{j \in [N]} \mathrm{d}\omega(x_j)\\
& = \lim\limits_{M \rightarrow \infty} \sum\limits_{T \subset [M], |T| = N} \prod\limits_{t \in T} \sigma_{t} \nonumber\\
& = \sum\limits_{T \subset \mathbb{N}^{*}, |T| = N} \prod\limits_{t \in T} \sigma_{t} \nonumber.
\end{align}
\end{proof}

\subsubsection{End of the proof of Proposition~\ref{prop:expected_value_of_product_of_cos}}
\begin{proof}
Remember that
\begin{equation}
\Prb \left( \Det \bm{E}(\bm{x})  \neq 0 \right) = 1.
\end{equation}
Then by Lemma~\ref{lemma:cos_ratio_det} and the fact that $\Det^{2} \bm{E}^{\mathcal{F}}(\bm{x}) = \prod\limits_{n \in [N]} \sigma_{n} \Det^{2} \bm{E}(\bm{x})$
\begin{equation}
\prod\limits_{\ell \in [N]} \frac{1}{\cos^{2} \theta_{\ell} \left(\mathcal{E}^{\mathcal{F}}_{N}, \mathcal{T}(\bm{x}) \right)} = \frac{\Det \bm{K}(\bm{x})}{\Det^{2} \bm{E}^{\mathcal{F}}(\bm{x})} = \frac{1}{\prod\limits_{n \in [N]} \sigma_{n}}\frac{\Det \bm{K}(\bm{x})}{\Det^{2} \bm{E}(\bm{x})}.
\end{equation}
Then, taking the expectation with respect to $\bm{x}$ resulting from a DPP of kernel $\KDPP(x,y)$,
\begin{align}
\EX_{\DPP} \prod\limits_{\ell \in [N]} \frac{1}{\cos^{2} \theta_{\ell}\left(\mathcal{E}^{\mathcal{F}}_{N}, \mathcal{T}(\bm{x}) \right)} = & \frac{1}{N!} \int_{\mathcal{X}^{N}} \Det^{2} \bm{E}(\bm{x}) \prod\limits_{\ell \in [N]} \frac{1}{\cos^{2} \theta_{\ell} \left(\mathcal{E}^{\mathcal{F}}_{N}, \mathcal{T}(\bm{x}) \right)} \otimes_{i =1}^{N} \mathrm{d}\omega(x_{i}) \\
 = & \frac{1}{N!} \int_{\mathcal{X}^{N}} \Det^{2} \bm{E}(\bm{x}) \frac{1}{\prod\limits_{n \in [N]} \sigma_{n}}\frac{\Det \bm{K}(\bm{x})}{\Det^{2} \bm{E}(\bm{x})} \otimes_{i =1}^{N} \mathrm{d}\omega(x_{i}) \nonumber\\
 = & \frac{1}{\prod\limits_{n \in [N]} \sigma_{n}} \frac{1}{N!} \int_{\mathcal{X}^{N}}\Det \bm{K}(\bm{x}) \otimes_{i =1}^{N} \mathrm{d}\omega(x_{i}) \nonumber.
\end{align}
Now, by Lemma~\ref{lemma:truncated_Fredholm_formula}
\begin{equation}
\frac{1}{N!} \int_{\mathcal{X}^{N}}\Det \bm{K}(\bm{x}) \otimes_{i =1}^{N} \mathrm{d}\omega(x_{i})  = \sum\limits_{\substack{T \subset \mathbb{N}^{*} \\ |T| = N}}  \prod\limits_{t \in T}\sigma_{t}.
\end{equation}
Therefore,
\begin{equation}
\EX_{\DPP} \prod\limits_{\ell \in [N]} \frac{1}{\cos^{2} \theta_{\ell}\left(\mathcal{E}^{\mathcal{F}}_{N}, \mathcal{T}(\bm{x}) \right)}  =  \sum\limits_{\substack{T \subset \mathbb{N}^{*} \\ |T| = N}} \frac{ \prod\limits_{t \in T}\sigma_{t}}{\prod\limits_{n \in [N]} \sigma_{n}}.
\end{equation}

\end{proof}

\subsection{Proof of Theorem~\ref{thm:main_theorem}}
\label{s:finalBound}
\begin{proof}
Thanks to Proposition~\ref{prop:kernel_perturbation_inequality} and Lemma~\ref{lemma:max_error_cos} (for $\tilde{\cal F}$ and $\tilde{k}$)
\begin{align}\label{eq:prooftheorem1_first}
  \max_{ n \in [N]}\sigma_n \|\bm{\Pi}_{\mathcal{T}(\bm{x})^{\perp}} e_{n}^{\mathcal{F}}\|_{\mathcal{F}}^{2}
  & \leq
  \sigma_1 \cdot \max_{ n \in [N]} \|\bm{\Pi}_{\tilde{\cal{T}}(\bm{x})^{\perp}} e_{n}^{\tilde{\mathcal{F}}}\|_{\tilde{\mathcal{F}}}^{2}\\
  & \leq
  \sigma_1 \cdot \left( \prod\limits_{n \in [N]}\frac{1}{\cos^{2} \theta_{n}(\tilde{\cal{T}}(\bm{x}),\mathcal{E}^{\tilde{\mathcal{F}}}_{N})} - 1\right).
\end{align}
%
%
Then Proposition~\ref{prop:expected_value_of_product_of_cos} applied to $\tilde{\cal F}$ with kernel $\tilde{k}$ yields
\begin{equation}
  \EX_{\DPP} \prod\limits_{n \in [N]} \frac{1}{\cos^{2} \theta_{n}\left(\mathcal{E}^{\tilde{\mathcal{F}}}_{N}, \tilde{\mathcal{T}}_{N}(\bm{x}) \right)}   =
  \sum\limits_{\substack{T \subset \mathbb{N}^{*} \\ |T| = N}} \frac{ \prod\limits_{t \in T}\tilde{\sigma}_{t}}{\prod\limits_{n \in [N]} \tilde{\sigma}_{n}}.
\end{equation}
Every subset $T \subset \mathbb{N}^{*}$ such that $|T| = N$ can be written as $T = V \cup W$ with $V \subset [N]$ and $W \subset \mathbb{N}^{*} \smallsetminus [N]$, and this decomposition is unique. Then
\begin{equation}
\frac{ \prod\limits_{t \in T}\tilde{\sigma}_{t}}{\prod\limits_{n \in [N]} \tilde{\sigma}_{n}} = \frac{\prod\limits_{v \in V}\tilde{\sigma}_{v} \prod\limits_{w \in W}\tilde{\sigma}_{w}}{\prod\limits_{n \in [N]} \tilde{\sigma}_{n}} = \frac{\prod\limits_{w \in W}\tilde{\sigma}_{w}}{\prod\limits_{n \in [N]\smallsetminus V} \tilde{\sigma}_{n}}.
\end{equation}
Therefore
\begin{align}
  \sum\limits_{\substack{T \subset \mathbb{N}^{*} \\ |T| = N}} \frac{ \prod\limits_{t \in T}\tilde{\sigma}_{t}}{\prod\limits_{n \in [N]} \tilde{\sigma}_{n}} & = \sum\limits_{\substack{T \subset \mathbb{N}^{*} \\ |T| = N\\ T = V \cup W}} \frac{\prod\limits_{w \in W}\tilde{\sigma}_{w}}{\prod\limits_{n \in [N]\smallsetminus V} \tilde{\sigma}_{n}} \\
  & = \sum\limits_{V \subset [N]}\sum\limits_{\substack{W \subset \mathbb{N}^{*}\smallsetminus [N]\\ |W| = N-|V|}} \frac{\prod\limits_{w \in W}\tilde{\sigma}_{w}}{\prod\limits_{n \in [N]\smallsetminus V} \tilde{\sigma}_{n}} \nonumber\\
  & = \sum\limits_{0 \leq \ell \leq N} \bigg[ \sum\limits_{\substack{V \subset [N]\\ |V| = \ell}}\prod\limits_{n \in [N]\smallsetminus V} \frac{1}{\tilde{\sigma}_{n}}\bigg]\bigg[\sum\limits_{\substack{W \subset \mathbb{N}^{*}\smallsetminus [N]\\ |W| = N-\ell}} \prod\limits_{w \in W}\tilde{\sigma}_{w}\bigg] \nonumber\\
  & = \sum\limits_{0 \leq \ell \leq N} \bigg[\sum\limits_{\substack{V \subset [N]\\ |V| = N-\ell}}\prod\limits_{n \in V} \frac{1}{\tilde{\sigma}_{n}}\bigg] \bigg[\sum\limits_{\substack{W \subset \mathbb{N}^{*}\smallsetminus [N]\\ |W| = N-\ell}} \prod\limits_{w \in W}\tilde{\sigma}_{w}\bigg] \nonumber\\
  & = \sum\limits_{0 \leq \ell \leq N} p_{N-\ell}\left(\left(\frac{1}{\tilde{\sigma}_{m}}\right)_{m \in [N]}\right)p_{N-\ell}\left((\tilde{\sigma}_{m})_{m \geq N+1}\right) \nonumber\\
  & = \sum\limits_{0 \leq \ell \leq N} p_{\ell}\left(\left(\frac{1}{\tilde{\sigma}_{m}}\right)_{m \in [N]}\right)p_{\ell}\left((\tilde{\sigma}_{m})_{m \geq N+1}\right) \nonumber,
\end{align}
where for $\ell \in [N]$, $p_{\ell}$ is the $\ell$-th symmetric polynomial with the convention that $p_{0} = 1$.

Finally, thanks to \eqref{eq:Maclaurin_inequality} above
\begin{align}
  \sum\limits_{\substack{T \subset \mathbb{N}^{*} \\ |T| = N}} \frac{ \prod\limits_{t \in T}\tilde{\sigma}_{t}}{\prod\limits_{n \in [N]} \tilde{\sigma}_{n}}
  & \leq  1 + \sum\limits_{\ell \in [N]} \frac{1}{\ell!^{2}} \left(\sum\limits_{m \in [N]}\frac{1}{\tilde{\sigma}_{m}} \sum\limits_{m \geq N+1} \tilde{\sigma}_{m}\right)^{\ell} \\
  & \leq 1+ \sum\limits_{\ell \in [N]} \frac{1}{\ell!^{2}} \left(\frac{N}{\sigma_1} \sum\limits_{m \geq N+1} \sigma_{m}\right)^{\ell} \nonumber.
\end{align}
As a consequence, by writing $r_{N} = \sum\limits_{m \geq N+1} \sigma_{m}$,
\begin{equation}
  \EX_{\DPP} \left[ \max_{ n \in [N]}\sigma_n \|\bm{\Pi}_{\mathcal{T}(\bm{x})^{\perp}} e_{n}^{\mathcal{F}}\|_{\mathcal{F}}^{2} \right]
  \leq
  \sigma_1 \cdot \sum_{\ell=1}^{N} \frac{1}{\ell!^{2}} \left(\frac{N  r_{N} }{\sigma_1}\right)^{\ell}
\end{equation}
which can be plugged in Lemma~\ref{lemma:approximation_error_spectral_bound} to conclude the proof.
\end{proof}


\section{The intuitions behind the algorithm}
The algorithm presented in this article is based on several intuitions. In this section, we summarize these intuitions.
\subsection{The geometric intuition}\label{app:geometric_intuition}
Recall that the quadrature problem in a RKHS boils down to a problem of interpolation of the mean element $\mu_{g}$ by a mixture of $k(x_{i},.)$, where $g \in \mathbb{L}_{2}(\mathrm{d}\omega)$ such that $\|g\|_{\mathrm{d}\omega} \leq 1$. A promising algorithm would thus be to select the nodes $\{x_i, i\in [N]\}$ so as to minimize the projection of $\mu_{g}$ onto $\mathcal{T}(\bm{x}) = \Span(k(x_i,\cdot); i\in [N])$. Upper bounding the approximation error $\|\mu_{g} - \bm{\Pi}_{\mathcal{T}(\bm{x})}\mu_{g}\|_{\mathcal{F}}$ is not easy in general. One the one side, we propose to replace $\mu_{g}$ by its projection $\bm{\Pi}_{\mathcal{E}_{N}^{\mathcal{F}}} \mu_{g}$ onto the first eigenfunctions of $\bm{\Sigma}$. Then it is easy to prove that
\begin{equation}
\|\mu_{g}- \bm{\Pi}_{\mathcal{E}_{N}^{\mathcal{F}}} \mu_{g}\|_{\mathcal{F}} \leq \sqrt{\sigma_{N+1}}.
\label{e:pca}
\end{equation}
On the other side, if we find a quadrature rule such that $\|\bm{\Pi}_{\mathcal{E}_{N}^{\mathcal{F}}}\mu_g - \bm{\Pi}_{\mathcal{T}(\bm{x})}\mu_{g}\|_{\mathcal{F}}$ is small, then we can guarantee an overall approximation error that is not too much worse than the PCA error \eqref{e:pca}.
After introducing an auxiliary RKHS $\tilde{\mathcal{F}}$ with kernel $\tilde k$, we express this second term using the principal angles between the subspaces $\tilde{\mathcal{T}}(\bm{x})$ and $\mathcal{E}_{N}^{\tilde{\mathcal{F}}}$ (see section \ref{subsec:expectation_under_the_DPP}). This yields a bound on the interpolation error
\begin{equation}
    \|\mu_{g} - \bm{\Pi}_{\mathcal{T}(\bm{x})}\mu_{g}\|^{2}_{\mathcal{F}} \leq 2 \bigg( \sigma_{N+1} +    \sigma_{1}  \|g\|_{\mathrm{d}\omega,1}^2\tan^{2} \theta_{N} \left(\mathcal{E}^{\tilde{\mathcal{F}}}_{N}, \tilde{\mathcal{T}}(\bm{x}) \right) \bigg) .
\label{e:boundingOrthogonalProjection_2}
\end{equation}

The first term in the right hand side of \eqref{e:boundingOrthogonalProjection_2} is $2 \sigma_{N+1}$, which corresponds to the approximation error observed in numerical simulations. The second term depends on the largest principal angle $\theta_{N}$ between the subspaces $\tilde{\mathcal{T}}(\bm{x})$ and $\mathcal{E}_{N}^{\tilde{\mathcal{F}}}$, see Figure~\ref{fig:func_principal_angles}. This term can in turn be bounded by the symmetrized quantity
\begin{equation}
\prod\limits_{\ell \in [N]} \frac{1}{\cos^{2} \theta_{\ell} \left(\mathcal{E}^{\tilde{\mathcal{F}}}_{N}, \tilde{\mathcal{T}}(\bm{x}) \right)} - 1 = \frac{\Det \tilde{\bm{K}}(\bm{x})}{\Det^{2} \bm{E}^{\tilde{\mathcal{F}}}(\bm{x})} - 1,
\end{equation}
which has a tractable expectation under the projection DPP that we consider in this paper.
As an illustration of \eqref{e:boundingOrthogonalProjection_2}, Figure~\ref{fig:principal_angles_comparison} compares the quality of approximation of a mean element $\mu_{g}$ using kernel interpolation based on two configurations of nodes: the first configuration (top) is well spread and the second configuration (bottom) is not. Observe that the largest principal angle $\theta_{N}$ for the first configuration is around $\pi/4$, so that $\tan^{2} \theta_{N} \approx 1$; while it is around $\pi/2$ for the second configuration so that $\tan^{2} \theta_{N} \gg 1$. Now observe that the first design of nodes gives the best reconstruction. This observation is consistent with \eqref{e:boundingOrthogonalProjection_2}.

\begin{figure}[]
\centering
\input{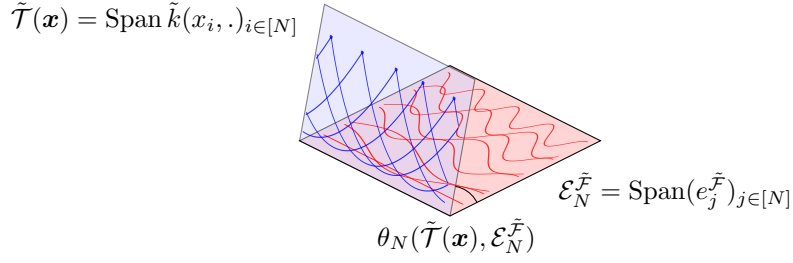}
 \caption{Illustration of the largest principal angle between the subspaces $\tilde{\mathcal{T}}(\bm{x})$ and $\mathcal{E}_{N}^{\tilde{\mathcal{F}}}$ in the case of the RKHS of Section~\ref{s:sobolev_numsim} (the periodic Sobolev space of order 1).  \label{fig:func_principal_angles}}
\end{figure}

\begin{figure}[]
\centering
\includegraphics[width= 0.85\textwidth]{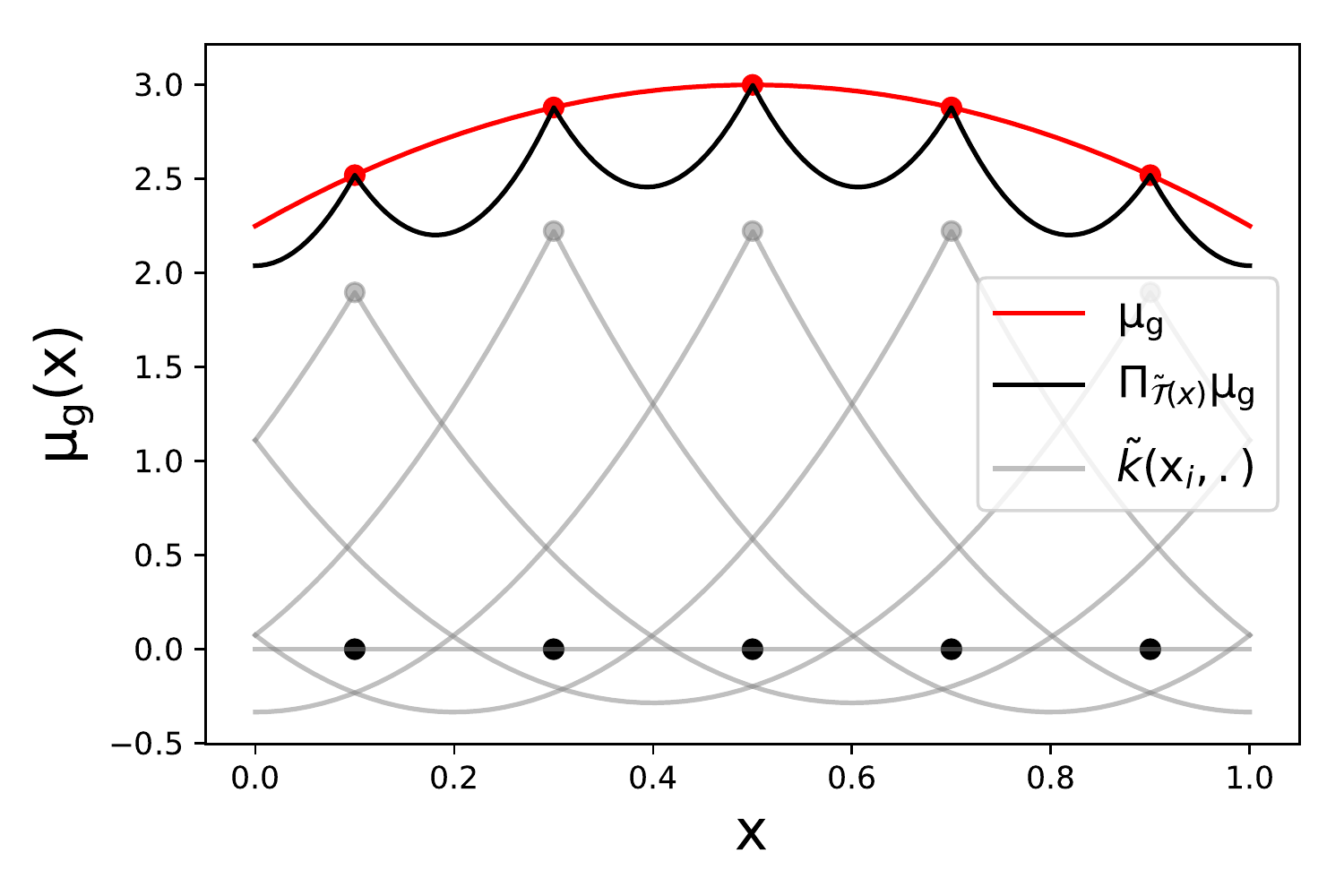}
\includegraphics[width= 0.25\textwidth]{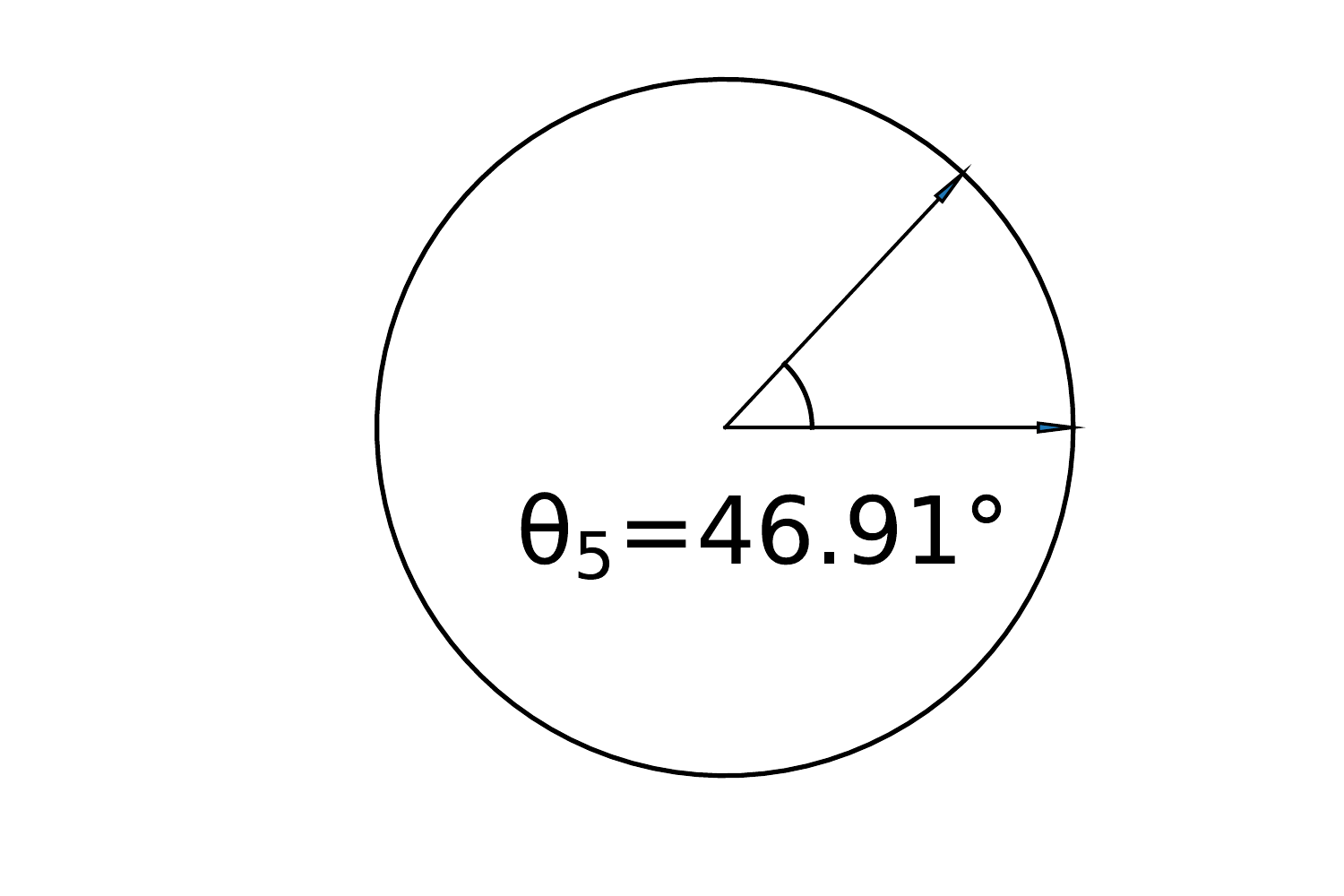}
\\
\includegraphics[width= 0.85\textwidth]{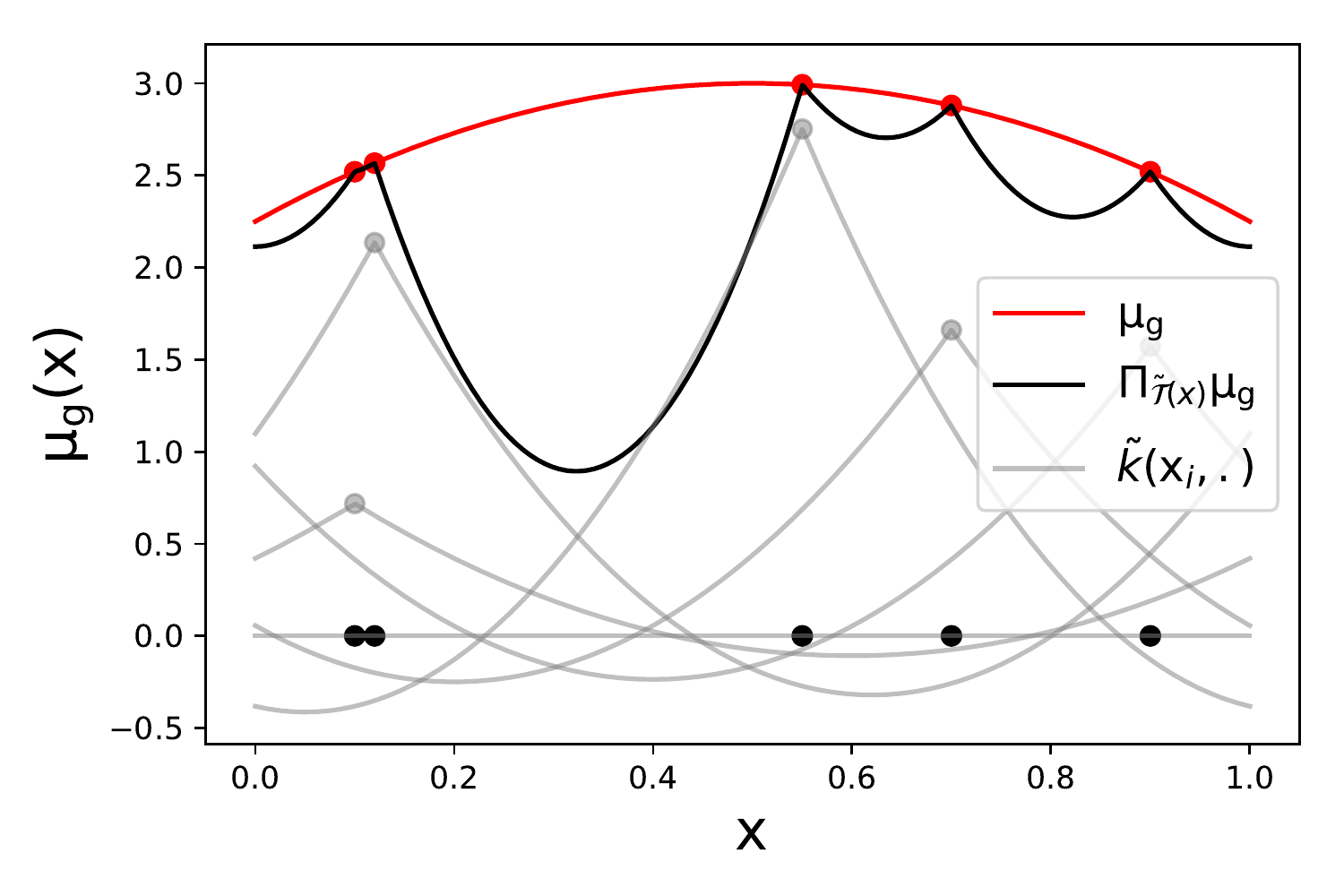}
\includegraphics[width= 0.25\textwidth]{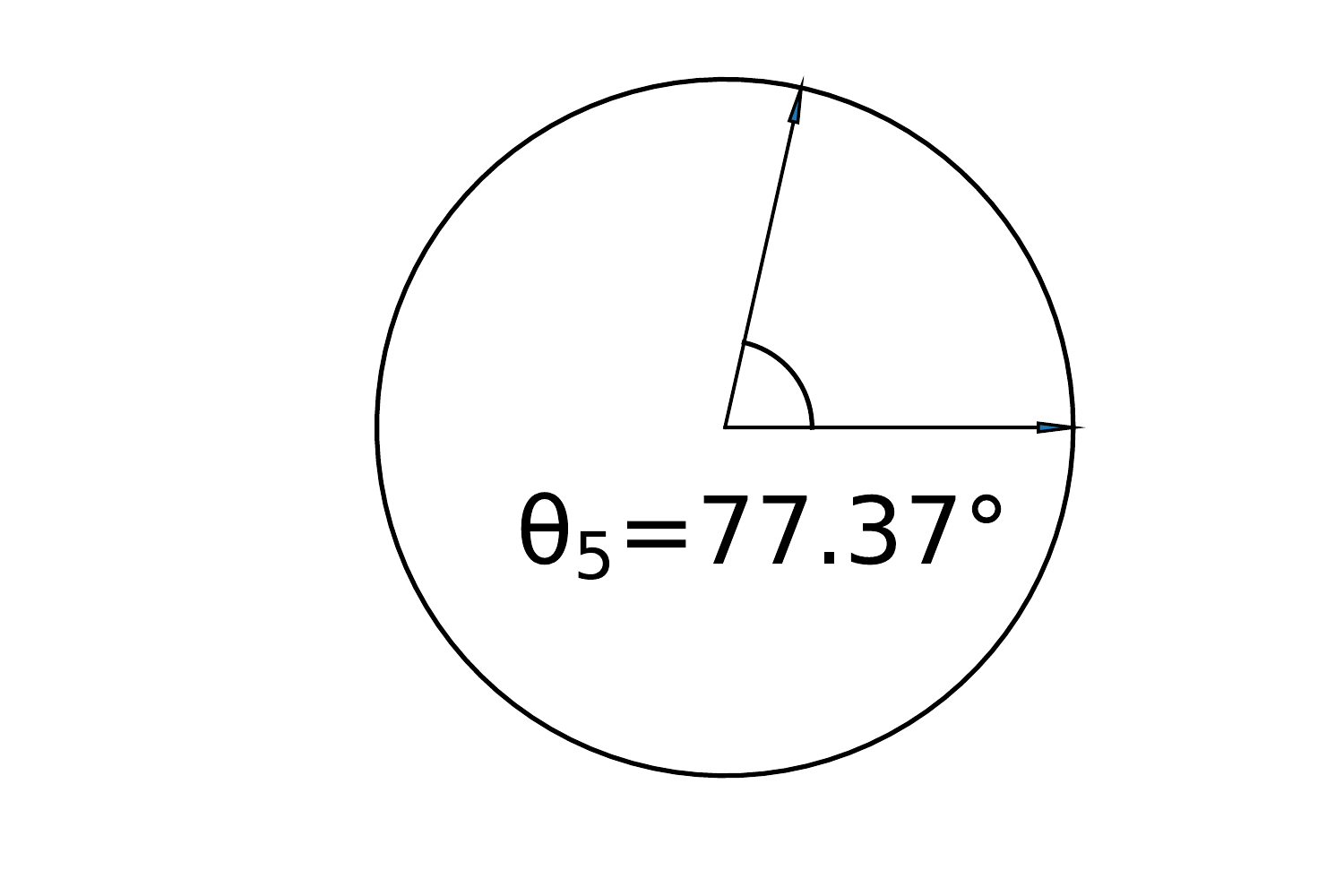}
 \caption{The dependency of the quality of reconstruction and the largest principal angle  $ \theta_{N}~=~\theta_{N}~\bigg(\tilde{\mathcal{T}}(\bm{x}),\mathcal{E}_{N}^{\tilde{\mathcal{F}}}\bigg)$ for $N = 5$. A comparison of a design of nodes well-spread (above) and a design of nodes with clustering (below). \label{fig:principal_angles_comparison}}
\end{figure}

\subsection{The inclusion probability of DPPs and the Christoffel functions}\label{app:christoffel}
The optimal distribution $q_{\lambda}$, see section \ref{s:bach}, can be linked to the so-called \emph{Christoffel functions} \citep{PaBaVe18}. These functions are rooted in the literature on orthogonal polynomials \citep{Nev86}. To make it simpler, we introduce them in dimension $d=1$. They are defined by
\begin{equation}
C_{\ell,\mathrm{d}\omega}: z \mapsto \min\limits_{\substack{P \in \mathbb{R}_{\ell}[X]\\ P(z) = 1}} \int_{\mathcal{X}} P(x)^{2} \mathrm{d}\omega(x), \quad \ell \in \mathbb{N}.
 \end{equation}
Christoffel functions have a more explicit form \citep{Nev86} that can be used for pointwise evaluation
\begin{equation}\label{eq:Christoffel_functions_definition}
C_{\ell,\mathrm{d}\omega}(z) = \frac{1}{ \sum\limits_{m \leq \ell} P_{m}(z)^{2} },\quad\ell\in\mathbb{N},
\end{equation}
where $(P_{m})_{m \in \mathbb{N}}$ are the orthonormal polynomials with respect to $\mathrm{d}\omega$. To establish a connection with $q_{\lambda}$, the authors of \cite{PaBaVe18} defined regularized Christoffel functions for some kernel $k$:
\begin{equation}
C_{\lambda,\mathrm{d}\omega,k}: z \mapsto \min\limits_{\substack{f \in \mathcal{F}\\ f(z) = 1}} \int_{\mathcal{X}} f(x)^{2} \mathrm{d}\omega(x) + \lambda \|f\|_{\mathcal{F}}^{2}, \quad \lambda \in \mathbb{R}_{+}^{*}.
 \end{equation}
The authors derived an asymptotic equivalent of the function $C_{\lambda,w,k}$ in the regime $\lambda \rightarrow 0$ under some assumptions on the kernel. Furthermore, they proved that $C_{\lambda,w,k}$ is tied to $q_{\lambda}$ by the following relationship (Lemma 5, \citep{PaBaVe18}):
\begin{equation}
q_{\lambda}(x) \propto \langle k(x,.), (\Sigma+ \lambda \mathbb{I}_{\mathcal{H}})^{-1} k(x,.) \rangle_{\mathcal{F}} = \frac{1}{C_{\lambda,w,k}(x)}.
\end{equation}

On the other hand, assume that the $(\psi_{n})$ are the family of orthonormal polynomials with respect to $\mathrm{d}\omega$. Let $x \in \mathcal{X}$ and $\bm{x}$ a random subset of $\mathcal{X}^{N}$ drawn from the Projection DPP $(\KDPP, \mathrm{d}\omega)$, then
\begin{equation}
\Prb_{\DPP}(x \in \bm{x}) = \frac{1}{N}\KDPP(x,x)\mathrm{d}\omega(x) = \frac{1}{N}\sum\limits_{n \in [N]} \psi_{n}(x)\psi_{n}(x) \mathrm{d}\omega(x)= \frac{1}{NC_{N,\mathrm{d}\omega}(x)} \mathrm{d}\omega(x).
\end{equation}
In other words, the inclusion probability of the corresponding projection DPP is related to the inverse of the Christoffel function as defined in \eqref{eq:Christoffel_functions_definition}. Figure \ref{fig:unidimHermiteEnsembleChristoffel} illustrates the evaluations of the inclusion probability of the projection DPP in the case of RKHS defined by the Gaussian kernel along with the Gaussian measure in the real line. Recall that in this case the eigenfunctions are given by
\begin{equation}
\tilde{e}_{m}(.) = H_{m}(\sqrt{2c}.).
\end{equation}
The theoretical analysis of the "bumps" of the functions $x \mapsto 1/N\KDPP(x,x)\mathrm{d}\omega(x)$ was carried out in \citep{DuEd05}. More precisely, the authors studied the approximations of those bumps by Gaussians centred on the Hermite polynomials roots, see Figure \ref{fig:unidimHermiteEnsembleChristoffel} (b). We observe a similar behaviour for the multidimensional Gaussian case as illustrated in Figure \ref{fig:HermiteEnsembleChristoffel}: the inclusion probability of the projection DPP have has local maxima around the tensor products of the Hermite polynomials roots. In other words, the quadratures based on nodes sampled according to a projection DPP are probabilistic relaxations of classical quadratures based on roots of orthogonal polynomials that can be defined even if $N$ is not the square of an integer (the cases $N \in \{17,21\}$ in Figure~\ref{fig:HermiteEnsembleChristoffel}).

\begin{figure}[]
    \centering
\includegraphics[width= 0.45\textwidth]{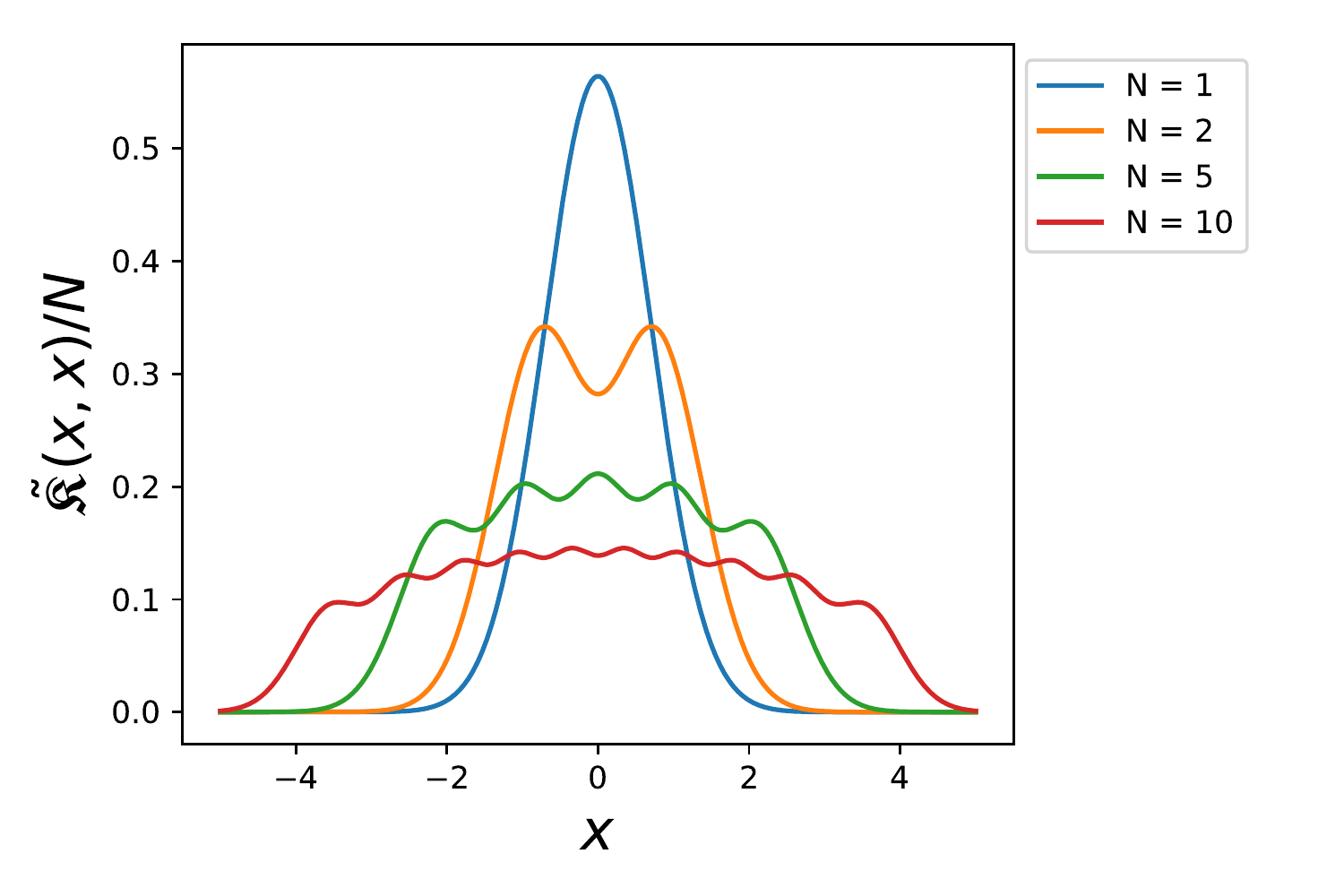}~
\includegraphics[width= 0.45\textwidth]{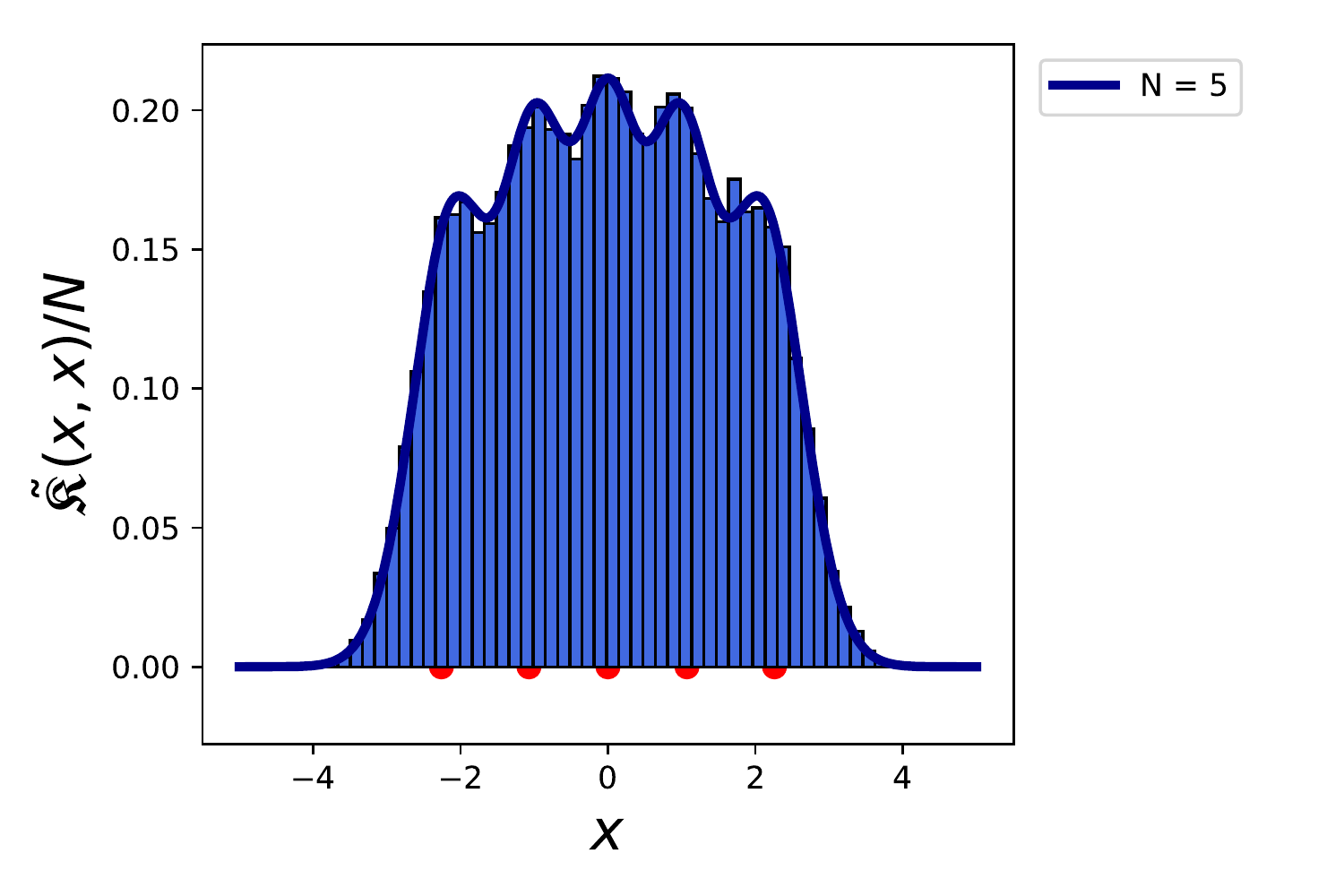}
\caption{The inclusion probability of the projection DPP in the Gaussian case $(d=1)$: (a) the evaluations of the functions $\displaystyle x \mapsto \frac{\KDPP(x,x)}{N}\mathrm{d}\omega(x)$ for $N \in \{1,2,5,10\}$ where $\mathrm{d}\omega$ is the measure of a normalized Gaussian variable, (b) the empirical inclusion probability based on 50000 realisations of the projection DPP compared to the evaluation of the Christoffel function ($N=5$), the dots in red corresponds to the zeros of the scaled Hermite polynomial of order 5.   \label{fig:unidimHermiteEnsembleChristoffel}}
\end{figure}
\begin{figure}[]
    \centering
\includegraphics[width= 0.48\textwidth]{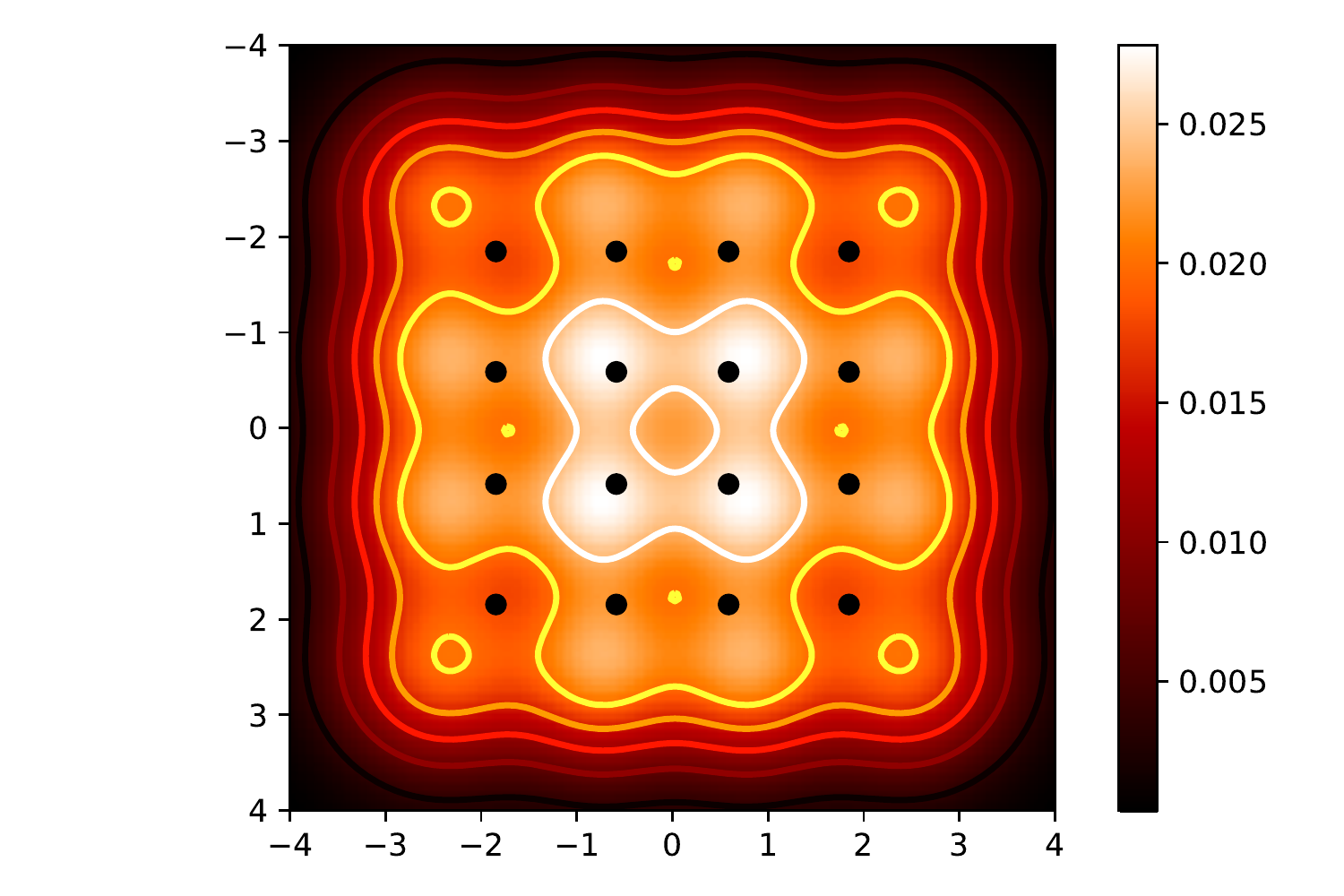}
\includegraphics[width= 0.48\textwidth]{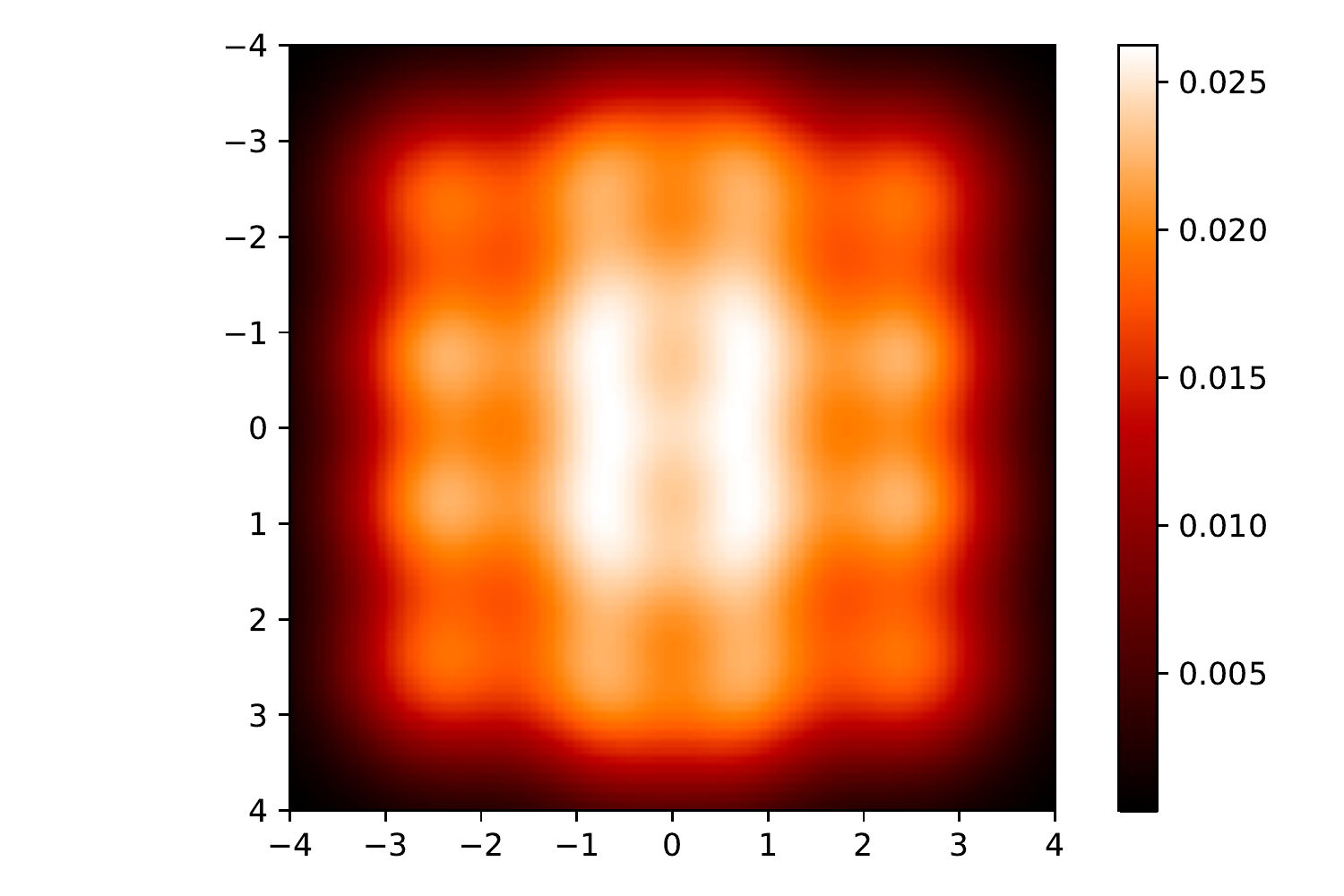}
\includegraphics[width= 0.48\textwidth]{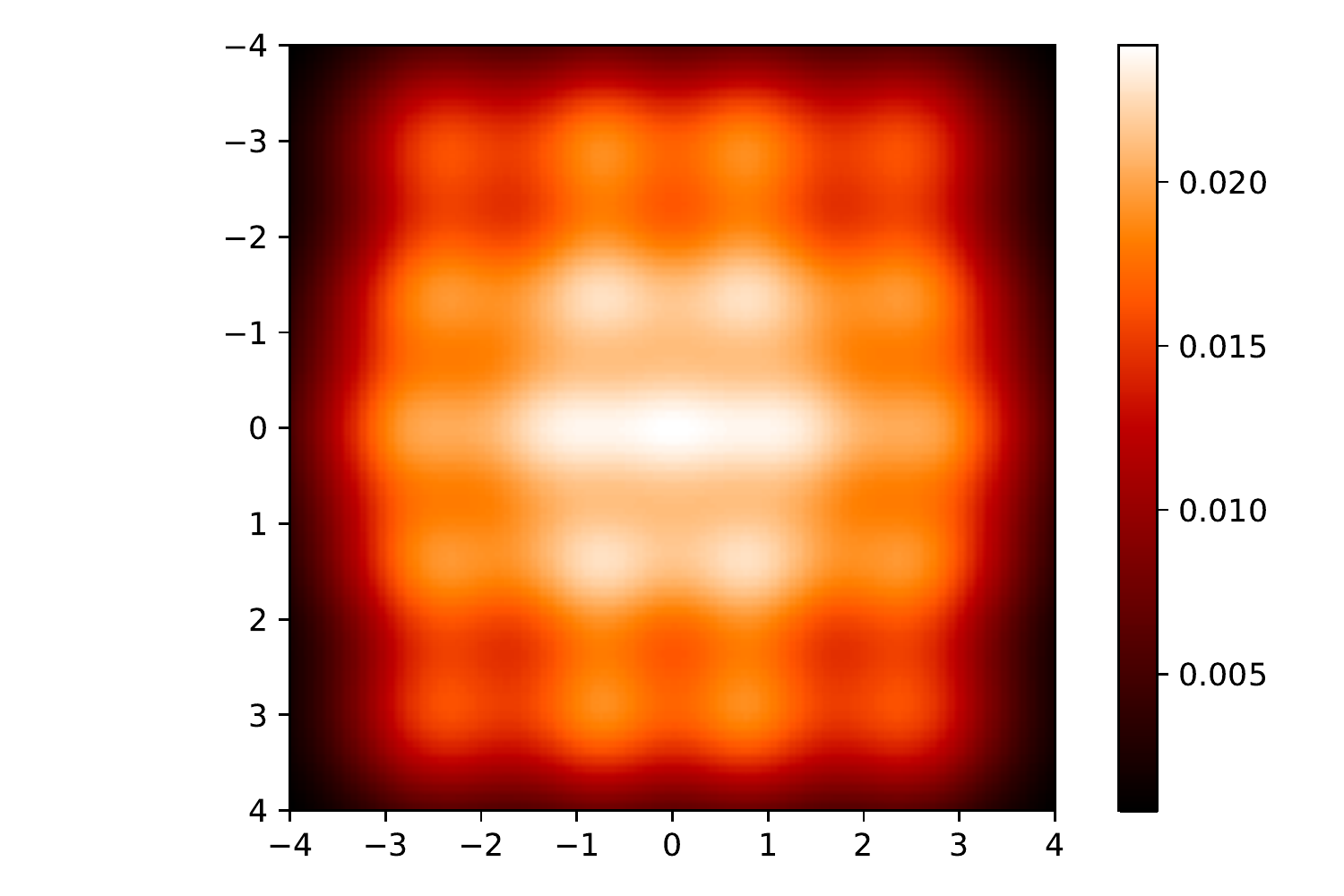}
\includegraphics[width= 0.48\textwidth]{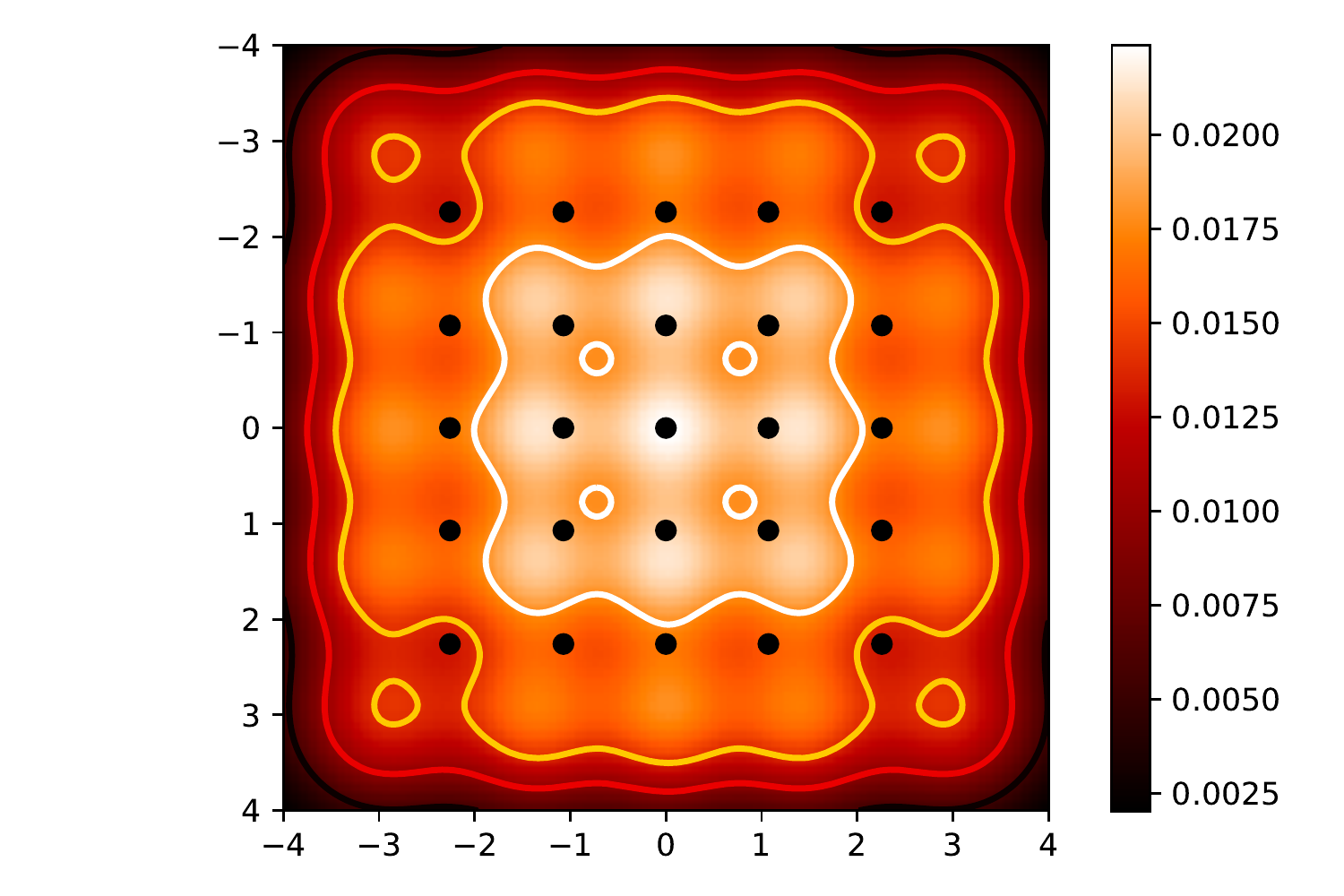} \\
\caption{The inclusion probability of the projection DPP in the multidimensional Gaussian case $(d=2)$: the evaluations of the functions $\displaystyle  x \mapsto \frac{\KDPP(x,x)}{N} \otimes_{i =1}^{d} \mathrm{d}\omega(x_{i})$ for $N \in \{16,17,21,25\}$, the dots in black corresponds to the tensor product of the zeros of the scaled Hermite polynomials.   \label{fig:HermiteEnsembleChristoffel}}
\end{figure}

\newpage


\end{document}